\pdfoutput=1 
\documentclass{article}

\usepackage[margin=3cm]{geometry}
\usepackage{indentfirst}

\usepackage{mathtools} 
\usepackage{amsthm,amssymb}
\usepackage{bbold} 

\numberwithin{equation}{section}
\mathtoolsset{showonlyrefs}

\usepackage{graphicx}
\usepackage{caption} 
\usepackage{tabularx}
\usepackage{booktabs}
\usepackage{xcolor}
\usepackage{tikz}
\usetikzlibrary{positioning}

\definecolor{phCap}{HTML}{8C8C8C}
\definecolor{phEarly}{HTML}{E8702A}
\definecolor{phQuad}{HTML}{1FA37A}
\definecolor{phNonq}{HTML}{F2B800}
\definecolor{phLate}{HTML}{8A4FD1}
\newcommand{\phbox}[1]{\textcolor{#1}{\rule{1.2ex}{1.2ex}}}  
\definecolor{wtsOrange}{HTML}{F8A800}
\definecolor{wtsPurple}{HTML}{5800A8}
\definecolor{wtsPink}{HTML}{E83870}

\usepackage{natbib}
\usepackage{hyperref}
\definecolor{tabblue}{HTML}{175987} 
\definecolor{tabpurple}{HTML}{9467BD}
\hypersetup{colorlinks=true, linkcolor=tabblue, citecolor=tabblue, urlcolor=tabpurple}

\theoremstyle{definition}
\newtheorem{defn}{Definition}
\newtheorem{theo}{Theorem}
\newtheorem{prop}{Proposition}
\newtheorem{lemma}{Lemma}
\newtheorem{obs}{Remark}
\newtheorem{assum}{Assumption}
\newtheorem{ex}{Example}
\newtheorem{claim}{Claim}
\newtheorem*{proofclaim}{Proof of the Claim}
\newcommand{\claimqed}{\hfill \ensuremath{\blacksquare}}

\newcommand{\E}{\mathbb{E}}
\newcommand{\R}{\mathbb{R}}
\DeclareMathOperator{\tr}{tr}
\DeclareMathOperator*{\argmin}{arg\,min}
\DeclareMathOperator{\diag}{diag}

\newcommand{\chg}[1]{#1}

\title{Quadratic Weak-to-Strong Generalization in Random Feature Networks via Random Matrix Theory}
\author{%
  Deborah Oliveira\\
  {\normalsize McGill University}\\
  {\normalsize Mila -- Quebec AI Institute}\\
  {\small\texttt{deborah.deoliveira@mail.mcgill.ca}}
  \and
  Elliot Paquette\\
  {\normalsize McGill University}\\
  {\normalsize Mila -- Quebec AI Institute}\\
  {\small\texttt{elliot.paquette@mcgill.ca}}
}
\date{}

\begin{document}

\maketitle

\begin{abstract}
Weak-to-strong generalization is the phenomenon where a strong student model trained with labels produced by a weak teacher model is able to generalize better than the teacher. In this paper, we study this phenomenon in two-layer random feature networks where the model strength is determined by its width. Using tools from random matrix theory, we derive deterministic equivalents for the population errors of an optimally trained teacher and a student trained with gradient flow. For ReLU activation and a pure spherical harmonic target, we obtain sharp asymptotics under a Gaussian universality assumption, showing a quadratic improvement: the student error scales as the square of the teacher error. These results attain the general lower bound of \citet{medvedev2025weak}. We also analyze how the student behaves under more general stopping times and targets supported on multiple harmonic degrees, characterizing the regimes in which weak-to-strong generalization occurs and identifying the transition between quadratic, non-quadratic, and no improvement.
\end{abstract}

\section{Introduction}\label{sec:intro}

Can a student trained only on a teacher's predictions surpass the teacher? \chg{A small model trained on the true labels still makes errors because of its limited capacity, and a larger student that sees only these flawed labels might be expected to reproduce them. \textit{Weak-to-strong generalization} (WTS), introduced in \citet{openai2024weaktostrong}, is the phenomenon where it does not: the strong student generalizes better than the weak teacher that supervises it.}

\chg{We study this phenomenon theoretically in a random feature neural network, in a regression problem with squared loss, where the strength of a model is its capacity, that is, its number of neurons. We choose this model for three reasons. Its strength has a single unambiguous meaning, the width. Its population loss is governed by the spectrum of the feature kernel, so the teacher's error and the student's training dynamics can be written explicitly, and, for data on the sphere, the expansion of the target in spherical harmonics gives a natural notion of its complexity, its harmonic degree. Finally, random feature models, introduced in \citet{rahimi2007random} as a way of reducing computation in kernel methods, are conceptually simple yet exhibit phenomena also present in deep networks, and are a standard surrogate for them \citep{mei2022generalization,bombari2024word,dohmatob2022nonrobustness,chizat2019lazy}. Our aim is to show that WTS emerges already in this model, and to quantify the contribution of each of its components.}

We closely follow the setup pioneered in \citet{medvedev2025weak}, where random feature networks also represent the student and teacher. Our two-step training procedure is the same as theirs: we first optimally train the teacher over the population error, so that its only source of error is its capacity, and then train the student\chg{, initialized at zero,} on the teacher's labels using early-stopped gradient flow. The goal is to stop the training before the student imitates too many of the teacher's mistakes, which is an implicit regularization effect \citep{li2020gradientdescent,yao2007early,raskutti2014early,ali2019continuous,heckel2021early}. In our specific case, it can be understood as a spectral cutoff: the student imitates the teacher's components \chg{up to} the target degree, while the remaining components barely move from their initialization.

\chg{We differ from \citet{medvedev2025weak} both in method and in scope. In method, we use random matrix theory, specifically the theory of Dyson equations and deterministic equivalents \citep{erdos2019matrix}. A deterministic equivalent removes the randomness of the features from the population errors and gives explicit formulas for both student and teacher errors: a single-resolvent formula for the teacher and a two-resolvent formula for the student, the latter being a deterministic equivalent for a two-stage training pipeline. This simplified mechanism reproduces the errors accurately and makes their analysis straightforward. In scope, the formulas hold at every stopping time and for general targets, which lets us analyze the optimal improvement and the full effect of the stopping time.}

\chg{We leave open the comparison between the deterministic equivalents and the random feature regression itself. Results of this kind exist in the single-stage setting \citep{defilippis2024dimensionfree,cheng2024dimension,misiakiewicz2024non}; extending them to the two-stage pipeline is a matter of mathematical technicalities that we do not expect to reveal new information about the learning problem, and our numerical comparisons show that the deterministic equivalents match the gradient-flow simulations closely.}

\paragraph{Our contributions.}
\begin{itemize}
\item \chg{\textbf{Deterministic equivalents and a general lower bound.} We derive deterministic equivalents for the population errors of the optimally trained teacher and of the early-stopped student (Section~\ref{subsection:ReLU}, Appendix~\ref{app:detequivalent}). The student error splits into a squared bias and a variance term. The decomposition explains both the improvement over the teacher and its loss under prolonged training, and implies that the student error is at least the square of the teacher error, for any activation and any target (Theorem~\ref{theo:quadratic-floor}).}
\item \chg{\textbf{Quadratic improvement.} For ReLU activation on the sphere and a pure spherical-harmonic target, we show that an appropriate stopping time attains this quadratic improvement (Theorem~\ref{theo:weaktostrongrelu}, Section~\ref{sec:quadratic}), attaining the lower bound of Theorem~\ref{theo:quadratic-floor}. The ReLU statements are conditional on a Gaussian universality assumption (Assumption~\ref{ass:de}).}
\item \chg{\textbf{Pure harmonic targets and general stopping times.} We characterize the phase diagram of the student error in terms of the teacher's capacity exponent $\alpha$ ($m=d^\alpha$ neurons) and the stopping time exponent $\gamma$ $(T=\Theta(d^{\gamma+1}))$, identifying which of the two error terms controls the error in each region and the transitions between quadratic, nonquadratic, and no improvement (Theorem~\ref{theo:phases}, Section~\ref{subsec:phase-diagram}).}
\item \chg{\textbf{Mixed targets and general stopping times.} For targets with components on several harmonic degrees, we give a condition for WTS in terms of unlearned target energy, and a coefficient-independent guarantee of polynomial improvement when the teacher can learn every mode. Two examples show that the improvement can nevertheless fail to be quadratic, because different modes require incompatible stopping times (Section~\ref{sec:mixed-main}).}
\end{itemize}

\paragraph{Related work.}
Since the first empirical investigation of \citet{openai2024weaktostrong}, many works have considered a theoretical analysis using simple settings. Some focus on how data structure influences the phenomenon. \citet{hunter2025theoretical} study the student's correction of teacher mistakes and its performance on regions not covered by teacher labels. \citet{shin2025weaktostrong} express improvement through the frequency of points containing both easy and hard patterns, while \citet{wu2025provable} study a transition from random guessing to perfect generalization in classification.

Other works quantify the improvement through the \textit{misfit error}, which measures the disagreement between the student and teacher \citep{charikar2024quantifying,yao2025understanding,mulgund2025relating}. From \chg{the point of view of representations}, \citet{xue2025representations} identify directions poorly captured by the teacher but well represented by the student. \citet{dong2025discrepancies} analyze how the student inherits teacher variance along shared directions and reduces it along others. \chg{In these works, differences between the two representation spaces prevent the student from fully inheriting the teacher's noise, which is the source of the improvement.}

With respect to regularization, \citet{moniri2026on} show that a student in linear ridge regression can surpass the teacher through better regularization, especially when it is aligned with the target. This is related to \citet{medvedev2025weak} and our work, where regularization is imposed implicitly by early-stopped gradient flow. Beyond linear settings, \citet{oh2025from} use convolutional networks and also find that early stopping improves student performance in a data-abundant regime.

\chg{Closest to our setting are analyses of weak-to-strong generalization in random feature and high-dimensional linear regression, where teacher and student differ in width or in regularization and the labels are produced by a fitted teacher. \citet{ildiz2025highdimensional} analyze ridge regression trained on the predictions of a surrogate model in the proportional regime, characterizing when the student improves on the teacher and deriving scaling laws, and \citet{wu2026improved} show that weak-to-strong training improves the scaling law of random feature ridge regression. In both works, as in ours, deterministic equivalents play a central role. Appendix~\ref{app:extended-related} gives an extended discussion.}


\paragraph{Notation.}
\chg{Throughout, $d$ is the dimension of the data and all asymptotic statements are as $d\to\infty$.} We write $a\asymp b$ for bounds in both directions by positive constants independent of $d$, $a\sim b$ for $a/b\to1$, and $\widetilde\Theta$ for order up to polylogarithmic factors. 

\paragraph{Acknowledgments.}
We would like to thank Nati Srebro for explaining the problem setup and suggesting the topic. The second author would like to acknowledge funding support for the project from a Discovery Grant from the Natural Science and Engineering Research Council (NSERC), a Google Canada Research Grant, a CIFAR AI Catalyst Grant, a Fonds de recherche du Quebec Nature et technologies (FRQNT)-NSERC NOVA Grant, and US AFOSR grant.

\section{Problem setup}\label{sec:setup} 

For both the weak teacher and the strong student, data vectors $x\in\mathbb R^d$ come from a distribution $\mathcal D$\chg{, and labels are in $\mathbb R$}. The predictor is a random feature model, which can be seen as a two-layer neural network with a fixed bottom random layer:
\begin{equation}
f_{u,w}(x)=\frac1{\sqrt m}\sum_{i=1}^m w_i\sigma(\langle u_i,x\rangle),\qquad u_i\stackrel{\rm iid}{\sim}\mathcal U.
\end{equation}
\chg{Here $m$ is the number of neurons, $\sigma:\mathbb R\to\mathbb R$ is the activation function, the feature weights $u_i\in\mathbb R^d$ are fixed, and only the output weights $w_i\in\mathbb R$ are trained. The $1/\sqrt m$ normalization makes the kernel of the network, $\frac1m\sum_{i=1}^m\sigma(\langle u_i,x\rangle)\sigma(\langle u_i,x'\rangle)$, an average over the features, so that it converges to the population kernel $K$ below as $m\to\infty$. 
We will take both $\mathcal D$ and $\mathcal U$ uniform on the sphere $\mathbb S^{d-1}$, although the machinery of the theory would work for many other choices.}
\chg{Apart from their training procedures, the only difference between teacher and student is their width, which is our notion of model strength: the teacher has $M_{\rm TE}$ neurons, and the student is taken infinitely wide, $M_{\rm ST}\to\infty$.  We expect that all results here would hold for a finite student as well, for student sizes which are sufficiently large to learn $f^{\star}$, but leave a quantitative analysis to future work.}

Both student and teacher are evaluated against the ground-truth target $f^\star$ using population squared loss, which we write using the norm $\|f\|_{\mathcal D}^2=\mathbb E_{x\sim\mathcal D}f(x)^2$, as
\begin{equation}\label{eq:teacherpoperror}
\mathcal L_{\rm TE}=\|f_{\rm teacher}-f^\star\|_{\mathcal D}^2,
\qquad \mathcal L_{\rm ST}(t)=\|f_{\rm student}^{(t)}-f^\star\|_{\mathcal D}^2.
\end{equation}
\chg{Weak-to-strong generalization means that the student error is smaller than the teacher error.  However, in this paper we use WTS to refer to the stronger statement that $\mathcal L_{\rm ST}(t)/\mathcal L_{\rm TE}$ decays as a power of $d$.}

\paragraph{Teacher.}
\chg{The teacher is optimally trained on the true labels: its output weights minimize the population loss $\mathcal L_{\rm TE}$ over $w\in\mathbb R^{M_{\rm TE}}$. Its only source of error is therefore its finite width. The optimal weights and the resulting error are computed in Appendix~\ref{app:detequivalent}, Lemma~\ref{lemma:optimalweigth}.}

\paragraph{Student.}
The student receives \chg{the teacher's predictions as labels}, starts from \chg{$w_{\rm s}=0$}, and is trained by the gradient flow,
\begin{equation}\label{eq:gradflowODE}
\frac{dw_{s,j}}{dt}=-\frac12\frac{\partial}{\partial w_{s,j}}
\mathbb E_{x\sim\mathcal D}(f_{\rm student}^{(t)}(x)-f_{\rm teacher}(x))^2\chg{,\qquad j=1,\dots,M_{\rm ST},}
\end{equation}
\chg{stopped at a time $t$ chosen below. As $M_{\rm ST}\to\infty$ this converges to the population spectral flow}
\begin{equation}\label{studenttrain}
f_{\rm student}^{(t)}(x)=\sum_i(1-e^{-\psi_it})\langle f_{\rm teacher},e_i\rangle_{\mathcal D}e_i(x),
\end{equation}
where $(\psi_i,e_i)$ are the eigenpairs in $L^2(\mathcal D)$ of the kernel $K(x,x')=\mathbb E_{u\sim\mathcal U}\sigma(\langle u,x\rangle)\sigma(\langle u,x'\rangle)$\chg{, the average of the student network kernel over the features (see Appendix~A.1 of \citealp{medvedev2025weak}).}

We shall order the eigenvalues to be decreasing $\psi_1\ge\psi_2\ge\cdots$.
Each feature $\sigma(\langle u_i,\cdot\rangle)$ has coordinates $\psi_j^{1/2}e_j(u_i)$ in this basis.
We let $\Psi$ be the covariance of the feature coordinates $\psi_j^{1/2}e_j(u_i)$ over $u_i\sim\mathcal U$, so that $\Psi = \diag(\psi_1,\psi_2,\dots)$. The target is $f^\star=\sum_j\beta_j\psi_j^{1/2}\phi_j$, so that $\Psi^{1/2}\beta$ is its coefficient vector in the orthonormal basis and $\|f^\star\|_{\mathcal D}^2=\langle\beta,\Psi\beta\rangle$, which we normalize to one. 

\section{Deterministic equivalents}\label{subsection:ReLU} 

In the random features setting we have two sources of randomness: the data distribution and the random features distribution. \chg{A deterministic equivalent allows us to find a system of fixed-point equations, around whose solution the student and teacher test errors concentrate, over both sources of randomness.}

\paragraph{\chg{Teacher.}}
\chg{The deterministic equivalent of the teacher error is governed by a single scalar $\nu>0$, defined as the unique solution of}
\begin{equation}\label{eq:main-fixed-point}
\operatorname{tr}\bigl[\Psi(\Psi+\nu I)^{-1}\bigr]=m.
\end{equation}
\chg{Since each summand $\psi_j/(\psi_j+\nu)$ is close to one for $\psi_j\gg\nu$ and close to zero for $\psi_j\ll\nu$, the equation says that the number of eigenvalues of $\Psi$ above $\nu$ is about $m$: $\nu$ is the eigenvalue scale at the edge of the teacher's capacity. It plays the role of an effective ridge parameter, since the teacher error has the form of a ridge regression error on the full feature space with ridge $\nu$ (Appendix~\ref{subapp:teacherdeteq}, Theorem~\ref{theo:selfconsistenteacher}):}
\begin{equation}\label{eq:main-teacher}
\overline{\mathcal L}_{\rm TE}
=\nu\langle\Psi^{1/2}\beta,(\Psi+\nu I)^{-1}\Psi^{1/2}\beta\rangle.
\end{equation}
\chg{The factor $\nu/(\psi_j+\nu)$ multiplying the energy of the target along $\phi_j$ is small on the directions the teacher can represent ($\psi_j\gg\nu$) and close to one on those it cannot.}

\paragraph{\chg{Student.}}
\chg{Substituting the optimal teacher into the spectral flow \eqref{studenttrain} shows that the student error depends on the teacher's random features through two projectors instead of one (Appendix~\ref{subapp:studentdeteq}). Its deterministic equivalent therefore requires a double-resolvent calculation, Theorem~\ref{theo:doubleresolventapprox}. Regrouping the resulting terms (Appendix~\ref{app:mechanism}) gives the two-term expression that we take as the definition of the deterministic equivalent of the student error: a squared bias $\overline B(t)$ and a variance $\overline V(t)$,}
\begin{equation}\label{eq:main-two-term}
\begin{aligned}
\overline{\mathcal L}_{\rm ST}(t)
={}&\underbrace{\left\|(\nu I+\Psi e^{-\Psi t})(\Psi+\nu I)^{-1}\Psi^{1/2}\beta\right\|^2}_{\overline B(t)}\\
&+\underbrace{\nu\langle\Psi^{1/2}\beta,\Psi(\Psi+\nu I)^{-2}\Psi^{1/2}\beta\rangle
\frac{\operatorname{tr}[\Psi(\Psi+\nu I)^{-2}(I-e^{-\Psi t})^2]}
{\operatorname{tr}[\Psi(\Psi+\nu I)^{-2}]}}_{\overline V(t)}.
\end{aligned}
\end{equation}
At $t=0$ this equals one\chg{, the norm of the target,} and at $t=\infty$ it equals $\overline{\mathcal L}_{\rm TE}$\chg{: without early stopping the student simply reproduces the teacher.} Appendix~\ref{app:mechanism} verifies \eqref{eq:main-two-term} algebraically.

The deterministic equivalent is obtained from the matrix Dyson equation \citep{erdos2019matrix}.  The derivation of the deterministic equivalent is for Gaussian feature coordinates, using Stein's lemma and removing lower order terms, which would vanish by concentration.  Hence, the deterministic equivalent for the ReLU model simultaneously reflects two approximations: first a concentration property and second a universality property.  See \ref{app:dyson-background} for a general introduction and \eqref{eq:cumulantexpansion} for a discussion of the universality.

We shall formulate all our theorems for the deterministic equivalent, and we expect that these transfer exactly to the random model as we have strong empirical evidence (see Figure \ref{fig:illustrative-time-relative}) and theoretical support (see discussion at the end of Appendix \ref{app:dyson-background}).  Precisely, we anticipate the following holds:
\begin{assum}[Gaussian universality and deterministic equivalence]\label{ass:de}
\chg{Let $\overline{\mathcal L}_{\rm TE}$ and $\overline{\mathcal L}_{\rm ST}(t)$ be the deterministic equivalents defined in \eqref{eq:main-teacher} and \eqref{eq:main-two-term} below, and let $(T_d)$ be any of the stopping-time sequences appearing in our statements. The random population losses of the ReLU model on the sphere satisfy, in probability as $d\to\infty$,}
\[
\mathcal L_{\rm TE}/\overline{\mathcal L}_{\rm TE}\to1,
\qquad \mathcal L_{\rm ST}(T_d)/\overline{\mathcal L}_{\rm ST}(T_d)\to1.
\]
For statements uniform over stopping times or target families, the corresponding relative approximation is assumed uniform.
\end{assum}



\subsection{The bias--variance mechanism}\label{subsec:bias-variance}

Averaging over the teacher features $\mathcal{U}$ gives the exact identity
\begin{equation}\label{eq:exact-bias-variance}
\mathbb E_{\mathcal{U}}\mathcal L_{\rm ST}(t)
=\|\mathbb E_{\mathcal{U}}f_{\rm student}^{(t)}-f^\star\|_{\mathcal D}^2
+\mathbb E_{\mathcal{U}}\|f_{\rm student}^{(t)}-\mathbb E_{\mathcal{U}}f_{\rm student}^{(t)}\|_{\mathcal D}^2.
\end{equation}
\chg{The deterministic equivalent of the first term (the squared mean residual) is $\overline B$, and that of the second term (the fluctuation of the student around its mean) is $\overline V$} (see Appendix~\ref{app:mechanism} for details\chg{; see also Appendices~\ref{app:reluproofs} and~\ref{app:generalsettings}, where $\overline V$ is term 4)(2) and the rest sum to $\overline B$}). 

The bias can be seen to decrease in time, while the variance increases.  Therefore, provided $\overline{B}(\infty) < \overline{V}(\infty)$ and $\overline{B}(0) > \overline{V}(0)$, there exists an optimal finite stopping time at which they are equal.  As this can be substantially smaller than the sum of these losses at time $\infty$, which is equal to $\overline{\mathcal{L}}_{\rm TE}$, this shows how the student trained with early stopped gradient flow can outperform the teacher.

Further, we can give a general lower bound, using this representation, which will be sharp in some instances:
\begin{theo}[General quadratic lower bound]\label{theo:quadratic-floor}
For any activation, any feature distribution, any data distribution, and any unit target for which~\eqref{eq:main-fixed-point}--\eqref{eq:main-two-term} are defined, every stopping time $t$ satisfies
\begin{equation}\label{eq:quadratic-floor}
\overline{\mathcal L}_{\rm ST}(t)\ge
\nu^2\langle\Psi^{1/2}\beta,(\Psi+\nu I)^{-2}\Psi^{1/2}\beta\rangle
\ge\overline{\mathcal L}_{\rm TE}^{\,2}.
\end{equation}
\end{theo}
The first inequality drops the nonnegative variance and the exponential in the bias. The second is Cauchy--Schwarz, using $\|\Psi^{1/2}\beta\|=1$. Thus quadratic improvement is the best possible order for the deterministic equivalent, for arbitrary targets. \chg{This is the deterministic-equivalent counterpart of the lower bound of \citet{medvedev2025weak}, who prove the analogous result for the random student error.}

\section{Optimal quadratic improvement}\label{sec:quadratic}

\paragraph{\chg{The ReLU model on the sphere.}}\label{subsec:relu-model}
We take $\sigma(r)=r_+$ and recall we use the uniform probability measure on $\mathbb S^{d-1}$ for both the data and the fixed weight layer, as in \citet{medvedev2025weak}. The population kernel then expands in an orthonormal spherical-harmonic basis $\phi_{k,l}$, $l=1,\dots,N_k$, where $N_k$ is the dimension of the space of degree-$k$ harmonics and $\sigma_k^2$ is the eigenvalue of $K$ on that space, so that $\sigma_k$ is the degree-$k$ coefficient of ReLU in the spherical-harmonic expansion; it vanishes for odd $k\ge3$. The kernel is
\begin{equation}\label{eq:relukernel}
K(x,x')=\sum_{k\in\mathcal K}\sigma_k^2\sum_{l=1}^{N_k}
\phi_{k,l}(x)\phi_{k,l}(x'),\qquad
\mathcal K=\{0,1\}\cup2\mathbb N.
\end{equation}
 In the notation \chg{of Section~\ref{subsection:ReLU}}, $\Psi$ has blocks $\sigma_k^2 I_{N_k}$; consequently traces become sums over degrees with multiplicities $N_k$\chg{, and, by the spectral flow \eqref{studenttrain}, a target of degree $k$ is learned by the student on the time scale $1/\sigma_k^2$}. For fixed $k$, $N_k\asymp d^k$ and $\sigma_k^2\asymp d^{-k-1}$. Appendix~\ref{subapp:asympNksigmak} gives their expressions and asymptotics.

Our first result is as follows.
\begin{samepage}
\begin{theo}[Quadratic improvement: pure harmonic target]\label{theo:weaktostrongrelu}
Let $f^\star$ be a unit spherical harmonic of fixed degree $k_0\in\mathcal K$, $k_0\ge1$, and let $m=d^\alpha$ with fixed $\alpha>0$. \chg{Then, with}
\begin{equation}\label{eq:quadratic-stop}
T=\frac{\log m+\log \overline{\mathcal L}_{\rm TE}^{-1}}{\sigma_{k_0}^2}\chg{,\qquad
\overline{\mathcal L}_{\rm TE}\asymp d^{-(\alpha-k_0)_+},
\qquad
\overline{\mathcal L}_{\rm ST}(T)\asymp \overline{\mathcal L}_{\rm TE}^2.}
\end{equation}
\end{theo}
\end{samepage}

The proof is in Appendix~\ref{app:reluproofs}. Together with Theorem~\ref{theo:quadratic-floor}, this establishes optimal quadratic improvement\chg{ in the ReLU setting once $\alpha > k_0$}. 

\subsection{\chg{The role of the activation}}\label{subsec:general-activation-main}
\chg{We briefly investigate how the choice of activation affects Theorem~\ref{theo:weaktostrongrelu}. The short answer is that only the set $\mathcal K$ of degrees present in the kernel changes. Let $\sigma:\mathbb R\to\mathbb R$ be measurable with $|\sigma(r)|\le Ce^{C|r|}$, and consider the features $\sigma(\sqrt d\,\langle u,x\rangle)$, so that the preactivation is of order one (for ReLU this is a constant multiple of the features above). With $z\sim\mathcal N(0,1)$ and $\mathrm{He}_k$ the probabilists' Hermite polynomials, set
\begin{equation}\label{eq:hermite-coeff-main}
c_k=\frac{\mathbb E[\sigma(z)\mathrm{He}_k(z)]}{\sqrt{k!}},\qquad
\mathcal K_\sigma=\{k\ge0:\ c_k\neq0\}.
\end{equation}
For ReLU, $\mathcal K_\sigma=\mathcal K=\{0,1\}\cup2\mathbb N$; for $\tanh$, $\mathcal K_\sigma$ is the odd integers; for the logistic sigmoid it is $\{0\}$ together with the odd integers, and a generic bias makes it all of $\mathbb N$ (Appendix~\ref{app:general-activation}, Examples~\ref{ex:general-relu}--\ref{ex:general-sigmoid}).}

\begin{samepage}
\begin{theo}[\chg{Quadratic improvement: general activations}]\label{theo:general-activation-main}
\chg{Let $\sigma$ be as above, let $f^\star$ be a unit spherical harmonic of fixed degree $k_0\in\mathcal K_\sigma$, $k_0\ge1$, and let $m=d^\alpha$ with fixed $\alpha>0$. Assume $c_{k_\star}\ne0$ for some $k_\star>\alpha$ (which holds for every $\alpha$ if $\sigma$ is not a polynomial). Then, with the stopping time $T$ of \eqref{eq:quadratic-stop},
\[
\overline{\mathcal L}_{\rm TE}\asymp d^{-(\alpha-k_0)_+},
\qquad
\overline{\mathcal L}_{\rm ST}(T)\asymp\overline{\mathcal L}_{\rm TE}^{\,2}.
\]}
\end{theo}
\end{samepage}

\chg{The proof is in Appendix~\ref{app:general-activation}. The kernel again has the form \eqref{eq:relukernel}, with $\sigma_k$ the degree-$k$ spherical-harmonic coefficient of $\sigma(\sqrt d\,\cdot)$, and $N_k\sigma_k^2\to c_k^2$ for each fixed $k$ (Lemma~\ref{lem:general-coefficients}), so $\sigma_k^2\asymp d^{-k}$ for $k\in\mathcal K_\sigma$, as for ReLU. The proof of Theorem~\ref{theo:weaktostrongrelu} uses the $\sigma_k$ only through these orders at finitely many degrees and the bound $\operatorname{tr}\Psi\lesssim1$, so it goes through unchanged.}


\begin{figure}[htbp]
\centering
\includegraphics[width=\linewidth]{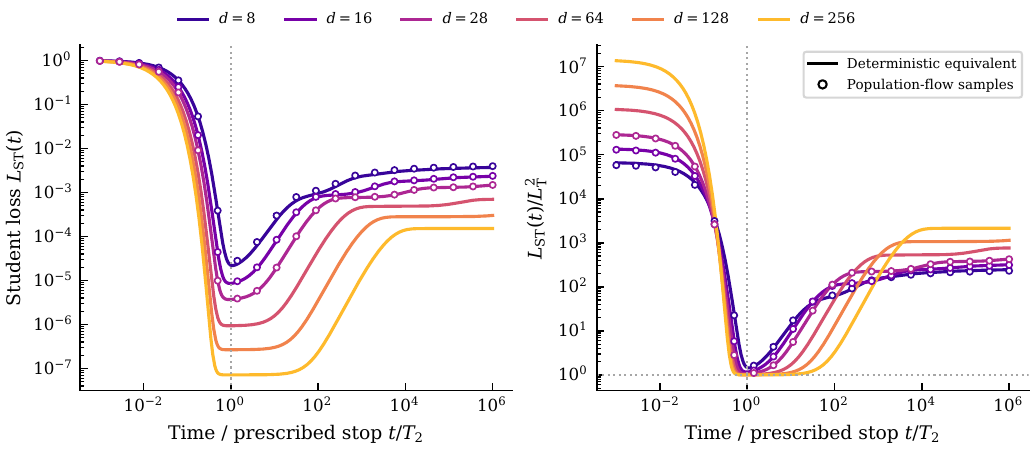}
\caption{\textbf{Quadratic improvement and the stopping window.} Unit degree-two target, $m=d^3$, with time divided by \chg{$T_2=4\log(d)/\sigma_2^2$, the stopping time $T$ of Theorem~\ref{theo:weaktostrongrelu} for $k_0=2$ and $m=d^3$}. Left: student loss; right: loss divided by squared teacher loss. Lines are \chg{the deterministic equivalent}; open markers \chg{(population-flow samples)} are population spherical-teacher observations at $d=8,16,28$ (32, 32, and 8 draws); see Appendix \ref{app:numerical-scheme} for details. See also uncertainty and numerical approximation quantification in Appendix~\ref{subapp:illustrative-time-curves}. The dotted line marks $T_2$.}
\label{fig:illustrative-time-relative}
\end{figure}


\section{Stopping-time sensitivity and the phase diagram}\label{subsec:phase-diagram}

To examine the effect of the stopping time in the WTS phenomenon, \chg{we vary the harmonic degree at which the student is stopped. For} a $\gamma\in\mathcal K$ with $\gamma > 0$, set
\begin{equation}\label{eq:phase-stop}
T_\gamma=\frac{\log d}{\sigma_\gamma^2}
\begin{cases}\alpha,&\alpha\le k_0,\\2\alpha-k_0,&\alpha>k_0.\end{cases}
\end{equation}
\chg{This is the stopping time \eqref{eq:quadratic-stop} of Theorem~\ref{theo:weaktostrongrelu} with the target degree $k_0$ replaced by $\gamma$, the logarithm evaluated using $m=d^\alpha$,  and the teacher error substituted by its leading asymptotic; $T_{k_0}$ recovers $T$.} Thus $T_\gamma\asymp d^{\gamma+1}\log d$\chg{, the time at which the student has learned the teacher's components up to degree $\gamma$ and not beyond}.

\begin{theo}[Stopping-time phases]\label{theo:phases}
In the pure-target setting of Theorem~\ref{theo:weaktostrongrelu}, \chg{for every stopping degree $\gamma\in\mathcal K$, the leading order of $\overline{\mathcal L}_{\rm ST}(T_\gamma)$ are as given in Figure~\ref{fig:phase-mechanism}.}
\end{theo}

\begin{figure}[htbp]
\centering
\begin{minipage}[c]{0.46\linewidth}
\centering
\begin{tikzpicture}[x=0.36cm,y=0.36cm,font=\footnotesize]
  \def\kz{2}   
  \def\xm{11}  
  \def\ym{11}  
  \fill[phCap!70]   (0,0) rectangle (\kz,\ym);                          
  \fill[phEarly!80] (\kz,0) rectangle (\xm,\kz);                        
  \fill[phQuad!80]  (\kz,\kz) -- (\xm,\kz) -- (\xm,{(\xm+\kz)/2}) -- cycle;   
  \fill[phNonq!80]  (\kz,\kz) -- (\xm,{(\xm+\kz)/2}) -- (\xm,\xm) -- cycle;   
  \fill[phLate!70]  (\kz,\kz) -- (\xm,\xm) -- (\xm,\ym) -- (\kz,\ym) -- cycle; 
  \draw[black!60,dashed] (\kz,\kz) -- (\xm,{(\xm+\kz)/2});
  \draw[black!60] (\kz,\kz) -- (\xm,\xm);
  \draw[black!60] (\kz,0) -- (\kz,\ym);
  \draw[black!60] (\kz,\kz) -- (\xm,\kz);
  \draw[->] (0,0) -- (\xm+0.8,0) node[right] {$\alpha$};
  \draw[->] (0,0) -- (0,\ym+0.8) node[above] {$\gamma$};
  \node[anchor=north] at (\xm/2,-0.9) {teacher capacity $\rightarrow$};
  \node[anchor=south,rotate=90] at (-1.0,\ym/2) {training time $\rightarrow$};
  \draw (\kz,0.15) -- (\kz,-0.15) node[below] {$k_0$};
  \draw (0.15,\kz) -- (-0.15,\kz) node[left] {$k_0$};
  \node[rotate=90] at (\kz/2,\ym/2) {capacity limit};
  \node at ({(\kz+\xm)/2},\kz/2) {early stop};
  \node at (9,3.4) {quadratic};
  \node at (9,6.9) {nonquadratic};
  \node at (5,8.8) {late stop};
  \node[anchor=west,black,inner sep=1pt] at (\xm+0.1,{(\xm+\kz)/2}) {$\gamma=\frac{\alpha+k_0}{2}$};
  \node[anchor=west,black,inner sep=1pt] at (\xm+0.1,\xm) {$\gamma=\alpha$};
  \fill[white] (3,2) circle (0.22); \draw (3,2) circle (0.22) node[above right=1pt] {Fig.~\ref{fig:illustrative-time-relative}};
\end{tikzpicture}
\end{minipage}\hfill
\begin{minipage}[c]{0.52\linewidth}
\footnotesize
\renewcommand{\arraystretch}{1.8}
\setlength{\tabcolsep}{4pt}
\begin{tabular}{@{}l l l l@{}}
\toprule
Phase & Region & $\overline{\mathcal L}_{\rm ST}(T_\gamma)$ & $\overline B / \overline V$  \\
\midrule
\textcolor{phCap}{\rule{1.2ex}{1.2ex}}\ capacity limit & $\alpha\le k_0$ & $\asymp1$ & $\overline B$ \\
\textcolor{phEarly}{\rule{1.2ex}{1.2ex}}\ early stop & $\gamma<k_0$ & $\widetilde\Theta(1)$ & $\overline B$ \\
\textcolor{phQuad}{\rule{1.2ex}{1.2ex}}\ quadratic & $k_0\le\gamma\le\frac{\alpha+k_0}2$ & $\widetilde\Theta(\overline{\mathcal L}_{\rm TE}^{\,2})$ & $\overline B$ \\
\textcolor{phNonq}{\rule{1.2ex}{1.2ex}}\ nonquadratic & $\frac{\alpha+k_0}2<\gamma<\alpha$ & $\widetilde\Theta(\overline{\mathcal L}_{\rm TE}^{\,\rho})$ & $\overline V$ \\
\textcolor{phLate}{\rule{1.2ex}{1.2ex}}\ late stop & $\gamma\ge\alpha$ & $\widetilde\Theta(\overline{\mathcal L}_{\rm TE})$ & $\overline V$ \\
\bottomrule
\end{tabular}
\end{minipage}
\caption{\chg{\textbf{Stopping-time phases (Theorem~\ref{theo:phases}).} Left: the five regions in the plane of teacher capacity exponent $\alpha$ and stopping degree $\gamma$ (drawn for $k_0=2$; admissible degrees form the discrete set $\mathcal K$). The marked point is the setting of Figure~\ref{fig:illustrative-time-relative}. Right: the phases, with the last four for $\alpha>k_0$, where $\overline{\mathcal L}_{\rm TE}\asymp d^{-(\alpha-k_0)}$, and $\rho=(3\alpha-k_0-2\gamma)/(\alpha-k_0)\in(1,2)$. The $\overline{B}/\overline{V}$ indicates which error terms dominates the bias/variance decomposition. Details are in Appendix~\ref{subapp:generaltime}.}}
\label{fig:phase-mechanism}
\end{figure}
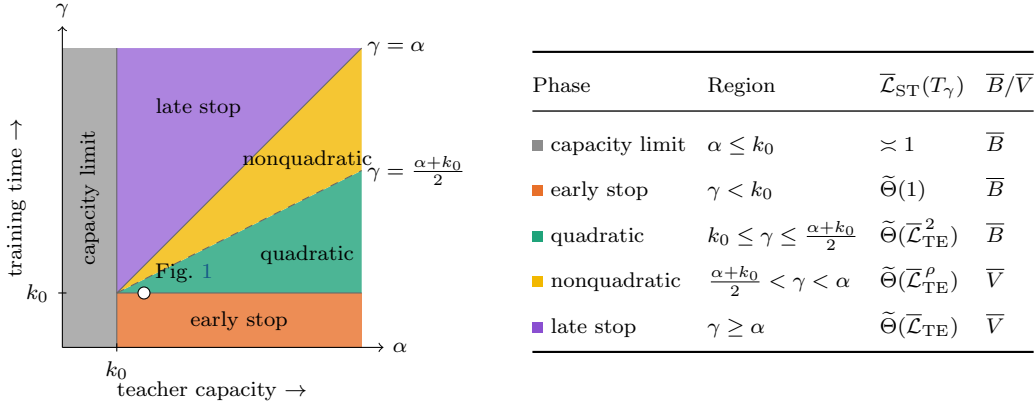

\chg{The five phases of Figure~\ref{fig:phase-mechanism} follow from comparing $\overline B(T_\gamma)$ and $\overline V(T_\gamma)$ (Appendix~\ref{subapp:generaltime}, cf.\ the summary \eqref{eq:student-error-general-stopping}). \phbox{phCap}~If $\alpha\le k_0$ the teacher cannot represent the target, its residual dominates the bias, and both errors are of order one whatever the stop.} \phbox{phEarly}~Before the target scale \chg{($\gamma<k_0$)}, the bias is dominated by the unlearned target and leaves order-one student error even if the teacher error vanishes. \phbox{phQuad}~After the target scale and before the quadratic boundary, $\overline B$ reaches its squared-teacher floor \chg{$\overline{\mathcal L}_{\rm TE}^{\,2}$} while $\overline V$ is smaller\chg{; on the boundary itself the two have the same polynomial order and a factor $(\log d)^2$ can appear}. \phbox{phNonq}~Further training leads the student to learn the variance of the teacher, leading to intermediate improvement exponents $\rho\in(1,2)$. \phbox{phLate}~At late harmonic stops \chg{($\gamma\ge\alpha$)} the variance has teacher-error order\chg{: the student has learned the degrees near the teacher's capacity edge, where the teacher's residual lives, and copies it}. 

%

\section{Mixed targets: common guarantees and incompatible times}\label{sec:mixed-main}

\chg{Since different spectral components of the target are learned on different timescales, for a target with several modes it may be impossible to choose} a single stopping time at which all of them are learned optimally. However, if the teacher has sufficient capacity to learn all spectral components of the target, weak-to-strong generalization is possible for suitable stopping times. 

Consider a general harmonic target with maximum degree $k_{0}$\chg{, and write $f^\star_k$ for its degree-$k$ component, so that $\|f^\star\|_{\mathcal D}^2=\sum_{k\le k_0}\|f^\star_k\|_{\mathcal D}^2$}. Figure \ref{fig:generalregimes} illustrates how $k_{0}$, the teacher capacity exponent $\alpha$ and the stopping-time exponent $\gamma$ give rise to different regimes of the WTS phenomenon. \chg{The regimes are established} in full detail in Appendix~\ref{subapp:gentimetarget}.

\begin{figure}[htbp]
    \centering
\begin{tikzpicture}[x=1cm,y=1cm,font=\footnotesize,
  seg/.style={line width=2.2pt,line cap=round},
  dot/.style={circle,fill=#1,draw=#1,line width=1pt,inner sep=0pt,minimum size=4.5pt}]  
  \def\xa{0}   
  \def\xb{10}  
  \def\ya{1.9}
  \node[anchor=east] at (\xa-0.3,\ya) {\textbf{(a)}\ \ $\alpha\le k_0$};
  \draw[seg,wtsOrange,<-] (\xa,\ya) -- (6.5,\ya);
  \draw[seg,wtsPink,->]  (6.5,\ya) -- (\xb,\ya) node[right,black] {$\gamma$};
  \node[dot=wtsPink] at (6.5,\ya) {};
  \node[above=4pt] at (3.25,\ya) {\textbf{Conditional WTS}};
  \node[above=4pt] at (8.25,\ya) {\textbf{No WTS}};
  \node[below=4pt] at (3.25,\ya) {$\gamma<\alpha$};
  \node[below=4pt] at (6.5,\ya)  {$\alpha$};
  \node[below=4pt] at (8.25,\ya) {$\gamma\ge\alpha$};
  \def\yb{0.4}
  \node[anchor=east] at (\xa-0.3,\yb) {\textbf{(b)}\ \ $\alpha>k_0$};
  \draw[seg,wtsOrange,<-] (\xa,\yb) -- (3.2,\yb);
  \draw[seg,wtsPurple]    (3.2,\yb) -- (6.5,\yb);
  \draw[seg,wtsPink,->]   (6.5,\yb) -- (\xb,\yb) node[right,black] {$\gamma$};
  \node[dot=wtsPurple] at (3.2,\yb) {};
  \node[dot=wtsPink]   at (6.5,\yb) {};
  \node[above=4pt] at (1.6,\yb)  {\textbf{Conditional WTS}};
  \node[above=4pt] at (4.85,\yb) {\textbf{WTS}};
  \node[above=4pt] at (8.25,\yb) {\textbf{No WTS}};
  \node[below=4pt] at (1.6,\yb)  {$\gamma<k_0$};
  \node[below=4pt] at (3.2,\yb)  {$k_0$};
  \node[below=4pt] at (4.85,\yb) {$k_0\le\gamma<\alpha$};
  \node[below=4pt] at (6.5,\yb)  {$\alpha$};
  \node[below=4pt] at (8.25,\yb) {$\gamma\ge\alpha$};
  \def\yl{-0.8}
  \draw[seg,wtsPurple] (0.3,\yl) -- (0.8,\yl);
  \node[anchor=west] at (0.9,\yl) {$\overline{\mathcal{L}}_{\rm ST}\lesssim d^{k_0-\alpha}\overline{\mathcal{L}}_{\rm TE}=o(\overline{\mathcal{L}}_{\rm TE})$};
  \draw[seg,wtsPink] (6.3,\yl) -- (6.8,\yl);
  \node[anchor=west] at (6.9,\yl) {$\overline{\mathcal{L}}_{\rm ST}\asymp\overline{\mathcal{L}}_{\rm TE}$};
\end{tikzpicture} 
    \caption{Weak-to-strong generalization regimes as a function of the stopping time exponent $\gamma$, for (a) $\alpha\leq k_{0}$ and (b) $\alpha>k_{0}$. \chg{\phbox{wtsPurple}~guaranteed WTS, for every target; \phbox{wtsOrange}~conditional WTS, governed by \eqref{eq:wts-condition}; \phbox{wtsPink}~no WTS.} The \phbox{wtsPurple}~interval includes $\gamma=k_{0}$ and excludes $\gamma=\alpha$.}
    \label{fig:generalregimes}
\end{figure}
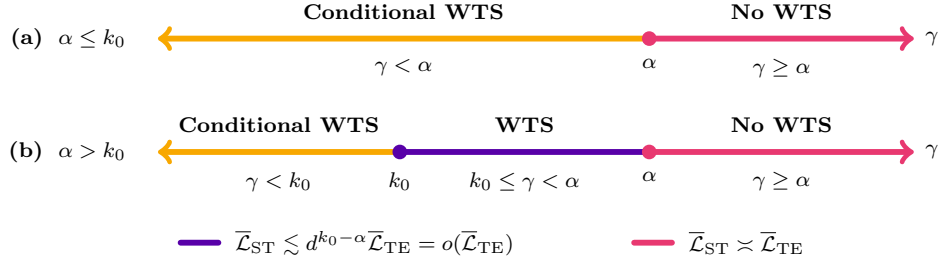

\paragraph{\chg{Guaranteed weak-to-strong generalization.}}
\chg{\phbox{wtsPurple}~If the teacher has enough capacity to learn every mode of the target, $\alpha>k_0$, and the student is stopped after the target degree but before the teacher's capacity edge, $k_0\le\gamma<\alpha$, then}
\begin{equation}\label{eq:mixed-guarantee}
\chg{\overline{\mathcal{L}}_{\rm ST}(T_\gamma)\lesssim d^{k_0-\alpha}\,\overline{\mathcal{L}}_{\rm TE}=o(\overline{\mathcal{L}}_{\rm TE})}
\end{equation}
\chg{(Appendix~\ref{subapp:gentimetarget}, Regime 4). This is weak-to-strong generalization for every mixed target, whatever its coefficients. The bound is sharp: a pure harmonic of degree $k_0$ attains it, since by Theorem~\ref{theo:phases} its student error at these stops is $\overline{\mathcal{L}}_{\rm TE}^2\asymp d^{k_0-\alpha}\overline{\mathcal{L}}_{\rm TE}$, so no better uniform guarantee is possible. The conceptual point is that although the modes of a mixed target are learned on different timescales, the conflict between them does not threaten polynomial improvement: as long as the teacher can learn the target, any stop in the window $[k_0,\alpha)$ delivers the full effect. The window ends at $\alpha$ for the same reason as in Theorem~\ref{theo:phases}. The teacher's residual is concentrated on the degrees near its capacity edge, where $\sigma_k^2\asymp\nu$, and once $\gamma\ge\alpha$ the student has learned those degrees and copies the residual, so $\overline{\mathcal{L}}_{\rm ST}(T_\gamma)\asymp\overline{\mathcal{L}}_{\rm TE}$ for every target (the \phbox{wtsPink}~intervals of Figure~\ref{fig:generalregimes}).}

\paragraph{\chg{Conditional weak-to-strong generalization.}}
\chg{In the \phbox{wtsOrange}~intervals of Figure~\ref{fig:generalregimes}, the stop $\gamma<\alpha$ comes before some degree carrying target energy. Define the held-out energy, the target energy above the stopping degree,}
\begin{equation}\label{eq:wts-target-tail}
    \chg{\mathcal{E}_{>\gamma}}:=\chg{\sum_{k>\gamma}^{k_{0}}\|f^\star_k\|_{\mathcal D}^2\asymp}\sum_{k>\gamma}^{k_{0}}d^{-k-1}\beta_{k}^{2}\chg{,}
\end{equation}
\chg{where $\beta_k^2$ is the squared norm of the degree-$k$ coefficient vector of $\beta$ (Appendix~\ref{subapp:gentimetarget}). This is the part of the target the student cannot have learned by time $T_\gamma$, and it is a floor for its error. Everything else is small: the bias on degrees up to $\gamma$ and the variance inherited from the teacher are both $o(\overline{\mathcal{L}}_{\rm TE})$ because $\gamma<\alpha$. Hence $\overline{\mathcal{L}}_{\rm ST}(T_\gamma)\asymp\mathcal E_{>\gamma}+o(\overline{\mathcal{L}}_{\rm TE})$, and the rule is a simple trade-off: WTS occurs if and only if the energy held out by the stop is negligible against the teacher's own loss}
\begin{equation}\label{eq:wts-condition}
\chg{\mathcal{E}_{>\gamma}=o(\overline{\mathcal{L}}_{\mathrm{TE}}).}
\end{equation}
\chg{The condition is the same in both cases of Figure~\ref{fig:generalregimes}; what changes between them is the form of the teacher error,}
\begin{equation}\label{eq:mixed-teacher-error}
\chg{\overline{\mathcal{L}}_{\rm TE}\asymp\sum_{k\le\alpha}d^{-\alpha-1}\beta_k^2+\sum_{\alpha<k\le k_0}d^{-k-1}\beta_k^2\ \ (\alpha\le k_0),\qquad
\overline{\mathcal{L}}_{\rm TE}\asymp d^{-\alpha-1}\sum_{k\le k_0}\beta_k^2\ \ (\alpha>k_0).}
\end{equation}
\chg{When $\alpha\le k_0$ the teacher learns the degrees up to $\alpha$ and leaves the rest essentially unlearned; when $\alpha>k_0$ it has the same error $d^{-\alpha-1}$ on every degree. In each case \eqref{eq:wts-condition} compares the held-out energy with the teacher error \eqref{eq:mixed-teacher-error} (Appendix~\ref{subapp:gentimetarget}, Remark~\ref{rem:one-condition}).}

\chg{Whether the condition holds depends on the coefficients, and  a two-mode target can illustrate the possible behaviors. Let}
\begin{equation}\label{eq:two-mode-target}
\chg{f_d^\star=\sqrt{1-w_d}\,Y_{2,d}+\sqrt{w_d}\,Y_{4,d},}
\end{equation}
\chg{where $Y_{2,d}$ and $Y_{4,d}$ are orthonormal spherical harmonics of degrees two and four, so $k_0=4$, and $w_d\to0$ is the energy of the small degree-four mode.}

\paragraph{\chg{A vanishing teacher error is not enough.}}
\chg{Take $m=d^3$, so $\alpha=3<k_0=4$ (case (a)), with $w_d=c/d$ for a constant $c>0$. The teacher error tends to zero, but only because $w_d\to0$: the teacher learns the dominant degree-two mode, while the degree-four direction lies above its capacity edge, so almost all of the small degree-four energy $w_d$ stays unlearned. The degree-four mode is then a floor for every student:}
\begin{equation}\label{eq:no-poly}
\chg{\overline{\mathcal L}_{\rm TE}\asymp d^{-1},\qquad
\inf_{t\ge0}\overline{\mathcal L}_{\rm ST}(t)\asymp d^{-1}.}
\end{equation}
Early stopping removes the degree-two error, but the unlearned degree-four energy $w_d\asymp d^{-1}$ remains at every stopping time, so the best student improves on the teacher by at most a constant factor, and no polynomial improvement is possible. In contrast, the guarantee \eqref{eq:mixed-guarantee} needs the teacher to learn every mode, not just to have small error. This is the role of the condition $\alpha>k_0$ in \eqref{eq:mixed-guarantee}.  \chg{Proposition~\ref{prop:mixed-capacity} in Appendix~\ref{app:mixed-examples} proves the bounds, and Figure~\ref{fig:capacity-obstruction} there illustrates their convergence.}

\paragraph{\chg{Incompatible stopping times can rule out quadratic improvement.}}
Take $m=d^5$, \chg{so $\alpha=5>k_0=4$ (case (b)), with} $w_d=Bd^{-5}$ \chg{for a constant $B>0$}. \chg{Now the teacher learns both modes, and each mode on its own would admit quadratic improvement. In the mixture, the time needed to learn the small degree-four mode is also the time at which the student is copying the teacher's degree-four error on the dominant mode. The infimum over $t$ is attained, up to constants, at any time of order $1/\sigma_4^2$, and}
\begin{equation}\label{eq:mixed-main-optimum}
\chg{\overline{\mathcal L}_{\rm TE}\asymp d^{-3},\qquad
\inf_{t\ge0}\overline{\mathcal L}_{\rm ST}(t)
\asymp d^{-5}\asymp \overline{\mathcal L}_{\rm TE}^{5/3}.}
\end{equation}
\chg{The improvement is polynomial, as \eqref{eq:mixed-guarantee} guarantees, but the lower bound holds at every stopping time and rules out quadratic improvement (Proposition~\ref{prop:mixed-timing} in Appendix~\ref{app:mixed-examples}).}

\begin{figure}[htbp]
\centering
\includegraphics[width=\linewidth]{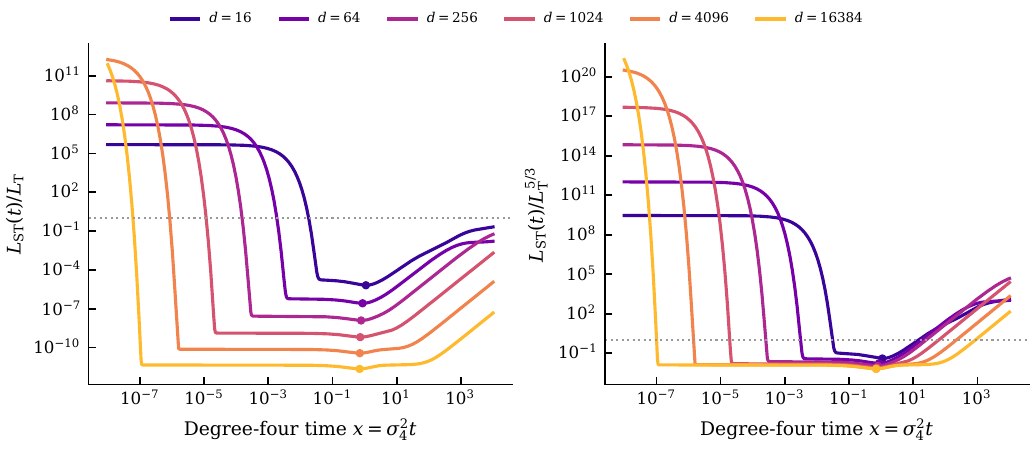}
\caption{\textbf{A timing conflict that prevents quadratic improvement.} \chg{Deterministic-equivalent} curves for the degree-two/degree-four mixture in~\eqref{eq:mixed-main-optimum}\chg{, with $B=C^2/24$ and $C=(12\pi-37)/24$ (Remark~\ref{rem:mixed-constants})}. Left: student-to-teacher loss ratio; right: student loss divided by $\overline{\mathcal L}_{\rm TE}^{5/3}$. Dots are numerically attained minima. The \chg{right-hand curves converge to $C^{1/3}/48$, the leading constant of the normalized optimum in \eqref{eq:mixed-main-optimum} (Remark~\ref{rem:mixed-constants})}.}
\label{fig:mixed-target-time-curves}
\end{figure}

\section{Conclusions and further discussion}\label{sec:conclusions}

Using deterministic equivalents, we have computed the leading-order population errors of a random feature teacher and an early-stopped student, and shown that a pure harmonic target admits optimal quadratic improvement. The mechanism is early stopping: the student learns each spectral component of the target on its own time scale and inherits the teacher's feature variance only on the slower time scale of the teacher's capacity edge, and weak-to-strong generalization is the window between the two. 

One of the key takeaways is that this window is target dependent: different targets require different stopping times for optimal improvement, and for a target with several components, incompatible stopping times can rule out quadratic improvement even when every component would admit it on its own. Nonetheless, the effect is robust provided the teacher has the capacity to learn every component of the target: a fixed gap below capacity guarantees polynomial improvement for a range of stopping times. When the target has a component beyond the teacher's capacity, no stopping time gives more than a constant-factor gain. This illustrates a key limitation of weak-to-strong generalization: the student can suppress the teacher's variance, but it cannot recover what the teacher could not represent.  

We mention some key limitations of our analysis.  The conclusions remain conditional on Gaussian universality and deterministic-equivalent accuracy; controlling that approximation is a mathematical direction for further work.  We have taken the student infinitely wide. A finite student of width $M_{\rm ST}$ adds a second source of feature variance, and one expects a threshold effect: once $M_{\rm ST}$ is large enough relative to the teacher.  Finite student widths admit deterministic equivalent analaysis as well.

\paragraph{AI Use Statement.}
This paper was prepared with the assistance of LLMs. LLMs were used to generate experiment code and figures, restructure and revise the document, generate small proofs and other mathematical artifacts (specifically, the examples and the last three appendices), and verify sections of the paper. The majority of the text was hand-written and all main text was hand-revised.

\bibliographystyle{plainnat}
\bibliography{references}

\clearpage
\appendix
\section*{Supplementary material}

The appendices give the extended related work, the deterministic-equivalent calculations, and the proofs for the ReLU setting. The general bias--variance decomposition follows the resolvent calculations. ReLU spectral and fixed-point estimates precede the error asymptotics; the stopping-time analysis, the two mixed-target examples, and the extension to other activations follow. Numerical details for the main figures are collected at the end.

\begin{center}
\begin{tabular}{@{}p{.08\linewidth}p{.77\linewidth}r@{}}
\toprule
 & Contents & Page\\
\midrule
\ref{app:extended-related} & Extended related work & \pageref{app:extended-related}\\ & Comparison with \citet{medvedev2025weak} & \pageref{subapp:medvedev-comparison}\\[3pt]
\ref{app:detequivalent} & Deterministic equivalents for teacher and student errors & \pageref{app:detequivalent}\\
 & Dyson equation and the Gaussian approximation & \pageref{app:dyson-background}\\
 & Single-resolvent teacher calculation & \pageref{subapp:teacherdeteq}\\
 & Double-resolvent student calculation & \pageref{subapp:studentdeteq}\\
 & Bias--variance decomposition & \pageref{app:mechanism}\\[3pt]
\ref{app:reluproofs} & Proofs for ReLU activation & \pageref{app:reluproofs}\\
 & Harmonic spectrum & \pageref{subapp:asympNksigmak}\\
 & Teacher error scaling & \pageref{app:teacherscaling}\\
  & Student error scaling: quadratic improvement & \pageref{app:studentscaling}\\
\ref{app:generalsettings} & General stopping times and targets & \pageref{app:generalsettings}\\
\ref{app:mixed-examples} & Two mixed-target examples & \pageref{app:mixed-examples}\\
 & A component above capacity: no polynomial improvement & \pageref{subapp:mixed-capacity}\\
 & Incompatible stopping times: no quadratic improvement & \pageref{subapp:mixed-timing}\\
 & Leading constants used in the figures & \pageref{rem:mixed-constants}\\[3pt]
\ref{app:general-activation} & Other activations & \pageref{app:general-activation}\\
 & Harmonic and Hermite coefficients & \pageref{lem:general-coefficients}\\
 & Quadratic improvement for general activations & \pageref{theo:general-activation}\\
 & Examples: ReLU, tanh, logistic sigmoid & \pageref{ex:general-relu}\\[3pt]
\ref{app:numerical-scheme} & Numerical details for the main figures & \pageref{app:numerical-scheme}\\
\bottomrule
\end{tabular}
\end{center}

\section{Extended related work}\label{app:extended-related}

Since \citet{openai2024weaktostrong} first empirical investigation of the WTS phenomenon, many works attempted to do a theoretical analysis of it using simple settings. We discuss below a selection of works representing the range of different approaches used to understand the phenomenon.

\textbf{Data-dependent mechanisms.} First, some works also focus on understanding how data structure influences the WTS phenomenon. \citet{hunter2025theoretical} demonstrate that the gain in performance comes not only from the student's correction of the teacher's mistakes but also from the student's good performance on a portion not covered by the teacher labels. This occurs when one assumes that the strong model's predictions are robust in neighborhoods of pseudo-labeled points (teacher-labeled points) and that bad points, that is, points with incorrect pseudo-labels, have many good neighbors. They also provide empirical evidence for these assumptions in finite samples. \citet{shin2025weaktostrong} define a quantity called \textit{overlap density}, which represents how frequently a dataset contains points with both hard and easy patterns. The weak model learns only the easy patterns, but the strong model is able to learn the hard patterns from the overlap points and thereby generalize to points with only hard patterns. Building on this intuition, they theoretically prove that the WTS phenomenon can be expressed as a function of the overlap density. In the classification setting, \citet{wu2025provable} define weak and strong models based on how well their features represent the data features. Considering a teacher whose labels are asymptotically equivalent to random guessing, they show that, under certain conditions, the student's performance undergoes an asymptotic transition from random guessing to perfect generalization.

\textbf{Misfit error.} \citet{charikar2024quantifying} consider a more general notion of weak and strong models, in which the difference between them is determined by how well their representations capture the underlying data. This can include differences in expressivity, as in our setting, as well as other factors such as the amount of pretraining. They study a regression setting with the squared loss and quantify the WTS gain in terms of the \textit{misfit error}. This quantity is defined as the error of the strong student model on the labels generated by the weak teacher model; equivalently, it measures the squared disagreement between the student and the teacher. Under their assumptions, they show that the error of the strong model evaluated on the ground-truth labels is bounded above by the error of the weak model on the ground-truth labels minus the misfit error. In our notation, this corresponds to $\mathcal{L}_{\mathrm{ST}}\leq \mathcal{L}_{\mathrm{TE}}-\operatorname{misfit}$
Using representation and projection techniques, they show that, globally, the component of the teacher predictor that is not reproduced by the student contributes to the teacher's error relative to the ground-truth labels. Consequently, the larger the misfit, the larger the improvement. \citet{yao2025understanding} establish a corresponding upper bound while also studying lower bounds. Unlike \citet{charikar2024quantifying}, they use a KL-divergence-based loss instead of the squared loss. Other work that studies the \textit{misfit error} includes \citet{mulgund2025relating}, who extend the misfit-based inequality to loss functions given by Bregman divergences. This extension also includes the classification setting, since cross-entropy can be expressed in terms of the KL divergence, which is a Bregman divergence.

\textbf{Representation geometry.} Other works study WTS from a geometric and representation-based perspective. \citet{xue2025representations} quantify the WTS improvement in terms of the overlap between the student principal-component subspace and the orthogonal complement of the teacher principal-component subspace. The principal-component subspace of each model contains the dominant directions of variance in its internal representation. Therefore, this overlap identifies directions that are not captured by the teacher but are well represented by the student, allowing the student to correct some of the teacher's errors. From a related geometric perspective, \citet{dong2025discrepancies} analyze ridgeless regression with intrinsically low-dimensional feature spaces. They show that the student inherits the teacher's variance along directions shared by their representations. In contrast, along directions represented by the teacher but not by the student, this variance is reduced by a factor given by the ratio between the dimension of the student feature space and the number of weakly labeled samples. Thus, differences between the two representation spaces prevent the student from fully inheriting the teacher noise and then improving generalization.

\textbf{Regularization and early-stopping.} With respect to how regularization procedures influence the phenomenon, \citet{moniri2026on} show that, in the setting of linear ridge regression, the student can surpass the teacher by having better regularization. This is similar to the findings of \citet{medvedev2025weak} and our work, but we use implicit regularization imposed by early-stopped gradient flow instead of explicit regularization. They also show that student regularization that is better aligned with the target function improves performance. They also go beyond linear models and frozen representations by considering a pretrained student in a nonlinear setting. The student learns easy, task-specific features from the teacher while retaining hard features acquired during pretraining that the teacher cannot capture, allowing it to outperform the teacher.

\textbf{Beyond linear.} Going beyond linear settings, \citet{oh2025from} use CNNs for both the teacher and the student and adopt yet another characterization of model strength. They characterize it through the activation function: a linear activation for the weak teacher and ReLU for the strong student. Even in this setting, which differs from ours, they find that, in a data-abundant regime, early stopping the student training leads to better performance.

\textbf{Scaling laws.} Finally, regarding the important problem of scaling laws, \chg{\citet{ildiz2025highdimensional} study ridge regression trained on the predictions of a surrogate model in the proportional high-dimensional regime. Using deterministic equivalents for the two-stage procedure, they characterize when the student improves on the surrogate and derive scaling laws in the sample size and model size.} \citet{wu2026improved} show that WTS can improve the scaling law of the test error in a random feature ridge regression setting. Closely related to our work, they also derive a deterministic equivalent for the student test error.

\subsection{\chg{Comparison with the ReLU bounds of \citet{medvedev2025weak}}}\label{subapp:medvedev-comparison}

\chg{Our model and two-step procedure are those of \citet{medvedev2025weak}, so their ReLU bounds can be evaluated in the setting of Theorems~\ref{theo:weaktostrongrelu} and~\ref{theo:phases} and compared term by term with the deterministic equivalent. \citet[Theorem~6.7]{medvedev2025weak} bounds the student error for a two-layer random-feature teacher with any activation, and \citet[Corollary~6.8]{medvedev2025weak} specializes it to ReLU on the sphere, where the eigenvalues come with the multiplicities $N_k$. In our notation (teacher width $m$, harmonic coefficients $\sigma_k$, admissible degrees $\mathcal K$, and $s^+$ the smallest degree in $\mathcal K$ larger than $s$), the corollary reads as follows. Let $f^\star$ be a unit target of degree $k_0$, let the teacher be the population-optimal one, and let $T>0$. Then, for every $s\in\mathcal K$ with $s\ge k_0$, with probability at least $1-2/m-N_{k_0}e^{-cm/N_{k_0}}$ over the teacher features,}
\begin{equation}\label{eq:medvedev-corollary}
\begin{aligned}
\mathcal L_{\rm ST}(T)\le\;
&\frac{4\bigl(\sum_{j\le s}\sigma_j^2N_j\bigr)\bigl(\sum_{j>s}\sigma_j^2N_j\bigr)\log^2 m}{\sigma_s^4\,m}\,\mathcal L_{\rm TE}
&&\Bigr\}\ \text{(A) alignment}\\
&+\sigma_{s^+}^4T^2\,\mathcal L_{\rm TE}
&&\Bigr\}\ \text{(B) learned high degrees}\\
&+\frac{e^{-\sigma_{k_0}^2T}}{2-e^{-\sigma_{k_0}^2T}}\,\|f^\star\|_{\mathcal D}^2
&&\Bigr\}\ \text{(C) unlearned target.}
\end{aligned}
\end{equation}
\chg{Since the bound holds for every admissible $s$ and the probability does not depend on $s$, it also holds with the infimum over $s\in\mathcal K$, $s\ge k_0$, of the right-hand side, which is the form stated in their corollary.}

\chg{\textbf{Evaluation.} Take $m=d^\alpha$ with $\alpha>k_0$, so that $\overline{\mathcal L}_{\rm TE}\asymp d^{-(\alpha-k_0)}$, and stop at the harmonic time $T_\gamma\asymp d^{\gamma+1}\log d$ of Section~\ref{subsec:phase-diagram} with $\gamma\ge k_0$; Theorem~\ref{theo:weaktostrongrelu} is the case $\gamma=k_0$. Using $N_j\asymp d^j$, $\sigma_j^2\asymp d^{-j-1}$ and $\sum_j\sigma_j^2N_j=K(x,x)=1/(2d)$, where every fixed degree carries a $d$-independent share (Appendix~\ref{subapp:asympNksigmak}), the three terms of \eqref{eq:medvedev-corollary} become}

\begin{table}[htbp]
\centering
\footnotesize
\renewcommand{\arraystretch}{1.4}
\setlength{\tabcolsep}{5pt}
\begin{tabular}{@{}l l l@{}}
\toprule
term & at a free $s$ & at the best $s$, namely $s=\gamma$ \\
\midrule
(A) alignment & $d^{\,2s-\alpha}\log^2d\cdot\mathcal L_{\rm TE}$ & $d^{\,2\gamma-\alpha}\log^2d\cdot\mathcal L_{\rm TE}$ \\
(B) learned high degrees & $d^{-2(s^+-\gamma)}\log^2d\cdot\mathcal L_{\rm TE}$ & $d^{-2(\gamma^+-\gamma)}\log^2d\cdot\mathcal L_{\rm TE}$ \\
(C) unlearned target & $\le\mathcal L_{\rm TE}/m=d^{-\alpha}\,\mathcal L_{\rm TE}$ & $\le\mathcal L_{\rm TE}/m=d^{-\alpha}\,\mathcal L_{\rm TE}$ \\
\bottomrule
\end{tabular}
\caption{\chg{The three terms of \eqref{eq:medvedev-corollary} for a unit target of degree $k_0<\alpha$, stopped at $T_\gamma$ with $\gamma\ge k_0$, at a free $s\in\mathcal K$ with $s\ge\gamma$ and at the $s$ that minimizes their sum. Term~(A) does not depend on the stopping time, and term~(C) does not depend on $s$.}}
\label{tab:medvedev-terms}
\end{table}

\chg{\textbf{Choice of $s$.} Term~(A) increases with $s$ and term~(B) decreases with $s$, so the sum is smallest where the two cross. Because $\mathcal K$ is discrete the crossing cannot be attained exactly, and the smallest admissible choice $s=\gamma$ is optimal for every $\alpha\le2\gamma+8$; beyond that a larger $s$ is better. Writing $\mathcal L_{\rm ST}(T_\gamma)\lesssim d^{-r}\log^2d\cdot\mathcal L_{\rm TE}$, the third column gives $r=\min\bigl(\alpha-2\gamma,\,2(\gamma^+-\gamma)\bigr)$, so term~(A) is the binding one exactly when $\alpha\le2\gamma+2(\gamma^+-\gamma)$, which is $\alpha\le2\gamma+4$ for even $\gamma$.}

\textbf{Key differences between \citet{medvedev2025weak} and this work:}
\begin{itemize}
\item \chg{\textbf{Improvement threshold.} At $\gamma=k_0$ the bound \eqref{eq:medvedev-corollary} improves on the teacher only when $\alpha>2k_0$, whereas the deterministic equivalent improves for every $\alpha>k_0$. The same threshold appears as the hypothesis $m\ge\Theta(d^{2k})$ of \citet[Theorem~3.2]{medvedev2025weak}.}
\item \chg{\textbf{Improvement at the quadratic stop.} At $\gamma=k_0$, in the range where term~(A) is rate-limiting, the bound \eqref{eq:medvedev-corollary} gives $\mathcal L_{\rm ST}/\mathcal L_{\rm TE}\lesssim d^{-(\alpha-2k_0)}\log^2d$, whereas Theorem~\ref{theo:weaktostrongrelu} gives $\overline{\mathcal L}_{\rm ST}/\overline{\mathcal L}_{\rm TE}\asymp d^{-(\alpha-k_0)}$.}
\item \chg{\textbf{Best exponent.} Writing $\mathcal L_{\rm ST}\lesssim\mathcal L_{\rm TE}^{\,e}$ with $e=1+r/(\alpha-k_0)$, at $\gamma=k_0$ this is $e=2-k_0/(\alpha-k_0)$ while term~(A) is rate-limiting, at most $5/3$ over all $k_0$ and $\alpha$, and tending to $3/2$ as $\alpha\to\infty$. That limit is the exponent $1.49$ of \citet[Theorem~3.2]{medvedev2025weak}, obtained there for fixed $d$ and $m\to\infty$. Theorem~\ref{theo:weaktostrongrelu} gives $e=2$.}
\item \chg{\textbf{Later stops.} Here $r\le\alpha-2\gamma$, so the bound \eqref{eq:medvedev-corollary} decays monotonically in $\gamma$ and is vacuous once $\gamma\ge\alpha/2$, with no phase structure. Theorem~\ref{theo:phases} gives improvement for every $\gamma<\alpha$, with exponent $\rho\in(1,2]$ and the five phases of Figure~\ref{fig:phase-mechanism}.}
\end{itemize}

\chg{\textbf{Scope.} The comparison is one-sided in an important respect. \citet[Theorem~6.7]{medvedev2025weak} is nonasymptotic, holds with high probability for the actual random-feature teacher without a Gaussian universality assumption, covers any bounded activation, and applies at fixed $d$ with $m\to\infty$, a regime about which our $d\to\infty$ scaling with $m=d^\alpha$ says nothing. Our results describe the deterministic equivalent and transfer to the spherical model under Assumption~\ref{ass:de}. Within the common regime of a pure harmonic target and polynomial teacher width, the deterministic equivalent gives the sharp exponent, the sharp threshold, and the phase diagram; the general bound gives an exponent strictly below the quadratic ceiling that their own lower bound establishes.}

\section{Deterministic equivalents for the teacher and student errors}\label{app:detequivalent}

In this appendix we calculate the deterministic equivalents of the teacher and student population errors. The calculations are general\chg{ and do not depend on the activation}. The specific values of $\mathcal{U},\mathcal{D}, \beta$ and $\Psi$ will be substituted in \chg{Appendix \ref{app:reluproofs}}. Also, to ease notation, we will use $m$ here to denote the number of neurons of the teacher\chg{, $m=M_{\rm TE}$; the student is the infinite-width spectral flow \eqref{studenttrain}}.

\subsection{Background on the Dyson equation and Gaussian approximation}\label{app:dyson-background}

In the random features setting we have two sources of randomness: the data distribution and the random features distribution. It turns out that when calculating the population error, the latter remains as a resolvent of a random covariance matrix. The concept of deterministic equivalents of random matrices can be useful in this case. The idea is to find a deterministic matrix around which this random resolvent concentrates. This deterministic matrix will have the same spectral information as its random counterpart when the dimension is very large and the reduced randomness makes the computations more tractable.

One might consider the expectation matrix as a natural choice of deterministic equivalent but, as argued in \cite{couillet2022random}, this choice is not desirable since it involves an integration over a probability space which can be very challenging numerically. Instead, a more tractable deterministic equivalent can be derived through the associated matrix Dyson equation \citep{erdos2019matrix}. For a Gaussian random matrix $A \in \R^{d \times d}$ the matrix Dyson equation for its resolvent $R(z,A)$ can be derived by applying the Stein's lemma together with the tautological equation 
\begin{equation}\label{eq:resolventtautology}
    R(z;A)(A-zI) = I
\end{equation}
where $I \in \R^{d\times d}$ is the identity matrix, and some concentration properties of the Gaussian distribution. This leads to a self-consistent equation for a scalar quantity, and by some technical stability results, one can show that the Dyson equation has a unique solution which is given in terms of this scalar quantity. 

We illustrate this theory with a related example.  Put $n=d$ and let $A=P+SGS$, where $P=\E[A]$, and $G$ is a standard GOE matrix, i.e, a matrix with independent upper-triangular entries given by
\begin{equation}
    G_{ij} \sim
    \begin{cases}
        \mathcal{N}\left(0,\frac{1}{n}\right), & i<j,\\
        \mathcal{N}\left(0,\frac{2}{n}\right), & i=j,
    \end{cases}
\end{equation}
and $S$ is a symmetric matrix that encodes the variance profile. 

Stein's lemma in matrix format gives
\begin{equation}
    \E[Zf(Z)] = \Sigma\E[\nabla f(Z)]
\end{equation}
where $f$ is a nice test function and $Z$ is a gaussian matrix. 
Define $Z=A-\E[A]$, we have $Z \sim \mathcal{N}(0, \Sigma)$ where $\Sigma\in \mathbb{R}^{n^{2}\times n^{2}}$ has entries
\begin{equation}
    \Sigma_{ij,k\ell}=
    \frac{1}{n}\left[\left(S^{2}\right)_{ik}\left(S^{2}\right)_{j\ell}+\left(S^{2}\right)_{i\ell}\left(S^{2}\right)_{jk}
    \right].
\end{equation}
Applying the lemma for $f(Z)=R(z,A)$ gives, to leading order,
\begin{equation}
    \E[R(z;A)(A-P)] \approx -\E[R(z;A)\mathcal{S}(R(z;A))]
\end{equation}
where
\begin{equation}
    \mathcal{S}(R) \coloneqq \frac1n S^{2}\tr{(RS^{2})}.
\end{equation}
The other GOE contraction contributes $-n^{-1}\E[RS^2RS^2]$, which is of smaller order for bounded $S$ and a spectral parameter bounded away from the real axis. Concentration then motivates factorizing the leading expected product. Writing $\overline R=\E[R(z;A)]$, we obtain
\begin{equation}\label{eq:protodyson}
    \E[R(z;A)(A-P)] \approx -\overline R\mathcal{S}(\overline R).
\end{equation}

Combining this with the expectation of \eqref{eq:resolventtautology} motivates the \textbf{Dyson equation}:
\begin{equation}\label{eq:dysonequation}
    M(P-\mathcal{S}(M)-zI)=I.
\end{equation}
where $M\in\mathbb C^{d\times d}$ is the deterministic equivalent of $R(z;A)$ for a nonreal spectral parameter $z$. This is a background derivation; extending its error control toward the spectrum requires the relevant stability estimates.

We then show that the resolvent $R(z;A)$ is a solution of a perturbed Dyson equation and by stability and uniqueness of the solution, the resolvent is close to the solution of the original one, that is, we ended up showing that $R(z;A)$ is close to $M$ in a suitable sense (see \cite{paquette2024random} for more details). 

Defining the scalar quantity $\kappa(z)\coloneqq n^{-1}\tr{(MS^{2})}$, \eqref{eq:dysonequation} becomes
\begin{equation}\label{eq:detequivalenteq}
    M = (\E[A]-S^{2}\kappa(z)-zI)^{-1}
\end{equation}
and we obtain the scalar self consistent equation
\begin{equation}
    \kappa(z) = \frac1n\tr{(S^{2}(\E[A]-S^{2}\kappa(z)-zI)^{-1})}.
\end{equation}
Thus, to find $M$ we only need to solve this self-consistent equation for $\kappa(z)$ and substitute in \eqref{eq:detequivalenteq}.

In our case, we will be looking for the deterministic equivalent of a feature-feature covariance matrix. Let $X_{1}, X_{2}, ..., X_{n}\in \R^{d}$ be the rows of $X\in \R^{n \times d}$. The corresponding feature-feature covariance matrix is given by
\begin{equation}
    A = \frac{1}{n}X^{T}X \in \R^{d \times d}.
\end{equation}

In this situation, if $X$ has Gaussian entries, the matrix $A$ will depend quadratically on a Gaussian distribution so applying Stein's lemma does not directly yield the linear Dyson form above. It turns out that it is possible to use a linearization trick from free probability (see subsection $10.3$ of \cite{mingo2017free}) to solve this issue and also obtain a Dyson equation for the resolvent of these type of matrices. Again, see \cite{paquette2024random} for details on how to derive the Dyson equation and deterministic equivalent in this case. We use these calculations in the single- and double-resolvent analyses below.



An issue arises for our ReLU random features setting. In this case, both the data and the random features are sampled uniformly from the sphere. The Dyson equation framework, however, is derived using Gaussian integration by parts, or equivalently Stein's lemma. We therefore invoke the \textbf{Gaussian Universality Ansatz} to connect the spherical ReLU model to a Gaussian model with the same relevant covariance structure.  We explain heuristically why one should expect it, and then we discuss related literature where versions of this have been proven rigorously.

\textbf{Cumulant expansion heuristic.} 
Stein's lemma can be generalized to general random variables, in the form of a \emph{cumulant expansion}.  For a random variable $X$ and a smooth enough function $f$
\begin{equation}\label{eq:cumulantexpansion}
    \E[X f(X)]=\kappa_{1}(X)\E[f(X)]+\kappa_{2}(X)\E[f'(X)]+\frac{\kappa_{3}(X)}{2}\E[f''(X)]+\cdots,
\end{equation}
where $\kappa_{i} (X)$ is the $i$-th cumulant of $X$.
Note that since all cumulants for $k>2$ of the Gaussian distribution are zero, \eqref{eq:cumulantexpansion} gives Stein lemma when $X$ is Gaussian. With that, what we are in fact supposing is that the higher cumulants of the random vector $\sigma(U^\top x)$ (for $U$ a matrix of uniform spherical vectors and data vectors $x$) do not influence resolvent statistics; indeed, entries of this vector are affected by large collections of almost independent random variables which lead to central limit theorem effects.


\textbf{Related literature on the Gaussian Universality Ansatz.}
A central work in this area is \citet{hu2023universality}, who prove that, under suitable assumptions in the proportional regime, nonlinear random feature models and Gaussian models with matching covariance have asymptotically equivalent training and generalization errors. Related results were obtained by \citet{montanari2022universality}, who establish covariance universality for the training error and the minimum test error among near empirical risk minimizers. The Gaussian equivalence principle was also studied by \citet{goldt2022gaussian}, who established rigorous conditions for single layer generative models and applied the Gaussian description to pre learned features and kernel methods. The work of \citet{schroder2023deterministic} is particularly close to ours, since it combines Gaussian universality with deterministic equivalents for deep random feature models. They prove universality of the ridge regression test error when the learner and target share the same intermediate layers, and conjecture extensions to more general architectures and convex losses.

\textbf{Related literature on deterministic equivalent approximations.}
Deterministic equivalents can provide quantitatively accurate approximations to population errors. In particular, \citet{defilippis2024dimensionfree} establish a dimension free nonasymptotic multiplicative approximation for random feature ridge regression under concentration and spectral assumptions. Their bounds give vanishing relative errors under suitable scaling of the sample size, number of features and regularization, even when the population error itself vanishes. Moreover, their infinite data limit agrees with our teacher deterministic equivalent. Related nonasymptotic bounds for ridge regression and kernel ridge regression are given by \citet{cheng2024dimension} and \citet{misiakiewicz2024non}, respectively. These results support expecting accurate deterministic approximations in our setting, although they do not directly control our approximation errors. Our analysis considers a two stage learning procedure, where teacher generated labels introduce additional dependencies and cross terms. A closely related result is given by \citet{wu2026improved}, who derive a deterministic equivalent for the student excess error in a teacher student random feature ridge regression model, together with an additive nonasymptotic bound. Their result establishes relative accuracy when the approximation remainder is negligible compared with the student error.


\subsection{Optimally trained teacher}

Since we assume that the teacher is optimally trained, let us first calculate the optimal weight vector. This is a standard calculation, but we repeat it here for completeness. 

\begin{lemma}\label{lemma:optimalweigth}
    Given the setting of Section \ref{sec:setup}, for the squared population loss of the teacher
    \begin{equation}
        \mathcal{L}_{\mathrm{TE}} = \E_{x \sim \mathcal{D}}\left[(f_{\mathrm{teacher}}(x)-f^{*}(x))^2\right],
    \end{equation}
    we have that the optimal weight is
    \begin{equation}
        w_{*} = (U_{m}\Psi U_{m}^{T})^{-1}U_{m}\Psi\beta
    \end{equation}
    where $U_{m}\coloneqq U/\sqrt{m}$, and $U\in \R^{m \times d}$ has the fixed weight vectors $u_{i}\in \R^{d}$ as its rows. With that, we obtain
    \begin{equation}
        \mathcal{L}_{\mathrm{TE}} = \langle U_{m}^{T}(U_{m}\Psi U_{m}^{T})^{-1}U_{m}\Psi\beta - \beta, \Psi[U_{m}^{T}(U_{m}\Psi U_{m}^{T})^{-1}U_{m}\Psi\beta -\beta]\rangle.
    \end{equation}   
\end{lemma}

\begin{proof}
Just plugging in the definitions, we get 
\begin{equation}\label{eq:teacherpopulationerror}
\begin{split}
    \mathcal{L}_{\mathrm{TE}} &= \E_{x \sim \mathcal{D}}\left[(f_{\mathrm{teacher}}(x)-f^{*}(x))^2\right]\\
    &= \E_{x \sim \mathcal{D}}\left[\left(\frac{1}{\sqrt{m}}w^{T}Ux-\beta^{ T}x\right)^2\right]\\
    &= \E_{x \sim \mathcal{D}}\left[\left(\sum_{l=1}^{d}\left(\frac{1}{\sqrt{m}}\sum_{k=1}^{m}w_{k}U_{kl}x_{l}-\beta_{l}x_{l}\right)\right)^{2}\right]\\
    &=\sum_{l=1}^{d}\psi_{l}\left(\frac{1}{\sqrt{m}}\sum_{k=1}^{m}w_{k}U_{kl}-\beta_{l}\right)^2
\end{split}
\end{equation}
where $w \in \R^{m}$ is the vector of trainable weights.

Thus, the optimal weight $w_{*} \in \R^{m}$ is given by
\begin{equation}
    w_{*} \coloneqq \argmin_{w\in \R^{m}} \left\langle \frac{1}{\sqrt{m}}U^Tw-\beta, \Psi\left(\frac{1}{\sqrt{m}}U^{T}w-\beta\right)\right\rangle.
\end{equation}
and it solves
\begin{equation}\label{eq:optimalweigth}
    \nabla_{w}\left\langle \frac{1}{\sqrt{m}}U^Tw-\beta, \Psi\left(\frac{1}{\sqrt{m}}U^{T}w-\beta\right)\right\rangle = 0.
\end{equation}
Using that
\begin{equation}
\begin{split}
    \nabla_{w}\left\langle \frac{1}{\sqrt{m}}U^Tw-\beta, \Psi\left(\frac{1}{\sqrt{m}}U^{T}w-\beta\right)\right\rangle &= \left(\frac{1}{\sqrt{m}}U^{T}w-\beta\right)^{T}(\Psi+\Psi^{T})\nabla_{w}\left(\frac{1}{\sqrt{m}}U^{T}w-\beta\right)\\
    &= 2\frac{1}{\sqrt{m}}U\Psi\left(\frac{1}{\sqrt{m}}U^{T}w-\beta\right), 
\end{split}
\end{equation}
we obtain from  \eqref{eq:optimalweigth} that 
\begin{equation}
     2\frac{1}{\sqrt{m}}U\Psi\left(\frac{1}{\sqrt{m}}U^{T}w_{*}-\beta\right)  = 0,
\end{equation}
such that
\begin{equation}\label{eq:optimalteachweight}
    w_{*} = (U_{m}\Psi U_{m}^{T})^{-1}U_{m}\Psi\beta.
\end{equation}

Plugging this expression into the teacher population error \eqref{eq:teacherpopulationerror} we get
\begin{equation}\label{eq:teacherror}
    \mathcal{L}_{\mathrm{TE}} = \langle U_{m}^{T}(U_{m}\Psi U_{m}^{T})^{-1}U_{m}\Psi\beta - \beta, \Psi[U_{m}^{T}(U_{m}\Psi U_{m}^{T})^{-1}U_{m}\Psi\beta -\beta]\rangle
\end{equation}
as claimed.
\end{proof}

\subsection{Single-resolvent teacher error deterministic equivalent}\label{subapp:teacherdeteq}

The self-consistent equation used below belongs to the classical sample covariance theory of \citet{silverstein1995empirical}; deterministic equivalents for bilinear forms of the resolvent are developed by \citet{hachem2013bilinear}. Following the exposition in the lecture notes of \citet{paquette2024random}, we introduce the deterministic resolvent and calculate its associated teacher error. The connection to the ReLU model uses the Gaussian Universality Ansatz. Theorem \ref{theo:selfconsistenteacher} below evaluates the deterministic quantity; approximation of the random teacher error is a separate question.

Recall that, in the coordinates used in this appendix, $\Psi=\diag(\psi_1,\ldots,\psi_d)$ is the population covariance matrix and $\beta$ is the target coefficient vector, so that $f^{*}(x)=\beta^{T}x$ and $\|f^{*}\|_{L^2(\mathcal{D})}^{2}=\langle\beta,\Psi\beta\rangle$. In particular, $\Psi^{1/2}\beta$ is the target coefficient vector in the corresponding orthonormal basis. Here $m=M_{\mathrm{TE}}$ denotes the teacher width. We use $\Psi+\nu$ as shorthand for $\Psi+\nu I$.

\begin{defn}[Self-consistent parameter]\label{def:teacher-nu}
For $0<m<\operatorname{rank}(\Psi)$, define $\nu>0$ to be the unique solution of
\begin{equation}\label{eq:nuequation}
    m = \tr\bigl(\Psi(\Psi+\nu I)^{-1}\bigr).
\end{equation}
The solution exists and is unique because the right-hand side is continuous and strictly decreasing in $\nu>0$, from $\operatorname{rank}(\Psi)$ to $0$.
\end{defn}

\begin{defn}[Deterministic resolvent]\label{def:deterministic-resolvent}
For the sample covariance matrix
\[
    H=\Psi^{1/2}U_m^{T}U_m\Psi^{1/2},
    \qquad R(z)=(H-zI)^{-1},
\]
define its deterministic resolvent by
\begin{equation}\label{eq:deterministic-resolvent}
    \overline{R}(z)=(-\Psi\kappa(z)-zI)^{-1},
\end{equation}
where $\kappa$ is the analytic solution on the upper half-plane with $\operatorname{Im}\kappa(z)>0$ of
\begin{equation}\label{eq:selfconsistentm}
    -\kappa(z)=\left(1+\frac{1}{m}\tr\bigl(\Psi\overline{R}(z)\bigr)\right)^{-1}.
\end{equation}
We extend this solution to the lower half-plane by complex conjugation. This is the deterministic resolvent associated with the deformed Marchenko--Pastur law (cf.\ Theorem~17 in the notes).
\end{defn}

Let us first rewrite the teacher population error in terms of the resolvent. In fact, by \eqref{eq:teacherror} we have
\[
\begin{split}
    \mathcal{L}_{\mathrm{TE}}
    &= -\langle (U_m\Psi U_m^T)^{-1}U_m\Psi\beta,U_m\Psi\beta\rangle
       +\langle\beta,\Psi\beta\rangle\\
    &= -\left\langle\Psi^{1/2}\beta,
       \left(\sum_{\lambda_j\neq0}v_jv_j^T\right)\Psi^{1/2}\beta\right\rangle
       +\langle\beta,\Psi\beta\rangle\\
    &= \left\langle\Psi^{1/2}\beta,
       \left(\sum_{\lambda_j=0}v_jv_j^T\right)\Psi^{1/2}\beta\right\rangle,
\end{split}
\]
where $\lambda_j$ and $v_j$ are the eigenvalues and an orthonormal basis of eigenvectors of $H$. Here we used that
$\Psi^{1/2}U_m^T(U_m\Psi U_m^T)^{-1}U_m\Psi^{1/2}$ is the orthogonal projector onto the range of $H$. Thus, writing $P_0\coloneqq\sum_{\lambda_j=0}v_jv_j^T$, we obtain
$\mathcal{L}_{\mathrm{TE}}=\langle\Psi^{1/2}\beta,P_0\Psi^{1/2}\beta\rangle$.

Then, if $\Gamma\subset\mathbb C$ is a small positively oriented contour around zero enclosing no nonzero eigenvalues of $H$, we can write
\[
    \mathcal{L}_{\mathrm{TE}}
    =-\frac{1}{2\pi i}\oint_\Gamma
    \langle\Psi^{1/2}\beta,R(z)\Psi^{1/2}\beta\rangle\,dz.
\]
Replacing the resolvent in this expression by its deterministic counterpart, we define the deterministic-equivalent teacher error as
\begin{equation}\label{eq:approxLTE}
    \overline{\mathcal{L}}_{\mathrm{TE}}
    \coloneqq-\frac{1}{2\pi i}\oint_\Gamma
    \langle\Psi^{1/2}\beta,\overline{R}(z)\Psi^{1/2}\beta\rangle\,dz,
\end{equation}
where we now take $\Gamma$ sufficiently small to enclose only the isolated pole at zero of $\overline{R}$. The following theorem gives an explicit expression for this deterministic quantity.

\begin{theo}[Deterministic-equivalent teacher error]
\label{theo:selfconsistenteacher}
With $\nu$ as in Definition~\ref{def:teacher-nu}, the deterministic-equivalent teacher error is
\begin{equation}\label{eq:LTEdetequiv}
    \overline{\mathcal{L}}_{\mathrm{TE}}
    =\nu\langle\Psi^{1/2}\beta,(\Psi+\nu I)^{-1}\Psi^{1/2}\beta\rangle.
\end{equation}
\end{theo}

\begin{proof}
In fact, by \eqref{eq:approxLTE} and Definition~\ref{def:deterministic-resolvent}, we have
\[
    \overline{\mathcal{L}}_{\mathrm{TE}}
    =-\frac{1}{2\pi i}\oint_\Gamma
    \langle\Psi^{1/2}\beta,(-\Psi\kappa(z)-zI)^{-1}\Psi^{1/2}\beta\rangle\,dz.
\]
We claim that $\kappa$ is holomorphic near $0$, with $\kappa(0)=0$, and that
$\kappa(z)=cz+O(z^2)$ with $c=\nu^{-1}$. To see this, let us write $\kappa(z)=zc(z)$. Substituting this into \eqref{eq:selfconsistentm} gives, for $z\neq0$,
\[
\begin{split}
    -zc(z)
    &=\left(1-\frac{1}{mz}\tr\bigl(\Psi(I+c(z)\Psi)^{-1}\bigr)\right)^{-1},\\
    -mzc(z)+c(z)\tr\bigl(\Psi(I+c(z)\Psi)^{-1}\bigr)
    &=m.
\end{split}
\]
Thus, the equation becomes
\[
    m=\tr\bigl(c(z)\Psi(I+c(z)\Psi)^{-1}\bigr)-mzc(z).
\]
At $z=0$, this gives
\[
    m=c(0)\tr\bigl(\Psi(I+c(0)\Psi)^{-1}\bigr),
\]
such that identifying $c(0)=\nu^{-1}$ gives exactly \eqref{eq:nuequation}, as we wanted. To justify the expansion near zero, observe that
\[
    \left.\frac{d}{dc}\tr\bigl(c\Psi(I+c\Psi)^{-1}\bigr)\right|_{c=\nu^{-1}}
    =\tr\bigl(\Psi(I+\nu^{-1}\Psi)^{-2}\bigr)>0.
\]
The analytic implicit function theorem therefore gives a local analytic solution
$c(z)=\nu^{-1}+O(z)$. Since $\kappa(z)=z/\nu+O(z^2)$ has positive imaginary part just above zero, it agrees with the branch in Definition~\ref{def:deterministic-resolvent}. We conclude that $\kappa(z)=cz+O(z^2)$ with $c=\nu^{-1}$, such that
\[
    \overline{R}(z)
    =-\frac{1}{z}(I+\nu^{-1}\Psi)^{-1}+O(1).
\]
Finally, substituting this expansion into \eqref{eq:approxLTE} and integrating in $z$ gives
\[
    \overline{\mathcal{L}}_{\mathrm{TE}}
    =\langle\Psi^{1/2}\beta,(I+\nu^{-1}\Psi)^{-1}\Psi^{1/2}\beta\rangle
    =\nu\langle\Psi^{1/2}\beta,(\Psi+\nu I)^{-1}\Psi^{1/2}\beta\rangle.
\]
as claimed.
\end{proof}

The next lemma gives the deterministic equivalent of the projection operator $P_0$ which also appears during the calculation of the student deterministic equivalent.

\begin{lemma}[Single-resolvent deterministic equivalent]\label{theo:singleresolventapprox}
    Define $\overline{P}_0\coloneqq-\operatorname*{Res}_{z=0}\overline{R}(z)$. Then
    \[
        \overline{P}_0=\nu(\Psi+\nu I)^{-1}=(I+\nu^{-1}\Psi)^{-1}.
    \]
\end{lemma}
\begin{proof}
    This can be seen directly from the expansion obtained in the proof of Theorem~\ref{theo:selfconsistenteacher}. In fact,
    \[
        \overline{R}(z)=-\frac{1}{z}(I+\nu^{-1}\Psi)^{-1}+O(1),
    \]
    such that
    \[
        \overline{P}_0=(I+\nu^{-1}\Psi)^{-1}=\nu(\Psi+\nu I)^{-1},
    \]
    as we wanted.
\end{proof}

\subsection{Double-resolvent student error deterministic equivalent}\label{subapp:studentdeteq}

Since student training follows teacher training in a two-step procedure, the student population error has a more complicated expression that reflects this dependence. For this reason, deriving its deterministic equivalent is more intricate and requires a double-resolvent trick, as we will see below.

\begin{theo}
    Under our setting, if we define $P_{0}\coloneqq \sum_{\lambda_{j}=0}v_{j}v_{j}^{T}\in \R^{d \times d}$ to be the projection operator onto the zero eigenspace of $\Psi^{\frac{1}{2}}U_{m}^{T}(U_{m}\Psi U_{m}^{T})^{-1}U_{m}\Psi^{\frac{1}{2}}$ we can write
    \begin{equation}
        \mathcal{L}_{\mathrm{ST}}(t) = \left\langle \sqrt{\Psi}\beta, [e^{-2\Psi t}-(I-e^{-\Psi t})^{2}P_{0}+P_{0}(I-e^{-2\Psi t})+P_{0}(I-e^{-\Psi t})^{2}P_{0}]\sqrt{\Psi} \beta\right\rangle.
    \end{equation}
\end{theo}

\begin{proof}

From \eqref{studenttrain} we have that the student neural network  for a time $t>0$ of gradient flow is given by
\begin{equation}
    f_{\mathrm{student}}(t)(\cdot) = \sum_{k= 1}^{d}(1-e^{-\lambda_{k} t})\langle f_{{\mathrm{teacher}}}(\cdot),e_{k}\rangle_{\mathcal{D}}e_{k},
\end{equation}
where $\lambda_{k}$ and $e_{k}$ are the eigenvalues and eigenvectors of the population random-feature kernel induced by the activation function. Thus, we can write the student population error at time $t$ as
\begin{equation}
    \begin{split}
        \mathcal{L}_{\mathrm{ST}}(t) &= \E_{x \sim \mathcal{D}}[(f_{{\mathrm{student}}}(t)-f^{*})^{2}]\\
        &=\E_{x \sim \mathcal{D}}\left[\left(\sum_{k= 1}^{d}(1-e^{-\lambda_{k}t})\langle f_{\mathrm{teacher}},e_{k}\rangle e_{k} - f^{*}\right)^{2}\right].
    \end{split}
\end{equation}
Now, writing $f^{*}=\sum_{k= 1}^{d}\langle f^{*},e_{k}\rangle_{\mathcal{D}}e_{k}$ and using the orthonormality of $\{e_{k}\}_{k=1}^
{d}$, we obtain
\begin{equation}
    \begin{split}
        \mathcal{L}_{\mathrm{ST}}(t) &= \E_{x \sim \mathcal{D}}\left[\left(\sum_{k= 1}^{d}\langle (1-e^{-\lambda_{k}t})f_{\mathrm{teacher}} - f^{*}, e_{k} \rangle_{\mathcal{D}} e_{k}\right)^{2}\right]\\
        &= \E_{x \sim \mathcal{D}}\left[\sum_{k= 1}^{d}\langle(1-e^{-\lambda_{k}t})f_{\mathrm{teacher}}-f^{*}, e_{k}\rangle_{D}^{2}\right]\\
        &=\sum_{k \geq 1}(\E_{x \sim \mathcal{D}}[(1-e^{-\lambda_k t})f_{\mathrm{teacher}}-f^{*}]e_{k})^{2}.
    \end{split}
\end{equation}

Replacing
\begin{equation}
    f_{\mathrm{teacher}}(x) = w_{*}U_{m}x \quad\text{ and }\quad f^{*}(x) = \beta^{T}x
\end{equation}
we get
\begin{equation}\label{eq:intermidiateLST}
    \begin{split}
        \mathcal{L}_{\mathrm{ST}}(t) &= \sum_{k=1}^{d}\left(\E_{x \sim \mathcal{D}}\left[\left((1-e^{-\psi_k t})w_{*}U_{m}x-\beta^{T}x\right)\frac{x_{k}}{\sqrt{\psi_{k}}}\right]\right)^{2}\\
        &= \sum_{k=1}^{d}\left(\E_{x \sim \mathcal{D}}\left[\left((1-e^{-\psi_k t})\sum_{l=1}^{d}(w_{*}U_{m})_{l}x_{l}-\sum_{l=1}^{d}\beta_{l}^{T}x_{l}\right)\frac{x_{k}}{\sqrt{\psi_{k}}}\right]\right)^{2}\\
        &=\sum_{k=1}^{d}\left((1-e^{-\psi_{k}t})(w_{*}U_{m})_{k}\sqrt{\psi_{k}}-\beta_{k}\sqrt{\psi_{k}}\right)^{2}\\
        &= \sum_{k=1}^{d}\psi_{k}((1-e^{-\psi_k}t)(w_{*}U_{m})_{k}-\beta_{k})^{2}.
    \end{split}
\end{equation}
Substituting the optimal weight
\begin{equation}
     w_{*} = (U_{m}\Psi U_{m}^{T})^{-1}U_{m}\Psi\beta
\end{equation}
in \eqref{eq:intermidiateLST} gives
\begin{equation}
    \mathcal{L}_{\mathrm{ST}}(t) = \langle (I-e^{-\Psi t})[U_{m}^{T}(U_{m}\Psi U_{m}^{T})^{-1}U_{m}\Psi\beta]-\beta, \Psi\{(I-e^{-\Psi t})[U_{m}^{T}(U_{m}\Psi U_{m}^{T})^{-1}U_{m}\Psi\beta]-\beta\}\rangle.
\end{equation}
Now, we can rewrite this error as
\begin{equation}
    \begin{split}
        \mathcal{L}_{\mathrm{ST}}(t) = &\langle (I-e^{-\Psi t})\sqrt{\Psi}U_{m}^{T}(U_{m}\Psi U_{m}^{T})^{-1}U_{m}\Psi\beta\\
        &(I-e^{-\Psi t})\sqrt{\Psi}U_{m}^{T}(U_{m}\Psi U_{m}^{T})^{-1}U_{m}\Psi\beta)\rangle \\
        -&2\langle (I-e^{-\Psi t})\sqrt{\Psi}U_{m}^{T}(U_{m}\Psi U_{m}^{T})^{-1}U_{m}\Psi\beta,\sqrt{\Psi}\beta \rangle \\
        +&\langle \sqrt{\Psi}\beta,\sqrt{\Psi}\beta \rangle,
    \end{split}
\end{equation}
such that 
\begin{equation}
\begin{split}
    \mathcal{L}_{\mathrm{ST}}(t) &= \left\langle (I-e^{-\Psi t})\left(\sum_{\tilde{\lambda}_{j}\neq 0}\tilde{v}_{j}\tilde{v}_{j}^{T}\right)\sqrt{\Psi}\beta, (I-e^{-\Psi t})\left(\sum_{\tilde{\lambda}_{j}\neq 0}\tilde{v}_{j}\tilde{v}_{j}^ {T}\right)\sqrt{\Psi}\beta\right\rangle\\
    &-2\left\langle (I-e^{-\Psi t})\left(\sum_{\tilde{\lambda}_{j}\neq 0}\tilde{v}_{j}\tilde{v}_{j}^{T}\right)\sqrt{\Psi}\beta, \sqrt{\Psi}\beta\right\rangle + \langle \sqrt{\Psi}\beta, \sqrt{\Psi}\beta\rangle\\
    &= \left\langle (I-e^{-\Psi t})\left(\sum_{\tilde{\lambda}_{j}\neq 0} \tilde{v}_{j}\tilde{v}_{j}^{T}\right)\sqrt{\Psi}\beta, \left[(I-e^{-\Psi t})\left(\sum_{\tilde{\lambda}_{j}\neq 0}\tilde{v}_{j}\tilde{v}_{j}^{T}\right)-2I\right]\sqrt{\Psi}\beta\right\rangle + \langle \sqrt{\Psi}\beta,\sqrt{\Psi}\beta\rangle\\
    &= \left\langle \sqrt{\Psi}\beta, \left[\left(\sum_{\tilde{\lambda}_{j}\neq 0}\tilde{v}_{j}\tilde{v}_{j}^{T}\right)(I-e^{-\Psi t})^{2}\left(\sum_{\tilde{\lambda}_{j}\neq 0}\tilde{v}_{j}\tilde{v}_{j}^{T}\right)-2\left(\sum_{\tilde{\lambda}_{j}\neq 0}\tilde{v}_{j}\tilde{v}_{j}^{T}\right)(I-e^{-\Psi t})+I\right]\sqrt{\Psi}\beta\right\rangle.
\end{split}
\end{equation}
Then, using $P_{0}\coloneqq \sum_{\lambda_{j}=0}v_{j}v_{j}^{T}$, the term in square brackets can be rewritten as follows
\begin{equation}
    \begin{split}
        &(I-P_{0})(I-e^{-\Psi t})^{2}(I-P_{0}) - 2(I-P_{0})(I-e^{-\Psi t}) + I\\
        &= (I-e^{-\Psi t})^{2}(I-P_{0}) - P_{0}(I-e^{-\Psi t})^{2}(I-P_{0}) - 2(I-e^{-\Psi t}) + 2P_{0}(I-e^{-\Psi t}) + I\\
        &= (I-e^{-\Psi t})^{2} - (I-e^{-\Psi t})^{2}P_{0} - P_{0}(I-e^{-\Psi t})^{2} + P_{0}(I-e^{-\Psi t})^{2}P_{0} - 2P_{0}(I-e^{-\Psi t}) + I\\
        &= I - 2e^{-\Psi t} + e^{-2\Psi t} - (I-e^{-\Psi t})^{2}P_{0} - P_{0}(I-2e^{-\Psi t}+e^{-2\Psi t})+ P_{0}(I-e^{-\Psi t})^{2}P_{0}\\
        &- 2(I-e^{-\Psi t}) + 2P_{0}(I-e^{-\Psi t}) + I.
    \end{split}
\end{equation}

Thus, we finally obtain
\begin{equation}\label{eq:studentdeteq}
    \mathcal{L}_{\mathrm{ST}}(t) = \left\langle \sqrt{\Psi}\beta, [e^{-2\Psi t}-(I-e^{-\Psi t})^{2}P_{0}+P_{0}(I-e^{-2\Psi t})+P_{0}(I-e^{-\Psi t})^{2}P_{0}]\sqrt{\Psi} \beta\right\rangle
\end{equation}
as we wanted.
\end{proof}

For the terms in \eqref{eq:studentdeteq} that depend on a single $P_0$, we use the deterministic representative $\overline{P}_0$ provided by Lemma \ref{theo:singleresolventapprox}. For the last term, we will need a double-resolvent calculation. The idea is to insert the matrix $S\coloneqq(I-e^{-\Psi t})^2$ by differentiating a perturbed resolvent. We first define the corresponding deterministic quantities, and then calculate them.

An alternative point of view is provided by the eigenlearning framework, where second moments of a learning transfer operator are computed by spectral differentiation \citep[Appendix~I.9]{simon2023eigenlearning}; related resolvent differentiation identities are used for random feature regression in \citet[Appendix~E.3]{simon2024more}.

\begin{defn}[Perturbed deterministic resolvent]\label{def:perturbed-resolvent}
For $S=(I-e^{-\Psi t})^2$, let
\[
    \overline{R}_\epsilon(z)
    =(\epsilon S-zI-\Psi\kappa_\epsilon(z))^{-1},
    \qquad
    -\kappa_\epsilon(z)
    =\left(1+\frac{1}{m}\tr\bigl(\Psi\overline{R}_\epsilon(z)\bigr)\right)^{-1}.
\]
For $z$ off the deterministic spectrum, we take the branch analytic in $\epsilon$ near zero that continues Definition~\ref{def:deterministic-resolvent}, so that $\overline{R}_0=\overline{R}$ and $\kappa_0=\kappa$. Define
\begin{equation}\label{eq:double-resolvent-definition}
    \overline{Q}_S(z)
    \coloneqq-\left.\frac{\partial\overline{R}_\epsilon(z)}{\partial\epsilon}\right|_{\epsilon=0},
    \qquad
    \overline{\mathcal{P}}(S)
    \coloneqq\frac{1}{2\pi i}\oint_\Gamma z\,\overline{Q}_S(z)\,dz,
\end{equation}
where $\Gamma$ is a sufficiently small positively oriented contour around zero. These are the deterministic counterparts of $R(z)SR(z)$ and $P_0SP_0$, respectively.
\end{defn}

The following theorem evaluates $\overline{\mathcal{P}}(S)$.

\begin{theo}[Double-resolvent deterministic equivalent]\label{theo:doubleresolventapprox}
For $S=(I-e^{-\Psi t})^2$, we have
\begin{equation}
    \overline{\mathcal{P}}(S) = \frac{(I-e^{-\Psi t})^{2}}{(I+\Psi\nu^{-1})^{2}} + \frac{\frac{\Psi}{(I+\Psi\nu^{-1})^{2}}\tr{\left(\frac{\nu^{-2}\Psi(I-e^{-\Psi t})^{2}}{(I+\Psi\nu^{-1})^{2}}\right)}}{m-\tr{\left(\frac{\Psi^{2}\nu^{-2}}{(I+\Psi\nu^{-1})^{2}}\right)}}
\end{equation}
where $\nu$ is the solution of \eqref{eq:nuequation}.
\end{theo}

\begin{proof}
        For $\epsilon$ near zero, let us define the following resolvent
    \begin{equation}
        R_{\epsilon}(z)\coloneqq\left(\frac{1}{m}\Psi^{\frac{1}{2}}U^{T}U\Psi^{\frac{1}{2}}+\epsilon S-zI\right)^{-1}
    \end{equation}
    where $S \coloneqq (I-e^{-\Psi t})^{2}$.
    
    Observe that 
    \begin{equation}\label{eq:derivativein0}
        \partial_{\epsilon}R_{\epsilon}(z)|_{\epsilon=0} = -R_{0}(z)SR_{0}(z),
    \end{equation}
    on the other hand, we also have that
    \begin{equation}
        R_{0}(z) = \sum_{j}\frac{v_{j}v_{j}^{T}}{\lambda_{j}-z}
    \end{equation}
    where $\lambda_{j}$ are the eigenvalues of $H=\frac{1}{m}\Psi^{\frac{1}{2}}U^{T}U\Psi^{\frac{1}{2}}$ and $v_{j}$ are the associated eigenvectors. Thus
    \begin{equation}
    \begin{split}
        R_{0}(z)SR_{0}(z) &= \left(\sum_{j}\frac{v_{j}v_{j}^{T}}{\lambda_{j}-z}\right)S\left(\sum_{k}\frac{v_{k}v_{k}^{T}}{\lambda_{k}-z}\right)\\
        &= \sum_{j,k}(v_{j}v_{j}^{T}Sv_{k}v_{k}^{T})\frac{1}{z^{2}}\left(\frac{1}{1-(\lambda_{j}+\lambda_{k})\frac{1}{z}+\frac{\lambda_{j}\lambda_{k}}{z^{2}}}\right).
    \end{split}
    \end{equation}
    Defining $f(z)=zR_{0}(z)SR_{0}(z)$ and $\Gamma$ a small positively oriented contour around 0 enclosing no nonzero eigenvalues of $H$, consider the integral
    \begin{equation}
        \frac{1}{2\pi i}\oint_{\Gamma}f(z)dz =\frac{1}{2\pi i} \oint_{\Gamma}\sum_{j,k}\frac{v_{j}v_{j}^{T}Sv_{k}v_{k}^{T}}{z}\left(\frac{1}{1-(\lambda_{j}+\lambda_{k})\frac{1}{z}+\frac{\lambda_{j}\lambda_{k}}{z^{2}}}\right)dz  .
    \end{equation}
    If at least one of $\lambda_j$ and $\lambda_k$ is nonzero, the corresponding summand of $f(z)$ is holomorphic inside $\Gamma$ and by Cauchy's theorem its integral is $0$. Thus, only the terms with $\lambda_j=\lambda_k=0$ remain, and we have
    \begin{equation}
        \frac{1}{2\pi i}\oint \sum_{\substack{\lambda_{j}=0 \\ \lambda_{k}=0}}\frac{v_{j}v_{j}^{T}Sv_{k}v_{k}^{T}}{z}dz = P_{0}SP_{0}
    \end{equation}
    since
    \begin{equation}
        \frac{1}{2\pi i}\oint_{\Gamma}\frac{1}{z}dz=1.
    \end{equation}
    Then, we get
    \begin{equation}\label{eq:P0SP0}
        \frac{1}{2\pi i}\oint_{\Gamma}zR_{0}(z)SR_{0}(z)dz = P_{0}SP_{0}.
    \end{equation}

    We now use a linearization procedure to recover the Dyson equation defining $\overline{R}_\epsilon(z)$. We define the larger matrix
    \begin{equation}
        L \coloneqq \begin{pmatrix}
            \epsilon S - zI_{d} & \frac{1}{\sqrt{m}}\Psi^{\frac{1}{2}}U^{T}\\
            \frac{1}{\sqrt{m}}U\Psi^{\frac{1}{2}} & -I_{m}
        \end{pmatrix} \in \mathbb{C}^{(d+m)\times(d+m)}.
    \end{equation}
    Using the Schur complement formula, we obtain
    \begin{equation}
    (L^{-1})_{11} = \left(\frac{1}{m}\Psi^{\frac{1}{2}}U^{T}U\Psi^{\frac{1}{2}}+\epsilon S -zI\right)^{-1}.
    \end{equation}
    Now, by the Dyson equation
    \begin{equation}
    \begin{split}
        \begin{pmatrix}
            \mathcal{M}_{11} & 0\\
            0 & \mathcal{M}_{22}
        \end{pmatrix}\left(\begin{pmatrix}
            \epsilon S - zI & 0\\
            0 & -I_{m}
        \end{pmatrix}-\begin{pmatrix}
            \frac{1}{m}\Psi \tr{(\mathcal{M}_{22})} & 0\\
            0 & \frac{1}{m}\tr{(\mathcal{M}_{11}\Psi)}I_{m}
        \end{pmatrix}\right) =\begin{pmatrix}
            I & 0\\
            0 & I_{m}
        \end{pmatrix},
    \end{split}
    \end{equation}
    we get
    \begin{equation}\label{eq:mathcaldef}
    \begin{split}
        \mathcal{M}_{11} &= \left(\epsilon S - zI -\frac{1}{m}\Psi\tr{(\mathcal{M}_{22})}\right)^{-1}\\
        \mathcal{M}_{22} &= \left(-1-\frac{1}{m}\tr{(\mathcal{M}_{11}\Psi)}\right)^{-1}I_m
    \end{split}
    \end{equation}
    and $\mathcal{M}_{11}=\overline{R}_\epsilon(z)$ is exactly the matrix in Definition~\ref{def:perturbed-resolvent}.
    In our case, $S=(I-e^{-\Psi t})^{2}$ and $\Psi$ share the same orthonormal basis of eigenvectors $\{w_{j}\}$. Let $\lambda_{j}^{S}$ and $\lambda_{j}^{\Psi}$ be the associated eigenvalues, then
    \begin{equation}
    \begin{split}
        \langle w_{j}, \mathcal{M}_{11}w_{j}\rangle = \left(\epsilon\lambda_{j}^{S}-z-\frac{\lambda_j^\Psi}{m}\tr(\mathcal{M}_{22})\right)^{-1}\\
        \mathcal{M}_{22} = \left(-1-\frac{1}{m}\tr{(\mathcal{M}_{11}\Psi)}\right)^{-1}I_m.
    \end{split}
    \end{equation}
    To ease notation in the calculation below, write $\kappa(z)=\kappa_\epsilon(z)$, suppressing its dependence on $\epsilon$. Then
    \begin{equation}
        \kappa(z) = \frac{1}{m}\tr{(\mathcal{M}_{22})},
    \end{equation}
    then 
    \begin{equation}\label{eq:mathcal11}
        \langle w_{j}, \mathcal{M}_{11} w_{j}\rangle = \frac{1}{\epsilon \lambda_{j}^{S}-z-\lambda_{j}^{\Psi}\kappa(z)}.
    \end{equation}
    Note that 
    \begin{equation}
        \tr{(\mathcal{M}_{22})} = m\left(-1-\frac{1}{m}\tr{(\mathcal{M}_{11}\Psi)}\right)^{-1}
    \end{equation}
    such that by \eqref{eq:mathcaldef} we have
    \begin{equation}
        \kappa(z)= \frac{1}{-1-\frac{1}{m}\sum_{j=1}^{d}\langle w_{j},\mathcal{M}_{11}w_{j}\rangle \lambda_{j}^{\Psi}}.
    \end{equation}
    Also, from \eqref{eq:mathcal11} we get
    \begin{equation}
        \mathcal{M}_{11} =\sum_{j=1}^{d}w_{j}w_{j}^{T}\left(\frac{1}{\epsilon \lambda_{j}^{S}-z-\lambda_{j}^{\Psi}\kappa(z)}\right),
    \end{equation}
    such that we have the following scalar self-consistent equation
    \begin{equation}
        \kappa(z) = \frac{1}{-1-\frac{1}{m}\sum_{j=1}^{d}\frac{\lambda_{j}^{\Psi}}{\epsilon\lambda_{j}^{S}-z-\lambda_{j}^{\Psi}\kappa(z)}}.
    \end{equation}
    The exact identities \eqref{eq:derivativein0} and \eqref{eq:P0SP0} motivate the definitions of $\overline{Q}_S$ and $\overline{\mathcal{P}}(S)$. Since $\mathcal{M}_{11}=\overline{R}_\epsilon$, Definition~\ref{def:perturbed-resolvent} gives
    \begin{equation}
        \overline{\mathcal{P}}(S) = -\frac{1}{2\pi i} \oint_{\Gamma}z \left.\frac{\partial \mathcal{M}_{11}}{\partial \epsilon}\right|_{\epsilon=0}\;dz.
    \end{equation}
    Calculating, we get
    \begin{equation}
        \frac{\partial\mathcal{M}_{11}}{\partial\epsilon} =-\sum_{j=1}^{d}\frac{w_{j}w_{j}^{T}}{(\epsilon\lambda_{j}^{S}-z-\lambda_{j}^{\Psi}\kappa(z))^{2}}\left(\lambda_{j}^{S}-\lambda_{j}^{\Psi}\frac{\partial \kappa(z)}{\partial \epsilon}\right).
    \end{equation}
    Also,
    \begin{equation}
    \begin{split}
        \frac{\partial \kappa(z)}{\partial \epsilon} = \frac{-1}{\left(-1-\frac{1}{m}\sum_{j=1}^{d}\frac{\lambda_{j}^{\Psi}}{\epsilon\lambda_{j}^{S}-z-\lambda_{j}^{\Psi}\kappa(z)}\right)^{2}}\left(\frac{1}{m}\sum_{j=1}^{d}\frac{\lambda_{j}^{\Psi}(\lambda_{j}^{S}-\lambda_{j}^{\Psi}\frac{\partial \kappa(z)}{\partial \epsilon})}{(\epsilon\lambda_{j}^{S}-z-\lambda_{j}^{\Psi}\kappa(z))^{2}}\right),
    \end{split}
    \end{equation}
    such that 
    \begin{equation}
        \left.\frac{\partial \kappa(z)}{\partial \epsilon}\right|_{\epsilon=0} =\left. \frac{-\frac{1}{m}\tr{\left(\frac{\kappa(z)^2\Psi S}{(z+\Psi\kappa(z))^{2}}\right)}}{1-\frac{1}{m}\tr{\left(\frac{\Psi^{2}\kappa(z)^{2}}{(z+\Psi\kappa(z))^{2}}\right)}}\right|_{\epsilon=0}.
    \end{equation}
    Then    \begin{equation}\label{eq:finaldoubleresolvent}
            \overline{\mathcal{P}}(S)
            = \frac{1}{2\pi i}\oint_{\Gamma}\sum_{j=1}^{d}\frac{zw_{j}w_{j}^{T}}{(z+\lambda_{j}^{\Psi}\kappa(z))^{2}}\left[\lambda_{j}^{S}+\frac{\frac{\lambda_{j}^{\Psi}}{m}\tr{\left(\frac{\kappa(z)^{2}\Psi (I-e^{-\Psi t})^{2}}{(z+\Psi\kappa(z))^{2}}\right)}}{1-\frac{1}{m}\tr{\left(\frac{\Psi^{2}\kappa(z)^{2}}{(z+\Psi\kappa(z))^{2}}\right)}}\right]_{\epsilon=0} \;dz.
    \end{equation}
    At $\epsilon=0$, the proof of Theorem \ref{theo:selfconsistenteacher} gives $\kappa(z)=cz+O(z^2)$ with $c=\nu^{-1}$. Observe also that the denominator at zero is strictly positive:
    \[
        m-\tr\left(\frac{c^2\Psi^2}{(I+c\Psi)^2}\right)
        =\tr\left(\frac{c\Psi}{(I+c\Psi)^2}\right)>0,
    \]
    where we used \eqref{eq:nuequation}. Thus, keeping the coefficient of $z^{-1}$ in the integrand, we can rewrite \eqref{eq:finaldoubleresolvent} as follows 
    \begin{equation}\label{eq:double-resolvent-two-term}
        \begin{split}
            \overline{\mathcal{P}}(S)&= \frac{1}{2\pi i}\oint_{\Gamma}\sum_{j=1}^{d}\frac{w_{j}w_{j}^{T}}{z(1+\Psi_{j}c)^{2}} \left[(1-e^{-\Psi_{j}t})^{2}+\frac{\frac{\Psi_{j}}{m}\tr{\left(\frac{c^2\Psi(I-e^{-\Psi t})^{2}}{(I+\Psi c)^{2}}\right)}}{1-\frac{1}{m}\tr{\left(\frac{\Psi^{2}c^{2}}{(I+\Psi c)^{2}}\right)}}\right]dz\\
            &= \frac{(I-e^{-\Psi t})^{2}}{(I+\Psi c)^{2}} + \frac{\frac{\Psi}{(I+\Psi c)^{2}}\tr{\left(\frac{c^{2}\Psi(I-e^{-\Psi t})^{2}}{(I+\Psi c)^{2}}\right)}}{m-\tr{\left(\frac{\Psi^{2}c^{2}}{(I+\Psi c)^{2}}\right)}}
        \end{split}
    \end{equation}
as claimed.
\end{proof}

The two terms in Theorem~\ref{theo:doubleresolventapprox} come from differentiating the explicit perturbation $\epsilon S$ and the self-consistent parameter $\kappa_\epsilon$, respectively. We can now combine this calculation with the single-resolvent formula to define the deterministic-equivalent student error.

\begin{defn}[Deterministic-equivalent student error]\label{def:student-error-de}
Let $S_t=(I-e^{-\Psi t})^2$. Replacing the single-projector terms and the double-projector term in \eqref{eq:studentdeteq} by their deterministic counterparts, define
\begin{equation}\label{eq:student-error-de}
    \overline{\mathcal{L}}_{\mathrm{ST}}(t)\coloneqq\big\langle\Psi^{1/2}\beta,\big[e^{-2\Psi t}-S_t\overline{P}_0+\overline{P}_0(I-e^{-2\Psi t})+\overline{\mathcal{P}}(S_t)\big]\Psi^{1/2}\beta\big\rangle,
\end{equation}
where $\overline{P}_0=\nu(\Psi+\nu I)^{-1}$ and $\overline{\mathcal{P}}(S_t)$ is evaluated in Theorem~\ref{theo:doubleresolventapprox}.
\end{defn}

\subsection{Bias--variance decomposition}\label{app:mechanism}

We first verify \eqref{eq:main-two-term} directly from Definition~\ref{def:student-error-de}, then give its moment interpretation.

\paragraph{Direct algebraic verification.}
The fixed-point equation gives
\[
\begin{aligned}
m-\operatorname{tr}[\Psi^2(\Psi+\nu I)^{-2}]
&=\operatorname{tr}[\Psi(\Psi+\nu I)^{-1}-\Psi^2(\Psi+\nu I)^{-2}]\\
&=\nu\operatorname{tr}[\Psi(\Psi+\nu I)^{-2}]>0.
\end{aligned}
\]
Using $(I+\Psi\nu^{-1})^{-2}=\nu^2(\Psi+\nu I)^{-2}$, Theorem~\ref{theo:doubleresolventapprox} \chg{evaluated at $S=S_t=(I-e^{-\Psi t})^2$} therefore yields
\[
\begin{aligned}
\overline{\mathcal P}(\chg{S_t})
={}&\nu^2(\Psi+\nu I)^{-2}(I-e^{-\Psi t})^2\\
&+\nu\Psi(\Psi+\nu I)^{-2}
\frac{\operatorname{tr}[\Psi(\Psi+\nu I)^{-2}(I-e^{-\Psi t})^2]}
{\operatorname{tr}[\Psi(\Psi+\nu I)^{-2}]}.
\end{aligned}
\]
The second term gives exactly $\overline V(t)$ after taking the quadratic form against $\Psi^{1/2}\beta$. The remaining terms in \eqref{eq:student-error-de} combine using
\[
I-e^{-2\Psi t}-(I-e^{-\Psi t})^2=2e^{-\Psi t}(I-e^{-\Psi t}).
\]
Since all the matrices below are functions of $\Psi$ and hence commute, their sum is
\[
\begin{aligned}
&e^{-2\Psi t}+2\nu e^{-\Psi t}(I-e^{-\Psi t})(\Psi+\nu I)^{-1}
+\nu^2(I-e^{-\Psi t})^2(\Psi+\nu I)^{-2}\\
&\quad=\bigl[e^{-\Psi t}+\nu(I-e^{-\Psi t})(\Psi+\nu I)^{-1}\bigr]^2\\
&\quad=\bigl[(e^{-\Psi t}(\Psi+\nu I)+\nu(I-e^{-\Psi t}))(\Psi+\nu I)^{-1}\bigr]^2\\
&\quad=(\nu I+\Psi e^{-\Psi t})^2(\Psi+\nu I)^{-2}.
\end{aligned}
\]
Its quadratic form against $\Psi^{1/2}\beta$ is
\[
\left\|(\nu I+\Psi e^{-\Psi t})(\Psi+\nu I)^{-1}\Psi^{1/2}\beta\right\|^2
=\overline B(t),
\]
which proves \eqref{eq:main-two-term}.

\paragraph{Moment interpretation.}
Here the bias and variance refer to the student prediction as an estimator of $f^*$, with respect to the randomness of the teacher features $\mathcal{U}$. The target and the time $t$ are fixed, while different draws of $\mathcal{U}$ produce different teacher labels and hence different student predictions. Recall that $P_0$ is the orthogonal projector onto the nullspace of $H$, and that the student residual is $[e^{-\Psi t}+(I-e^{-\Psi t})P_0]\Psi^{1/2}\beta$. Its squared Euclidean norm equals the squared $L^2(\mathcal{D})$ norm of $f^*-f_{\mathrm{student}}(t)$. Averaging over the teacher features gives the exact identity
\begin{equation}
\begin{split}
&\mathbb E_{\mathcal{U}}\|[e^{-\Psi t}+(I-e^{-\Psi t})P_0]\Psi^{1/2}\beta\|^2\\
&\quad=\|[e^{-\Psi t}+(I-e^{-\Psi t})\mathbb E_{\mathcal{U}}[P_0]]\Psi^{1/2}\beta\|^2\\
&\qquad+\mathbb E_{\mathcal{U}}\|(I-e^{-\Psi t})(P_0-\mathbb E_{\mathcal{U}}[P_0])\Psi^{1/2}\beta\|^2.
\end{split}
\end{equation}
The first term is the squared norm of the average residual. Since $f^*$ is deterministic,
\begin{equation}
\begin{split}
&\|[e^{-\Psi t}+(I-e^{-\Psi t})\mathbb E_{\mathcal{U}}[P_0]]\Psi^{1/2}\beta\|^2\\
&\quad=\|f^*-\mathbb E_{\mathcal{U}}[f_{\mathrm{student}}(t)]\|_{\mathcal{D}}^2\\
&\quad=\mathbb E_{x\sim\mathcal{D}}\left[\left(\mathbb E_{\mathcal{U}}[f_{\mathrm{student}}(t)(x)]-f^*(x)\right)^2\right].
\end{split}
\end{equation}
For each fixed $x$, the difference $\mathbb E_{\mathcal{U}}[f_{\mathrm{student}}(t)(x)]-f^*(x)$ is the bias of the student prediction as an estimator of $f^*(x)$. Thus, the first term is the squared bias averaged over the input distribution.

The second term is the expected squared norm of the residual minus its mean. Centering the residual cancels the target:
\begin{equation}
\begin{split}
&(f^*-f_{\mathrm{student}}(t))-\mathbb E_{\mathcal{U}}[f^*-f_{\mathrm{student}}(t)]\\
&\quad=\mathbb E_{\mathcal{U}}[f_{\mathrm{student}}(t)]-f_{\mathrm{student}}(t).
\end{split}
\end{equation}
Consequently,
\begin{equation}
\begin{split}
&\mathbb E_{\mathcal{U}}\|(I-e^{-\Psi t})(P_0-\mathbb E_{\mathcal{U}}[P_0])\Psi^{1/2}\beta\|^2\\
&\quad=\mathbb E_{\mathcal{U}}\|f_{\mathrm{student}}(t)-\mathbb E_{\mathcal{U}}[f_{\mathrm{student}}(t)]\|_{\mathcal{D}}^2\\
&\quad=\mathbb E_{x\sim\mathcal{D}}\left[\operatorname{Var}_{\mathcal{U}}(f_{\mathrm{student}}(t)(x))\right].
\end{split}
\end{equation}
Thus, the second term is the variance of the student prediction across draws of the teacher features, averaged over the input distribution.

To identify these moments with $\overline B(t)$ and $\overline V(t)$, the deterministic equivalents must approximate the relevant first and second moments at the claimed asymptotic scale. Under this additional control, replacing $\mathbb E_{\mathcal{U}}[P_0]$ by $\nu(\Psi+\nu I)^{-1}$ in the squared mean residual gives $\overline B(t)$ by the algebra above.

For the variance, we first write the exact identity
\begin{equation}
\begin{split}
&\mathbb E_{\mathcal{U}}\|(I-e^{-\Psi t})(P_0-\mathbb E_{\mathcal{U}}[P_0])\Psi^{1/2}\beta\|^2\\
&\quad=\mathbb E_{\mathcal{U}}\langle\Psi^{1/2}\beta,P_0(I-e^{-\Psi t})^2P_0\Psi^{1/2}\beta\rangle\\
&\qquad-\|(I-e^{-\Psi t})\mathbb E_{\mathcal{U}}[P_0]\Psi^{1/2}\beta\|^2.
\end{split}
\end{equation}
The first term on the right-hand side is approximated by the quadratic form of $\overline{\mathcal P}(S_t)$, against $\Psi^{1/2}\beta$. Replacing $\mathbb E_{\mathcal{U}}[P_0]$ by $\nu(\Psi+\nu I)^{-1}$ in the subtracted squared mean gives precisely the quadratic form of the first term of $\overline{\mathcal P}(S_t)$. These contributions therefore cancel, leaving the quadratic form of its second term, which is $\overline V(t)$. As above, this identification requires the relevant moment approximations; convergence in probability of the loss alone does not establish them.

\section{Proofs for ReLU activation}\label{app:reluproofs} 

\paragraph{Scope.} The asymptotic loss calculations in this appendix concern the deterministic equivalents. Transfer to the random spherical losses uses Assumption~\ref{ass:de}. \chg{Lemma~\ref{lem:nu-relu} gives the order of $\nu$, which Appendix~\ref{app:mixed-examples} reuses for the two mixed-target examples.}

\chg{The kernel \eqref{eq:relukernel} contains the degree-zero mode, a single direction with $N_0=1$ and coefficient $\sigma_0=\mu_d=\E\,(\langle u,x\rangle)_+=\Gamma(d/2)/(2\sqrt\pi\,\Gamma((d+1)/2))\asymp d^{-1/2}$, the population mean of a feature. It is included in the sums and traces below, which run over $k\in\mathcal{K}=\{0,1\}\cup2\mathbb{N}$.}

In this appendix we calculate the deterministic equivalents of the teacher and student model in the setting of Theorem \ref{theo:weaktostrongrelu}. From \eqref{eq:relukernel}, we see that the important quantities to calculate are the ReLU coefficients in spherical harmonics $\sigma_{k}$ and the dimension of the space of spherical harmonics of degree $k$, namely $N_{k}$. Thus, we start calculating their the asymptotics.

\subsection{Asymptotics of $N_{k}$ and $\sigma_{k}$}\label{subapp:asympNksigmak}

The explicit value of $N_{k}$ can be found, for example, in Appendix $D.1$ of \citet{bach2017breaking}. For a $d$-dimensional sphere $\mathbb{S}^{d-1}$, we have
\begin{equation}\label{eq:definitionNk}
    N_{k} = \frac{2k+d-2}{k} \binom{k+d-3}{d-2}.
\end{equation}
for $k\geq 1$ integer. For $k=0$, the basis element is the constant function such that $N_{0}=1$.

Expanding \eqref{eq:definitionNk}, we obtain
\begin{equation}
    \begin{split}
        N_{k} &= \frac{2k+d-2}{k} \frac{(k+d-3)(k+d-4)...(d-2)!}{(k-1)!(d-2)!}\\
        &= \frac{2k+d-2}{k!}(d-2+(k-1))...(d-1)
    \end{split}
\end{equation}
that is
\begin{equation}\label{eq:Nkfinal}
    N_{k} = \frac{(2k+d-2)(d-2+(k-1))(d-2+(k-2))...(d-1)}{k!}.
\end{equation}

The expression for $\sigma_{k}$ can be found in Appendix $A.1$ of \citet{petrini2022learning}. For $k>1$ we get
\begin{equation}
    \sigma_{k} = \frac{-\Gamma\left(\frac{d}{2}\right)}{\sqrt{\pi}\Gamma\left(\frac{d}{2}\right)}\frac{\Gamma\left(\frac{d-1}{2}\right)}{\Gamma\left(k+\frac{d-1}{2}\right)}\left(-\frac{1}{2}\right)^{k} \left. \frac{d^{k-2}}{dt^{k-2}}(1-t^{2})^{k+\frac{d-3}{2}} \right|_{t=0}^{t=1}.
\end{equation}
For $k=0$ we have another expression which is
\begin{equation}\label{eq:sigma0}
\begin{split}
\sigma_0&= \frac{\Gamma(d/2)}{\sqrt{\pi}\,\Gamma((d-1)/2)}
   \int_0^1 t(1-t^2)^{\frac{d-3}{2}}\,dt \\&= \frac{\Gamma(d/2)}{2\sqrt{\pi}\,\Gamma((d-1)/2)}
   \int_0^1 u^{\frac{d-3}{2}}\,du \\&= \frac{\Gamma(d/2)}{2\sqrt{\pi}\,\Gamma((d-1)/2)}
   \left[\frac{2}{d-1}u^{\frac{d-1}{2}}\right]_0^1 \\&= \frac{\Gamma(d/2)}
   {(d-1)\sqrt{\pi}\,\Gamma((d-1)/2)} \\&= \frac{\Gamma(d/2)}
   {2\sqrt{\pi}\,\Gamma((d+1)/2)}\\
   &\asymp \frac{1}{\sqrt{d}}
\end{split}
\end{equation}

For $k=1$ we have a similar expression
\begin{equation}
    \sigma_{1} = \frac{\Gamma\left(\frac{d}{2}\right)}{\sqrt{\pi}\Gamma(\frac{d-1}{2})} \int_{0}^{1} t^{2}(1-t^{2})^{\frac{d-3}{2}}dt.
\end{equation}
Using the change of variables $u=t^2$ such that $du=2tdt$ and we get for the integral
\begin{equation}
    \begin{split}
        \int_{0}^{1}\frac{u(1-u)^{\frac{d-3}{2}}}{2\sqrt{u}}du &=\frac{1}{2}\int_{0}^{1}u^{\frac{1}{2}}(1-u)^{\frac{d-3}{2}}du\\
        &= \frac{1}{2}B\left(\frac{3}{2},\frac{d-1}{2}\right)
        = \frac{1}{2}\frac{\Gamma\left(\frac{3}{2}\right)\Gamma\left(\frac{d-1}{2}\right)}{\Gamma\left(\frac{d}{2}+1\right)}\\
        &= \frac{\sqrt{\pi}}{4}\frac{\Gamma\left(\frac{d-1}{2}\right)}{\frac{d}{2}\Gamma\left(\frac{d}{2}\right)}
        =\frac{\sqrt{\pi}}{2d}\frac{\Gamma\left(\frac{d-1}{2}\right)}{\Gamma\left(\frac{d}{2}\right)}
    \end{split}
\end{equation}
where we have written the integral in terms of a Beta function. Then, we obtain a simple expression 
\begin{equation}\label{eq:sigma1}
    \sigma_{1} = \frac{1}{2d}.
\end{equation}

 Substituting $k=1$ in \eqref{eq:definitionNk} we conclude that
 \begin{equation}\label{eq:N1sigma1squared}
     N_1\sigma_{1}^{2} = \frac{d}{4d^2} = \frac{1}{4d}.
 \end{equation}

For the derivative appearing in the expression of $\sigma_{k}$ for $k>1$ we claim the following.

\begin{lemma}
    We have
    \begin{equation}
       \left. \frac{d^{k-2}}{dt^{k-2}}(1-t^{2})^{k+\frac{d-3}{2}} \right|_{t=0}^{t=1} = \begin{cases}
           (-1)^{\frac{k}{2}}\binom{k+\frac{d-3}{2}}{\frac{k-2}{2}}(k-2)!,\; \text{$k$ even}\\
           0,\; \text{$k$ odd}
       \end{cases}
    \end{equation}
\end{lemma}

\begin{proof}
    Using the generalized binomial theorem and denoting $a \coloneqq k+\frac{d-3}{2}$ we have that for $t=0$
    \begin{equation}
        (1-t^{2})^{a} = \sum_{l=0}^{\infty} \binom{a}{l} (-1)^{l}t^{2l}
    \end{equation}.
    
    If $2l<k-2$, the $(k-2)-$th derivative is $0$ because it is the derivative of a constant. If $2l>k-2$, the $(k-2)-$th derivative is also $0$ since we will have $t^{s}$ for $s>0$ evaluated at $t=0$. We conclude
    \begin{equation}
    \begin{split}
        -\frac{d^{k-2}}{dt^{k-2}}(1-t^{2})^{a} &= -\binom{a}{\frac{k-2}{2}}(-1)^{\frac{k-2}{2}}\left[2\left(\frac{k-2}{2}\right)\right]!\\
        &= \binom{a}{\frac{k}{2}}(-1)^{\frac{k-2}{2}}\left(k-2\right)!
    \end{split}
    \end{equation}

    Since we only have the term $2l=k-2$ we necessarily have that $k$ is even. Moreover, for $t-1$, we observe that all terms in $\frac{d^{k-2}}{dt^{k-2}}$ will contain a factor of $(1-t^{2})^{\frac{d-3}{2}+b}$ where $b\geq 2$ such that in $t=1$ this will be evaluated to be zero. This concludes the proof.
\end{proof}

\begin{prop}
    For $d\gg 1$ and $k\geq 2$ we have the asymptotics
\begin{equation}
    N_k \sigma_k^2
    \asymp
    \begin{cases}
        \displaystyle
        \frac{\sqrt{d}\,\;(2k+d)}
        {2^k k! (d+k)^{5/2}},\;
        & \text{if } k=2,\; \\[6pt]
        \displaystyle
        \frac{\sqrt{d}\,\;(2k+d)}
        {2^k k! (d+k)^{5/2}}
        \left[
            (k-2)(k-3)\cdots\left(\frac{k}{2}\right)
        \right]^2,\;
        & \text{if } k>2.
    \end{cases}
\end{equation}

\end{prop}
\begin{proof}
Since we are always considering the setting where $d\gg 1$ we can use the asymptotics of the Gamma function to get the following simplication
\begin{equation}
    \frac{\Gamma\left(\frac{d}{2}\right)}{\sqrt{\pi}\Gamma\left(\frac{d-1}{2}\right)} = \frac{\Gamma(\frac{d-1}{2}+\frac{1}{2})}{\sqrt{\pi}\Gamma\left(\frac{d-1}{2}\right)} \asymp \frac{\Gamma\left(\frac{d-1}{2}\right)\left(\frac{d-1}{2}\right)^{\frac{1}{2}}}{\sqrt{\pi}\Gamma\left(\frac{d-1}{2}\right)} \asymp \sqrt{d}
\end{equation}
and
\begin{equation}
    \frac{\Gamma\left(\frac{d-1}{2}\right)}{\Gamma\left(\frac{d-1}{2}+k\right)}=\frac{1}{\left(\frac{d-1}{2}+(k-1)\right)...\left(\frac{d-1}{2}\right)}.
\end{equation}

Putting everything together we obtain the following expression for $\sigma_{k}$ when $k>1$ in the regime where $d$ is very large
\begin{equation}\label{eq:sigmakasymp}
    \sigma_{k} \asymp \frac{(-1)^{\frac{k}{2}+1}}{2^{k}} \frac{\sqrt{d}}{\left(\frac{d-1}{2}+(k-1)\right)...\left(\frac{d-1}{2}\right)}\binom{k+\frac{d-3}{2}}{\frac{k-2}{2}}(k-2)!
\end{equation}
for $k>1$ even and $\sigma_{k}=0$ for $k>1$ odd.

Using \eqref{eq:Nkfinal} and \eqref{eq:sigmakasymp} we conclude that for $k>1$ even
\begin{equation}
\begin{split}
    N_{k}\sigma_{k}^{2} &\asymp \frac{(2k+d-2)(d-2+(k-1))(d-2+(k-2))...(d-1)}{k!}\\
    &\cdot\frac{d}{4^k}\frac{1}{\left(\frac{d-1}{2}+(k-1)\right)^2\left(\frac{d-1}{2}+(k-2)\right)^2...\left(\frac{d-1}{2}\right)^2}\left(\frac{(k+\frac{d-3}{2})!}{\left(\frac{k+d-1}{2}\right)!\left(\frac{k-2}{2}\right)!}\right)^2((k-2)!)^2
\end{split}
\end{equation}

Observing that
\begin{equation}
    \left(\frac{d-1}{2}+(k-1)\right)\left(\frac{d-1}{2}+ (k-2)\right)...\left(\frac{d-1}{2}\right) = \frac{\left(k+\frac{d-3}{2}\right)!}{\left(\frac{d-3}{2}\right)!}
\end{equation}
we get
\begin{equation}\label{eq:Nsigma}
    \begin{split}
        N_{k}\sigma_{k}^{2} &\asymp \frac{(2k+d-2)(d-2+ (k-1))(d-2+(k-2))...(d-1)}{k!}\\
        &\cdot\frac{d}{4^k}\left(\frac{\left(\frac{d-3}{2}\right)!\left(k-2\right)!}{\left(\frac{k+d-1}{2}\right)!\left(\frac{k-2}{2}\right)!}\right)^{2}.
    \end{split}
\end{equation}

Notice that
\begin{equation}
    (d-2+(k-1))(d-2+(k-2))...(d-1) = \frac{(d-3+k)!}{(d-2)!}
\end{equation}
and that
\begin{equation}
    \begin{split}
        \frac{(d-3+k)!}{\left(\frac{k+d-1}{2}\right)!} &= (d+k-3)(d+k-4)...\left(\frac{k+d-1}{2}+1\right)\\
        &= (d+k-3)(d+k-4)...\left(\frac{d+k+1}{2}\right)
    \end{split}
\end{equation}
also
\begin{equation}
    \frac{(k-2)!}{\left(\frac{k-2}{2}\right)!} = \begin{cases}
        (k-2)(k-3)...\left(\frac{k}{2}\right),\; \text{if $k>2$}\\
        1,\; \text{if $k=2$}. 
    \end{cases}
\end{equation}
such that substituting in the expression for $N_{k}\sigma_{k}^{2}$ and separating the terms depending only on $k$, only on $d$ or on both, we get for $k> 2$ even
\begin{equation}
    \begin{split}
        N_{k}\sigma_{k}^{2} &\asymp \frac{(2k+d-2)(d+k-3)(d+k-4)...(\frac{d+k+1}{2})}{\left(\frac{k+d-1}{2}\right)!}\\
        &\cdot \frac{d\left[\left(\frac{d-3}{2}\right)!\right]^{2}}{(d-2)!} \frac{[(k-2)(k-3)...(\frac{k}{2})]^2}{4^k k!}
    \end{split}
\end{equation}
 and for $k=2$
\begin{equation}
    \begin{split}
        N_{k}\sigma_{k}^{2} &\asymp \frac{(2k+d-2)(d+k-3)(d+k-4)...(\frac{d+k+1}{2})}{\left(\frac{k+d-1}{2}\right)!}\frac{d\left[\left(\frac{d-3}{2}\right)!\right]^{2}}{(d-2)!4^k k!}.
    \end{split}
\end{equation}

Now, since $d$ is very large we will use Stirling approximation on the terms that contains only $d$ or that contains both as follows
\begin{equation}
\begin{split}
    \frac{d\left[\left(\frac{d-3}{2}\right)\right]^{2}}{(d-2)!} &\asymp \frac{d\left[\left(\frac{d-3}{2}\right)^{\frac{1}{2}}\left(\frac{d-3}{2e}\right)^{\frac{d-3}{2}}\right]^{2}}{(d-2)^{\frac{1}{2}}(\frac{d-2}{e})^{d-2}}\\
    &\asymp \frac{d \cdot d^{1+d-3}}{d^{\frac{1}{2}+d-2}}\frac{e^{d-2}}{(2e)^{d-3}} \asymp \frac{d^{\frac{1}{2}}}{2^{d}}
\end{split}
\end{equation}
and
\begin{equation}
    \begin{split}
        \frac{(2k+d-2)(d+k-3)(d+k-4)...\left(\frac{d+k+1}{2}\right)}{\left(\frac{k+d-1}{2}\right)!}
        &\asymp \frac{(2k+d-2)(d+k-3)!}{[\left(\frac{d+k-1}{2}\right)]^{2}} \\
        &\asymp \frac{(2k+d)(d+k-3)^{\frac{1}{2}}\left(\frac{d+k-3}{e}\right)^{d+k-3}}{\left[\left(\frac{d+k-1}{2}\right)^{\frac{1}{2}}\left(\frac{d+k-1}{2e}\right)^{\frac{d+k-1}{2}}\right]^{2}}\\
        &\asymp (2k+d)(d+k)^{-\frac{5}{2}}2^{d+k}.
    \end{split}
\end{equation}

Putting it all together we get 
\begin{equation}\label{eq:Nksigmak2expression}
    N_k \sigma_k^2
    \asymp
    \begin{cases}
        \displaystyle
        \frac{\sqrt{d}\,\;(2k+d)}
        {2^k k! (d+k)^{5/2}},\;
        & \text{if } k=2,\; \\[6pt]
        \displaystyle
        \frac{\sqrt{d}\,\;(2k+d)}
        {2^k k! (d+k)^{5/2}}
        \left[
            (k-2)(k-3)\cdots\left(\frac{k}{2}\right)
        \right]^2,\;
        & \text{if } k>2.
    \end{cases}
\end{equation}
as we wanted.
\end{proof}

It will actually be useful later to have the asymptotic expressions for $N_{k}$ and $\sigma_{k}^{2}$ separately. For $N_{k}$, using Stirling approximation again, we get
\begin{equation}\label{eq:Nkasymptotics}
    \begin{split}
        N_{k} &= \frac{(2k+d-2)(d+k-3)!}{k!(d-2)!}\\
        &\asymp \frac{(2k+d)(d+k)^{\frac{1}{2}}\left(\frac{d+k}{e}\right)^{d+k-3}}{k!d^{\frac{1}{2}}\left(\frac{d}{e}\right)^{d-2}}\\
        &\asymp \frac{(2k+d)d^{\frac{3}{2}-d}(d+k)^{d+k-\frac{5}{2}}}{k!e^{k}}
    \end{split}
\end{equation}
Now, for $\sigma_{k}^{2}$, from the second line of \eqref{eq:Nsigma} and using Stirling again, we have for $k>2$
\begin{equation}\label{eq:sigmakasymptotics}
    \begin{split}
        \sigma_{k}^{2} &= \frac{d\left[\left(\frac{d-3}{2}\right)!\right]^{2}}{\left[\left(\frac{d+k-1}{2}\right)!\right]^{2}}\frac{[(k-2)(k-3)...(k/2)]^{2}}{4^{k}}\\
        &\asymp \frac{d\left[d^{\frac{1}{2}}\left(\frac{d}{2e}\right)^{\frac{d-3}{2}}\right]^{2}}{\left[(d+k)^{\frac{1}{2}}\left(\frac{d+k}{2e}\right)^{\frac{d+k-1}{2}}\right]^{2}}\frac{[(k-2)(k-3)...(k/2)]^{2}}{4^{k}} \\
        &\asymp \frac{d^{d-1}}{(d+k)^{d+k}}\left(\frac{e}{2}\right)^{k}[(k-2)(k-3)...(k/2)]^{2}
    \end{split}
\end{equation}
and for $k=2$ the following
\begin{equation}\label{eq:sigma2}
    \sigma_{2}^{2} \asymp \frac{d^{d-1}}{(d+k)^{d+k}}\left(\frac{e}{2}\right)^{k}.
\end{equation}

\subsection{Scaling of $\nu$}

Again, since $\nu$ is the key deterministic quantity that appears in the errors we start by calculating its asymptotic behavior.

\begin{lemma}[\textbf{Order of magnitude of $\nu$ -- ReLU activation}]\label{lem:nu-relu}
    Under our setting, we have 
    \begin{equation}
    \nu = \Theta(d^{-\alpha-1}).
\end{equation}
\end{lemma}

\begin{proof}
We use again \eqref{eq:nuequation} from Theorem \ref{theo:selfconsistenteacher}, that is,
\begin{equation}
    m = \tr{\left(\Psi(\Psi+\nu)^{-1}\right)}
\end{equation}
where now $m=M_{\mathrm{TE}}=d^{\alpha}$.

Calculating
\begin{equation}
    d^{\alpha} = \tr{\left(\frac{\Psi}{\Psi+\nu}\right)} = \sum_{k\in \mathcal{K}}\frac{N_{k}\sigma_{k}^{2}}{\sigma_{k}^{2}+\nu}.
\end{equation}
since the eigenvalues of $\Psi$ are $\sigma_{k}^{2}$ and there are $N_{k}$ of them. 

We want to find upper and lower bounds for $\nu$ to find its order of magnitude with respect to $d$ in the regime where $d$ is very large. Let us separate the terms for $k=1$ and $k=2$ from the summation to get
\begin{equation}
    d^{\alpha} = \frac{N_{0}\sigma_{0}^{2}}{\sigma_{0}^{2}+\nu} + \frac{N_{1}\sigma_{1}^{2}}{\sigma_{1}^{2}+\nu} +  \frac{N_{2}\sigma_{2}^{2}}{\sigma_{2}^{2}+\nu} + \sum_{\substack{k\in \mathcal{K}\\ k \geq 4}}\frac{N_{k}\sigma_{k}^{2}}{\sigma_{k}^{2}+\nu}
\end{equation}

For $k=0$, we get
\begin{equation}
    \frac{N_{0}\sigma_{0}^{2}}{\sigma_{0}^{2}+\nu} \asymp \frac{d^{-1}}{d^{-1}+\nu} = \frac{1}{1+d\nu}.
\end{equation}

For $k=1$, using \eqref{eq:sigma1} and \eqref{eq:N1sigma1squared} we obtain
\begin{equation}
\begin{split}
    \frac{N_{1}\sigma_{1}^{2}}{\sigma_{1}^{2}+\nu} &= \frac{\frac{1}{4d}}{\left(\frac{1}{2d}\right)^{2}+\nu}\\
    &= \frac{4d^{2}}{1+4d^{2}\nu}\frac{1}{4d}\\
    &= \frac{d}{1+4d^{2}\nu}.
\end{split}
\end{equation}

Now, for $k=2$ we use \eqref{eq:Nksigmak2expression} and \eqref{eq:sigma2} to get
\begin{equation}
    \frac{N_{2}\sigma_{2}^{2}}{\sigma_{2}^{2}+\nu} = \frac{\sqrt{d}(4+d)}{8(d+2)^{\frac{5}{2}}}\frac{1}{\frac{e^{2}}{4}d^{d-1}(d+2)^{-d-2}+\nu}
\end{equation}

Finally, for the $k\geq 4$ even
\begin{equation}\label{eq:sumkgeq4}
\begin{split}
    \sum_{\substack{k\in \mathcal{K}\\ k \geq 4}} \frac{N_{k}\sigma_{k}^{2}}{\sigma_{k}^{2}+\nu} 
    &= \frac{\sqrt{d}(2k+d)[(k-2)(k-3)...(k/2)]^{2}}{2^{k}k!(d+k)^{\frac{5}{2}}\left[\frac{d^{d-1}e^{k}}{(d+k)^{d+k}2^{k}}[(k-2)(k-3)...(k/2)]^{2}+\nu\right]}\\
    &= \frac{\sqrt{d}(2k+d)[(k-2)(k-3)...(k/2)]^{2}}{k!(d+k)^{\frac{5}{2}-d-k}d^{d-1}e^{k}[(k-2)(k-3)...(k/2)]^{2} + 2^{k}k!(d+k)^{\frac{5}{2}}\nu}\\
    &= \frac{\sqrt{d}(2k+d)(d+k)^{d+k}[(k-2)(k-3)...(k/2)]^{2}}{k!(d+k)^{\frac{5}{2}}d^{d-1}e^{k}[(k-2)(k-3)...(k/2)]^{2} + 2^{k}k!(d+k)^{d+k+\frac{5}{2}}\nu}.
\end{split}
\end{equation}

\noindent\textbf{Upper bound on $\nu$.} In this case we argue that the summation in $k$ is not large. 

Ignoring some positive terms in the denominators we can bound
\begin{equation}\label{eq:d2k0equation}
\begin{split}
    d^{\alpha} &\lesssim \frac{1}{d\nu}+ \frac{d}{4d^{2}\nu} + \frac{\sqrt{d}(4+d)}{8(d+2)^{\frac{5}{2}}\nu} + \sum_{\substack{k\in \mathcal{K}\\ k \geq 4}}\frac{\sqrt{d}(2k+d)[(k-2)(k-3)...(\frac{k}{2})]^{2}}{2^k k!(d+k)^{\frac{5}{2}}\nu}\\
    &\lesssim \frac{1}{d\nu} + \frac{1}{4d\nu} + \frac{1}{4d\nu} + \frac{\sqrt{d}}{d^{\frac{3}{2}}\nu}\sum_{\substack{k\in \mathcal{K}\\ k \geq 4}}\frac{[(k-2)(k-3)...(k/2)]^{2}}{2^{k-1}k!}.
\end{split}
\end{equation}

The following claim states that the terms in the summation are not very large and that they actually decrease in $k$.

\begin{claim}
    If $k\geq 4$ is an even number, then
    \begin{equation}
        \frac{[(k-2)(k-3)...(k/2)]^{2}}{2^{k-1}k!} \leq \frac{1}{2k(k-1)}.
    \end{equation}
\end{claim}

\begin{proofclaim}
    We have
    \begin{equation}
        \begin{split}
            \frac{[(k-2)(k-3)...(k/2)]^{2}}{2^{k-1}k!} &= \frac{[(k-2)!/(k/2 -1)!]^{2}}{k(k-1)(k-2)!2^{k-1}}\\
            &= \frac{(k-2)!}{(k/2-1)!k(k-1)2^{k-1}}.
        \end{split}
    \end{equation}
    Now, write $k=2l$ for $l\in \mathbb{N}$ and $l\geq2$ to get
    \begin{equation}
        \begin{split}
            \frac{[(k-2)(k-3)...(k/2)]^{2}}{2^{k-1}k!} &= \frac{(2l-2)!}{[(l-1)!]^{2}}\frac{1}{2l(2l-1)2^{2l-1}}\\
            &= \binom{2l-2}{l-1}\frac{1}{2l(2l-1)2^{2l-1}}.
        \end{split}
    \end{equation}

Now, we use that
\begin{equation}
    \binom{2l-2}{l-1} \leq \sum_{i=0}^{2l-2}\binom{2l-2}{i} = 2^{2l-2}
\end{equation}
since it is the number of subsets of a set of $2l-2$ elements. This gives
\begin{equation}
    \frac{[(k-2)(k-3)...(k/2)]^{2}}{2^{k-1}k!} \leq \frac{2^{2l-2}}{2l(2l-1)2^{2l-1}} = \frac{1}{2k(k-1)}
\end{equation}
as we wanted. \claimqed
\end{proofclaim}

Using the previous claim in \eqref{eq:d2k0equation} we get
\begin{equation}
    d^{\alpha} \lesssim \frac{1}{d\nu} +  \frac{1}{2d\nu} + \frac{1}{d\nu}\sum_{\substack{k\in \mathcal{K}\\ k \geq 4}}\frac{1}{2k(k-1)}
    \asymp \frac{1}{d\nu}
\end{equation}
since the sum is convergent.

We conclude that
\begin{equation}
    \nu \lesssim d^{-\alpha-1}.
\end{equation}

\noindent\textbf{Lower bound on $\nu$.} We have
\begin{equation}
    d^{\alpha} \asymp \sum_{k\in \mathcal{K}}\frac{N_{k}\sigma_{k}^{2}}{\sigma_{k}^{2}+\nu}
\end{equation}
such that ignoring all terms in the sum except the term with $k=k_{\alpha}$ being any even fixed integer with $k>\alpha$  Since all the terms in the sum are positive, we get
\begin{equation}
    d^{\alpha}\gtrsim\frac{N_{k_{\alpha}}\sigma_{k_{\alpha}}^{2}}{\sigma_{k_{\alpha}}^{2}+\nu}.
\end{equation}
Since $\alpha=O(1)$, we can choose $k_{\alpha}=O(1)$ and use from \eqref{eq:Nkasymptotics}, \eqref{eq:sigmakasymptotics}, and \eqref{eq:sigma2}  that
\begin{equation}
    N_{k_{\alpha}}\asymp d^{k_{\alpha}},
    \qquad
    \sigma_{k_{\alpha}}^{2}\asymp d^{-k_{\alpha}-1}.
\end{equation}
Thus,
\begin{equation}
    d^{\alpha}
    \gtrsim
    \frac{d^{k_{\alpha}}d^{-k_{\alpha}-1}}{d^{-k_{\alpha}-1}+\nu}
    =
    \frac{d^{-1}}{d^{-k_{\alpha}-1}+\nu}.
\end{equation}
It follows that
\begin{equation}
    d^{-k_{\alpha}-1}+\nu\gtrsim d^{-\alpha-1}.
\end{equation}
Since $k_{\alpha}>\alpha$, we have
\begin{equation}
    d^{-k_{\alpha}-1}=o(d^{-\alpha-1}),
\end{equation}
and therefore, for $d$ large enough,
\begin{equation}
    \nu \gtrsim d^{-\alpha-1}.
\end{equation}

Together with the upper bound we obtain
\begin{equation}
    \nu = \Theta(d^{-\alpha-1}).
\end{equation}
as claimed.
\end{proof}


\subsection{Teacher error scaling} \label{app:teacherscaling}

The teacher error, as well as the student, will end up depending on how the teacher capacity coefficient $\alpha$ relates with the target degree $k_{0}$. 

\begin{obs}
    Remember that $\beta_{0,l}=0$ for all $l$ such that every sum in $k\in \mathcal{K}$ that involves $\beta$ will not take $k=0$ into account.
\end{obs}

\begin{theo}\label{theorem:teachererror}
        In our setting of a pure harmonic target of degree $k_{0}\geq1$ with $\lVert f^{*}\rVert_{\mathcal{D}}=1$, if $M_{\mathrm{TE}} = d^{\alpha}$ for some $\alpha>0$, then
\begin{equation}
    \overline{\mathcal{L}}_{\mathrm{TE}} \asymp \begin{cases}
       1,\; \alpha \in (0,\;k_{0}]\\
        d^{-\alpha+k_{0}},\; \alpha > k_{0}.
    \end{cases}
\end{equation}
\end{theo}

\begin{proof}We have the same expression for the teacher deterministic equivalent error as the one for the linear case
\begin{equation}
    \overline{\mathcal{L}}_{\mathrm{TE}} = \nu \langle \Psi^{\frac{1}{2}}\beta, (\Psi+\nu)^{-1}\Psi^{\frac{1}{2}}\beta\rangle. 
\end{equation}

Notice that for all $k\leq k_{0}$ where $k_{0}=O(1)$ and $d\gg 1$ we get from \eqref{eq:Nkasymptotics}, \eqref{eq:sigmakasymptotics}, and \eqref{eq:sigma2} that
\begin{equation}\label{eq:sigmaNapprox}
\begin{split}
    N_{k} &\asymp d^{k}\\
    \sigma_{k}^{2} &\asymp d^{-k-1}
\end{split}
\end{equation}
for $k=1$ or $k$ even. This is also true for $k=0$ from equation \eqref{eq:sigma0} and $N_{0}=1$.

The unit norm hypothesis for the target gives
\begin{equation}
    \sum_{k\in \mathcal{K}}\sum_{l=1}^{N_{k}} \sigma_{k}^{2}\beta_{k,l}^{2} = 1.
\end{equation}

We obtain
\begin{equation}\begin{split}
    \overline{\mathcal{L}}_{\mathrm{TE}} &= \nu \langle \Psi^{\frac{1}{2}}\beta, (\Psi+\nu)^{-1}\Psi^{\frac{1}{2}}\beta \rangle\\
    &= \nu \sum_{k\in \mathcal{K}}\frac{\sigma_{k}^{2}}{\sigma_{k}^{2}+\nu} \sum_{l=1}^{N_{k}}\beta_{k,l}^{2}\\
    &\asymp d^{-\alpha-1} \sum_{k\in \mathcal{K}}\frac{d^{-k-1}}{d^{-k-1}+d^{-\alpha-1}}\sum_{l=1}^{N_{k}}\beta_{k,l}^{2} 
\end{split}
\end{equation}

\noindent\textbf{Case 1: $0<\alpha\leq k_{0}$.}
\begin{equation}
    \overline{\mathcal{L}}_{\mathrm{TE}} \asymp d^{-\alpha-1} \left\{\sum_{\substack{k\in \mathcal{K}\\ 0 \leq k \leq \alpha}} \frac{d^{-k-1}}{d^{-k-1}}\sum_{l=1}^{N_{k}}\beta_{k,l}^{2} + \sum_{\substack{k\in \mathcal{K}\\ \alpha+1 \leq k \leq k_{0}}}\frac{d^{-k-1}}{d^{-\alpha-1}}\sum_{l=1}^{N_{k}}\beta_{k,l}^{2}\right\}
\end{equation}
using $\beta_{k,l}^{2}\neq 0$ only for $k=k_{0}$, i.e., the pure harmonic target hypothesis, we get
\begin{equation}\begin{split}
    \overline{\mathcal{L}}_{\mathrm{TE}} &\asymp d^{-\alpha -1}\frac{d^{-k_{0}-1}}{d^{-\alpha -1}}\sum_{l=1}^{N_{k_{0}}}\beta_{k_{0},l}^{2}\\
    &\asymp 1.
\end{split}
\end{equation}

\noindent\textbf{Case 2: $\alpha>k_{0}$.}
\begin{equation}\begin{split}
    \overline{\mathcal{L}}_{\mathrm{TE}} &\asymp d^{-\alpha -1} \frac{d^{-k_{0}-1}}{d^{-k_{0}-1}} \sum_{l=1}^{N_{k_{0}}}\beta_{k_{0},l}^{2}\\
    &\asymp d^{-\alpha +k_{0}}.
\end{split}
\end{equation}

In conclusion we obtained
\begin{equation}\label{eq:teacherreluerror}
    \overline{\mathcal{L}}_{\mathrm{TE}} \asymp \begin{cases}
        1,\; \text{ if $\alpha \in (0,\;k_{0}]$},\;\\
        d^{-\alpha+k_{0}},\; \text{ if $\alpha > k_{0}$}
    \end{cases}
\end{equation}
as we desired.
\end{proof}

Intuitively, we can interpret this result as follows. When the teacher capacity exponent $\alpha$ is not larger than the target degree $k_{0}$, the teacher is not able to learn the target accurately, and its error remains of order $1$. However, once the teacher is sufficiently large, i.e., $\alpha>k_{0}$, its error starts decreasing with its capacity. In fact, the key quantity is the difference $\alpha-k_{0}$, which measures the gap between the teacher capacity and the target complexity.

\subsection{Student error scaling: quadratic improvement}\label{app:studentscaling}

We finally prove the WTS result for the ReLU setting.

\begin{proof}[Proof of Theorem~\ref{theo:weaktostrongrelu}]
Under the unit norm hypothesis $\lVert f^{*}\rVert_{\mathcal{D}}=1$, we start with \eqref{eq:studentdeteq} that gives the student deterministic equivalent error at $t=T$, with $\overline{P}_0=\nu(\Psi+\nu I)^{-1}$ and $S_T=(I-e^{-\Psi T})^2$,
 \begin{equation}
    \overline{\mathcal{L}}_{\mathrm{ST}}(T) = \langle \Psi^{\frac{1}{2}}\beta, [e^{-2\Psi T}-S_T\overline{P}_0+\overline{P}_0(I-e^{-2\Psi T})+\overline{\mathcal{P}}(S_T)]\Psi^{\frac{1}{2}}\beta\rangle.
\end{equation}

We use the stopping-time from the statement of Theorem \ref{theo:weaktostrongrelu}, that is,
\begin{equation}
\begin{split}
    T &= \frac{1}{\sigma_{k_{0}}^{2}}(\log{M}_{\mathrm{TE}}+\log{\overline{\mathcal{L}}_{\mathrm{TE}}^{-1}}),
\end{split}
\end{equation}
such that, using \eqref{eq:teacherreluerror} we get
\begin{equation}
    T \asymp \begin{cases}
    d^{k_{0}+1}\log{(d^{\alpha})},\; \text{ if $\alpha \in (0,\;k_{0}]$}\\
    d^{k_{0}+1}\log{(d^{2\alpha-k_{0}})},\; \text{ if $\alpha > k_{0}$}.
    \end{cases}
\end{equation}

We treat each of the four terms of $\overline{\mathcal{L}}_{\mathrm{ST}}(T)$ above separately. Since $\beta_{0,l}=0$ for all $l$, we do not write $k=0$ in the index summations below when $\beta$ is in the sum. 

\noindent\textbf{Term 1:} 
$\langle \Psi^{\frac{1}{2}}\beta, e^{-2\Psi T}\Psi^{\frac{1}{2}}\beta\rangle$

\noindent\textbf{Case 1: $0<\alpha\leq k_{0}$.}
\begin{equation}\label{eq:student1alpha1}
\begin{split}
    1) \coloneqq \langle \Psi^{\frac{1}{2}}\beta, e^{-2\Psi T}\Psi^{\frac{1}{2}}\beta\rangle &= \sum_{\substack{k\in \mathcal{K}\\ 1 \leq k \leq k_{0}}}\frac{\sigma_{k}^{2}}{e^{2\sigma_{k}^{2}T}}\sum_{l=1}^{N_{k}}\beta_{k,l}^{2}\\ 
    &\asymp \frac{\sigma_{k_{0}}^{2}\sum_{l=1}^{N_{k_{0}}}\beta_{k_{0},l}^{2}}{\exp{(2\sigma_{k_{0}}^{2}\sigma_{k_{0}}^{-2}\log{(d^{\alpha})})}}\\
    &\asymp \frac{1}{d^{2\alpha}}\\
    &\asymp d^{-2\alpha} 
\end{split}
\end{equation}

\noindent\textbf{Case 2: $\alpha>k_{0}$.}
\begin{equation}\label{eq:student1alpha2}
    \begin{split}
    1) = \langle \Psi^{\frac{1}{2}}\beta, e^{-2\Psi T}\Psi^{\frac{1}{2}}\beta\rangle &\asymp \frac{1}{d^{4\alpha-2k_{0}}}\\   
    &\asymp d^{-4\alpha +2k_{0}}
    \end{split}
\end{equation}

\noindent\textbf{Term 2:} $\langle \Psi^{\frac{1}{2}}\beta, -(I-e^{-\Psi T})^{2}\overline{P}_0\Psi^{\frac{1}{2}}\beta\rangle$

We have 
\begin{equation}\begin{split}
    2) \coloneqq \langle \Psi^{\frac{1}{2}}\beta, -(I-e^{-\Psi T})^{2}\overline{P}_0\Psi^{\frac{1}{2}}\beta\rangle &\asymp \sum_{\substack{k\in \mathcal{K}\\ 1 \leq k \leq k_{0}}} \frac{\sigma_{k}^{2}(1-e^{-\sigma_{k}^{2}T})^{2}}{\sigma_{k}^{2}\nu^{-1}+1}\sum_{l=1}^{N_{k}}\beta_{k,l}^{2}.
\end{split}
\end{equation}

\noindent\textbf{Case 1: $0<\alpha\leq k_{0}$.}
\begin{equation}\label{eq:student2alpha1}
\begin{split}
    \langle \Psi^{\frac{1}{2}}\beta, -(I-e^{-\Psi T})^{2}\overline{P}_0\Psi^{\frac{1}{2}}\beta\rangle &\asymp -\frac{d^{-k_{0}-1}(1-d^{-\alpha})^{2}}{d^{-k_{0}-1+\alpha+1}+1}\sum_{l=1}^{N_{k_{0}}}\beta_{k_{0},l}^{2}\\
    &\asymp -(1-2d^{-\alpha}+d^{-2\alpha})
\end{split}
\end{equation}

\noindent\textbf{Case 2: $\alpha>k_{0}$.}
\begin{equation}\label{eq:student2alpha2}
\begin{split}
    \langle \Psi^{\frac{1}{2}}\beta, -(I-e^{-\Psi T})^{2}\overline{P}_0\Psi^{\frac{1}{2}}\beta\rangle &\asymp -\frac{d^{-k_{0}-1}(1-d^{-2\alpha+k_{0}})^{2}}{d^{\alpha-k_{0}}} \sum_{l=1}^{N_{k_{0}}}\beta_{k_{0},l}^{2}\\
    &\asymp -d^{-\alpha+k_{0}}(1-2d^{-2\alpha+k_{0}}+ d^{-4\alpha+2k_{0}})
\end{split}
\end{equation}

\noindent\textbf{Term 3:} $\langle \Psi^{\frac{1}{2}}\beta, \overline{P}_0(I-e^{-2\Psi T})\Psi^{\frac{1}{2}}\beta\rangle$

This term is treated similarly to the previous one
\begin{equation}
\begin{split}
    3) \coloneqq \langle \Psi^{\frac{1}{2}}\beta, \overline{P}_0(I-e^{-2\Psi T})\Psi^{\frac{1}{2}}\beta\rangle &\asymp \sum_{\substack{k\in \mathcal{K}\\ 1 \leq k \leq k_{0}}} \frac{\sigma_{k}^{2}(1-e^{-2\sigma_{k}^{2}T})}{\sigma_{k}^{2}\nu^{-1}+1}\sum_{l=1}^{N_{k}}\beta_{k,l}^{2}
\end{split}
\end{equation}

\noindent\textbf{Case 1: $0<\alpha\leq k_{0}$.}
\begin{equation}\label{eq:student3alpha1}
\begin{split}
    \langle \Psi^{\frac{1}{2}}\beta, \overline{P}_0(I-e^{-2\Psi T})\Psi^{\frac{1}{2}}\beta\rangle &\asymp d^{-k_{0}-1}(1-d^{-2\alpha})\sum_{l=1}^{N_{k_{0}}}\beta_{k_{0},l}^{2}\\
    &\asymp (1-d^{-2\alpha})
\end{split}
\end{equation}

\noindent\textbf{Case 2: $\alpha>k_{0}$.}
\begin{equation}\label{eq:student3alpha2}
\begin{split}
    \langle \Psi^{\frac{1}{2}}\beta, \overline{P}_0(I-e^{-2\Psi T})\Psi^{\frac{1}{2}}\beta\rangle &\asymp d^{-\alpha+k_{0}}(1-d^{-4\alpha+2k_{0}}).
\end{split}
\end{equation}

Putting together the contributions of terms $2$ and $3$ we get in first order the following.

\noindent\textbf{Case 1: $0<\alpha\leq k_{0}$.}
\begin{equation}\label{eq:student23alpha1}
\begin{split}
    2)+3) &\asymp (2d^{-\alpha}-2d^{-2\alpha})\\
    &\lesssim d^{-\alpha}. 
\end{split}
\end{equation}

\noindent\textbf{Case 2: $\alpha>k_{0}$.}
\begin{equation}\label{eq:student23alpha2}
\begin{split}
    2)+3)&\asymp (2d^{-2\alpha+k_{0}}-2d^{-4\alpha+2k_{0}})d^{-\alpha+k_{0}}\\
    &\lesssim d^{-3\alpha+2k_{0}}.
\end{split}
\end{equation}
\noindent\textbf{Term 4:} $\langle \Psi^{\frac{1}{2}}\beta, \overline{\mathcal{P}}(S_T)\Psi^{\frac{1}{2}}\beta\rangle$

By Theorem \ref{theo:doubleresolventapprox} we have
\begin{equation}
    \langle \Psi^{\frac{1}{2}}\beta, \overline{\mathcal{P}}(S_T)\Psi^{\frac{1}{2}}\beta\rangle = \left\langle \Psi^{\frac{1}{2}}\beta, \frac{(I-e^{-\Psi T})^{2}}{(I+\Psi\nu^{-1})^{2}}+\frac{\frac{\Psi}{(I+\Psi \nu^{-1})^{2}}\tr{\left(\frac{\nu^{-2}\Psi(I-e^{-\Psi T})^{2}}{(I+\Psi\nu^{-1})^{2}}\right)}}{d^{\alpha}-\tr{\left(\frac{\Psi^{2}\nu^{-2}}{(I+\Psi\nu^{-1})^{2}}\right)}} \Psi^{\frac{1}{2}}\beta\right\rangle
\end{equation}

Again, we treat each term separately. The first one is 
\begin{equation}
\begin{split}
4)(1) \coloneqq \left\langle \Psi^{\frac{1}{2}}\beta,
\frac{(I-e^{-\Psi T})^{2}}{(I+\Psi\nu^{-1})^{2}}
\Psi^{\frac{1}{2}}\beta\right\rangle
&=\sum_{\substack{k\in \mathcal{K}\\ 1 \leq k \leq k_{0}}}\frac{d^{-k-1}(1-e^{-\sigma_{k}^{2}T})^{2}}{(1+d^{\alpha-k})^{2}}\sum_{l=1}^{N_{k}}\beta_{k,l}^{2}.
\end{split}
\end{equation}

\noindent\textbf{Case 1: $0<\alpha\leq k_{0}$.}
\begin{equation}\label{eq:student4.1alpha1}
\begin{split}
\left\langle \Psi^{\frac{1}{2}}\beta,
\frac{(I-e^{-\Psi T})^{2}}{(I+\Psi\nu^{-1})^{2}}
\Psi^{\frac{1}{2}}\beta\right\rangle
&\asymp d^{-k_{0}-1}(1-d^{-\alpha})^{2}\sum_{l=1}^{N_{k_{0}}}\beta_{k,l}^{2}\\
&\asymp 1.
\end{split}
\end{equation}
\noindent\textbf{Case 2: $\alpha>k_{0}$.}
\begin{equation}\label{eq:student4.1alpha2}
\begin{split}
    \left\langle \Psi^{\frac{1}{2}}\beta,
\frac{(I-e^{-\Psi T})^{2}}{(I+\Psi\nu^{-1})^{2}}
\Psi^{\frac{1}{2}}\beta\right\rangle
&\asymp \frac{d^{-k_{0}-1}(1-d^{-2\alpha+k_{0}})^{2}}{d^{2\alpha-2k_{0}}}\sum_{l=1}^{N_{k_{0}}}\beta_{k,l}^{2}\\
&\asymp d^{-2\alpha +2k_{0}}
\end{split}
\end{equation}

For the second term,
\begin{equation}\label{eq:secondtermstudent}
\begin{split}
4)(2) \coloneqq \left\langle \Psi^{\frac{1}{2}}\beta,
\frac{\frac{\Psi}{(I+\Psi \nu^{-1})^{2}}\tr{\left(\frac{\nu^{-2}\Psi(I-e^{-\Psi T})^{2}}{(I+\Psi\nu^{-1})^{2}}\right)}}{d^{\alpha}-\tr{\left(\frac{\Psi^{2}\nu^{-2}}{(I+\Psi\nu^{-1})^{2}}\right)}}\beta\right\rangle &= \sum_{\substack{k\in \mathcal{K}\\ 1 \leq k \leq k_{0}}} \frac{\tr{\left(\frac{\nu^{-2}\Psi(I-e^{-\Psi T})^{2}}{(I+\Psi\nu^{-1})^{2}}\right)}}{d^{\alpha}-\tr{\left(\frac{\Psi^{2}\nu^{-2}}{(I+\Psi \nu^{-1})^{2}}\right)}}\frac{\sigma_{k}^{4}N_{k}}{(1+d^{-k-1+2k_{0}+1})^{2}},
\end{split}
\end{equation}
we first observe that we can simplify the denominator. In fact, by \eqref{eq:nuequation}, we can write
\begin{equation}
    \begin{split}
        d^{\alpha} - \tr{\left(\frac{\Psi^{2}\nu^{-2}}{(I+\Psi \nu^{-1})^{2}}\right)} &= \tr{\left(\frac{\Psi\nu^{-1}}{I+\Psi \nu^{-1}}\right)} - \tr{\left(\frac{\Psi^{2}\nu^{-2}}{(I+\Psi \nu^{-1})^{2}}\right)}\\
        &= \tr{\left(\frac{\Psi\nu^{-1}+\Psi^{2}\nu^{-2}-\Psi^{2}\nu^{-2}}{(I+\Psi\nu^{-1})^{2}}\right)}\\
        &= \tr{\left(\frac{\Psi\nu^{-1}}{(I+\Psi\nu^{-1})^{2}}\right)}
    \end{split}
\end{equation}
such that we will need a lower bound for this trace.

\noindent\textbf{Lower bound for the denominator trace.}

We want to lower bound 
\begin{equation}
    \tr{\left(\frac{\Psi\nu^{-1}}{(I+\Psi\nu^{-1})^{2}}\right)} = \sum_{k\in \mathcal{K}} \frac{\sigma_{k}^{2}d^{\alpha+1}}{(1+\sigma_{k}^{2}d^{\alpha+1})^{2}}N_{k}.
\end{equation}

\chg{Since the terms are positive, a single degree gives a lower bound. Let $k_{\alpha}$ be a fixed even degree with $k_{\alpha}>\alpha$, as in the proof of Lemma~\ref{lem:nu-relu}. Then $\sigma_{k_{\alpha}}^{2}d^{\alpha+1}\asymp d^{\alpha-k_{\alpha}}\to0$, so
\begin{equation}
    \tr{\left(\frac{\Psi\nu^{-1}}{(I+\Psi\nu^{-1})^{2}}\right)}\ \geq\ \frac{\sigma_{k_{\alpha}}^{2}d^{\alpha+1}}{(1+\sigma_{k_{\alpha}}^{2}d^{\alpha+1})^{2}}N_{k_{\alpha}}\ \asymp\ d^{-k_{\alpha}-1}d^{\alpha+1}d^{k_{\alpha}}=d^{\alpha}.
\end{equation}
Conversely, $\frac{\Psi\nu^{-1}}{(I+\Psi\nu^{-1})^{2}}\leq\frac{\Psi\nu^{-1}}{I+\Psi\nu^{-1}}$ and \eqref{eq:nuequation} give the upper bound $d^{\alpha}$. We conclude that
\begin{equation}
     \tr{\left(\frac{\Psi\nu^{-1}}{(I+\Psi\nu^{-1})^{2}}\right)} \asymp d^{\alpha}.
\end{equation}}
%

\noindent\textbf{Upper bound for the numerator trace.}

We want to calculate
\begin{equation}
    \tr{\left(\frac{\nu^{-2}\Psi(I-e^{-\Psi T})^{2}}{(I+\Psi\nu^{-1})^{2}}\right)} = \sum_{k\in \mathcal{K}}\frac{d^{2\alpha+2}(1-e^{-\sigma_{k}^{2}T})^{2}}{(1+\sigma_{k}^{2}d^{\alpha+1})^{2}}N_{k}\sigma_{k}^{2}.
\end{equation}

Again, we divide into two cases depending on the value of $\alpha$.

\noindent\textbf{Case 1: $0<\alpha\leq k_{0}$.} We use the stopping-time
\begin{equation}
    T = \frac{1}{\sigma_{k_{0}}^{2}}\log{(d^{\alpha})}
\end{equation}

let us divide the sum into four terms and calculate each one separately.

We divide the sum into four parts
\begin{equation}
    \begin{split}
        \tr{\left(\frac{\nu^{-2}\Psi(I-e^{-\Psi T})^{2}}{(I+\Psi\nu^{-1})^{2}}\right)} &= 
        \underbrace{\sum_{\substack{k\in \mathcal{K}\\ 0 \leq k \leq \alpha}}\frac{d^{2\alpha+2}(1-e^{-\sigma_{k}^{2}\sigma_{k_{0}}^{-2}\log{(d^{\alpha})}})^{2}}{(1+\sigma_{k}^{2}d^{\alpha+1})^{2}}N_{k}\sigma_{k}^{2}}_{(i)}\\
        &+\underbrace{\sum_{\substack{k\in \mathcal{K}\\ \alpha \leq k \leq k_{0}}}\frac{d^{2\alpha+2}(1-e^{-\sigma_{k}^{2}\sigma_{k_{0}}^{-2}\log{(d^{\alpha})}})^{2}}{(1+\sigma_{k}^{2}d^{\alpha+1})^{2}}N_{k}\sigma_{k}^{2}}_{(ii)}\\
        &+\underbrace{\sum_{\substack{k\in \mathcal{K}\\ k_{0}+2 \leq k \leq d}}\frac{d^{2\alpha+2}(1-e^{-\sigma_{k}^{2}\sigma_{k_{0}}^{-2}\log{(d^{\alpha})}})^{2}}{(1+\sigma_{k}^{2}d^{\alpha+1})^{2}}N_{k}\sigma_{k}^{2}}_{(iii)}\\
        &+ \underbrace{\sum_{\substack{k\in \mathcal{K}\\ k > d}}\frac{d^{2\alpha+2}(1-e^{-\sigma_{k}^{2}\sigma_{k_{0}}^{-2}\log{(d^{\alpha})}})^{2}}{(1+\sigma_{k}^{2}d^{\alpha+1})^{2}}N_{k}\sigma_{k}^{2}}_{(iv)}.
    \end{split}
\end{equation}

For the first sum, since $k\leq \alpha$ and $\alpha=O(1)$ then, in particular, $k\ll d$ and we can use $\sigma_{k}^{2} \asymp d^{-k-1}$ and $N_{k} \asymp d^{k}$ such that
\begin{equation}
\begin{split}
    (i) &\asymp \sum_{\substack{k\in \mathcal{K}\\ 0 \leq k \leq \alpha}} \frac{d^{2\alpha+2}(1-e^{-d^{-k+k_{0}\log{(d^{\alpha})}}})^{2}}{(1+d^{-k-1}d^{\alpha+1})^{2}}d^{-k-1}d^{k}\\
    &\asymp d^{2\alpha+1}\sum_{\substack{k\in \mathcal{K}\\ 0 \leq k \leq \alpha}} \frac{(1-e^{-d^{-k+k_{0}\log{(d^{\alpha})}}})^{2}}{(d^{\alpha-k})^{2}}
\end{split}
\end{equation}
where we used that $k<\alpha$ such that the denominator is dominated by the term $d^{\alpha-k}$. Moreover, since $k\leq \alpha \leq k_{0}$ then the exponential term is too small and we obtain
\begin{equation}
\begin{split}
    (i) &\asymp \sum_{\substack{k\in \mathcal{K}\\ 0 \leq k \leq \alpha}}d^{2k+1} \lesssim d^{2\alpha+1}.
\end{split}
\end{equation}

Now, for the second term, since $k \leq k_{0}$ and $k_{0}=O(1)$ we can use again that $\sigma_{k}^{2}\asymp d^{-k-1}$ and $d^{k}$. Also, the exponential is still very small. However, the dominant term in the denominator is now $1$ and we get
\begin{equation}
   (ii) \asymp d^{2\alpha+2} \sum_{\substack{k\in \mathcal{K}\\ \alpha < k \leq k_{0}}} d^{-k-1}d^{k} \asymp d^{2\alpha+1}.
\end{equation}

For the third term the exponential is non-negligible and we use a Taylor expansion for it. Moreover, since $k\leq d$ and $d$ is large we have
\begin{equation}
\begin{split}
    \sigma_{k}^{2} &\asymp \frac{d^{d-1}}{(d+k)^{d+k}}\left(\frac{e}{2}\right)^{k}[(k-2)(k-3)...(k/2)]^{2}\\
    &\gtrsim \frac{d^{d-1}}{(2d)^{d+ k}} = d^{-k-1}2^{-d-k},
\end{split}
\end{equation}
such that
\begin{equation}\label{eq:generaldenominator}
    (1+\sigma_{k}^{2}d^{\alpha+1})^{2} \gtrsim (1+d^{\alpha-k}2^{-d-k}) \asymp 1.
\end{equation}
since we are in the case where $\alpha\leq k_{0} \leq k$.

Now, for $k\leq d$ we have
\begin{equation}
\begin{split}
    N_{k}\sigma_{k}^{2} &\asymp \frac{\sqrt{d}(2k+d)[(k-2)(k-3)...(k/2)]^{2}}{2^{k}k!(d+k)^{\frac{5}{2}}}\\
    &\lesssim \frac{d^{\frac{1}{2}+1}}{d^{\frac{5}{2}}}g(k) = \frac{1}{d}g(k)
\end{split}
\end{equation}
where
\begin{equation}
    g(k) \coloneqq \frac{[(k-2)(k-3)...(k/2)]^{2}}{2^{k}k!}
\end{equation}
and one can check that this function is decreasing in $k$.

Since in the range $2k_{0}+2 \leq k \leq d$ we cannot use $\sigma_{k}\asymp d^{-k-1}$ we prove the following result which will suffice for our purpose.

\begin{claim}\label{eq:lemmasigmadecreasing}
    We have that the sequence $(\sigma_{k})_{k \text{ even}}$ is decreasing. Moreover, for evert $k\in \mathbb{N}$ even we have
    \begin{equation}
        \frac{\sigma_{k+2}^{2}}{\sigma_{k}^{2}}  \lesssim \frac{1}{d^{2}}.
    \end{equation}
\end{claim}

\begin{proofclaim}
    Using \eqref{eq:sigmakasymptotics} we have
    \begin{equation}
    \begin{split}
        \frac{\sigma_{k+2}^{2}}{\sigma_{k}^{2}} &\asymp \frac{(d+k)^{d+k}}{(d+k+2)^{d+k+2}}\left(\frac{e}{2}\right)^{2}\left(\frac{k}{k/2}\right)^{2}\\
        &\asymp \frac{1}{(d+k+2)^{2}}\left(\frac{d+k}{d+k+1}\right)^{d+k}\\
    \end{split} 
    \end{equation}
Now, using that $\log{x}\leq x-1$ for $x=1 -\frac{1}{d+k+1}$ we obtain
\begin{equation}
    \left(\frac{d+k}{d+k+1}\right)^{d+k} \leq e^{-1}
\end{equation}
such that
\begin{equation}
    \begin{split}
        \frac{\sigma_{k+2}^{2}}{\sigma_{k}^{2}} &\lesssim \frac{1}{d^{2}}
    \end{split}
\end{equation}
as we wanted. \claimqed
\end{proofclaim}
With the previous claim and calculations, we get
\begin{equation}
    (ii) \lesssim d^{2\alpha+2}\sum_{\substack{k\in \mathcal{K}\\ k_{0}+2 \leq k \leq d}} (d^{-k+k_{0}}\log{(d^{\alpha})})^2\frac{1}{d}g(k).
\end{equation}

Now, we need to subdivide into the cases where $k\ll d$ and $k=cd$ with $c \in(0,1]$. 

\noindent\underline{Subcase 1.1: $k\ll d$.} The function $g(k)$ does not contribute in leading order and we also note that the leading order is $k=k_{0}+2$ such that
\begin{equation}\label{eq:numerator(ii)1}
\begin{split}
    (ii) &\lesssim d^{2\alpha+2}d^{-4}(\log{(d^{\alpha})})^{2}d^{-1}\\
    &\asymp d^{2k_{0}-3}(\log{(d^{\alpha})})^{2}.
\end{split}   
\end{equation}

\noindent\underline{Subcase 1.2: $k=cd,\ c\in(0,1]$.} We can use a Stirling approximation for $g(k)$ and get
\begin{equation}
\begin{split}
    g(k) &=\frac{\left[\frac{(k-2)!}{(\frac{k}{2}-1)}\right]^{2}}{2^k k!} \asymp \frac{\left[k^{\frac{1}{2}}\left(\frac{k}{e}\right)^{k-2}k^{-\frac{1}{2}}\left(\frac{k}{2e}\right)^{-\frac{k}{2}+1}\right]^{2}}{2^{k}k^{\frac{1}{2}}\left(\frac{k}{e}\right)^{k}}\\
    &\asymp \frac{k^{2k-4}e^{-2k+4}k^{-k+2}2^{k-2}e^{k-2}}{2^{k}k^{\frac{1}{2}}k^{k}e^{-k}}\\
    &\asymp \frac{k^{k-2}e^{-k+2}2^{k-2}}{k^{k+\frac{1}{2}}e^{-k}2^{k}}\\
    &\asymp k^{-\frac{5}{2}}
 \end{split}
\end{equation}
such that
\begin{equation}\label{eq:}
    \begin{split}
        (iii) &\lesssim \sum_{\substack{k\in \mathcal{K}\\ k_{0}+2 \leq k \leq d\\ k=cd}} d^{2\alpha+2}[d^{-k+k_{0}}\log{(d^{\alpha})}]^{2} d^{-1}k^{-\frac{5}{2}}\\
        &\asymp d^{2\alpha+1}(\log{(d^{\alpha})})^{2}\sum_{\substack{k\in \mathcal{K}\\ k_{0}+2 \leq k \leq d\\ k=cd}} d^{-2cd+2k_{0}}(cd)^{-\frac{5}{2}}\\
        &\asymp d^{2\alpha + 2k_{0}-\frac{3}{2}}(\log{(d^{\alpha})})^{2}\sum_{\substack{k\in \mathcal{K}\\ k_{0}+2 \leq k \leq d\\ k=cd}}d^{-2cd}
    \end{split}
\end{equation}
which is superexponentially small in d. 

Finally, for the fourth term we can use the fact that $\sigma_{k}^{2}$ decreases in $k$ from Lemma \ref{eq:lemmasigmadecreasing} and \eqref{eq:generaldenominator} to conclude that the denominator is again lower bounded by something of order $1$. Using Lemma \eqref{eq:lemmasigmadecreasing} again in the exponential and Taylor expansion we obtain
\begin{equation}
    (iv) \lesssim \sum_{\substack{k\in \mathcal{K}\\ k > d}} d^{2\alpha+2}[d^{-k+k_{0}}\log{(d^{\alpha})}]^{2}N_{k}\sigma_{k}^{2}.
\end{equation}
Now, for $k>d$ we have $d$ and $k$ large such that
\begin{equation}
\begin{split}
    N_{k}\sigma_{k}^{2} &\asymp \frac{\sqrt{d}(2k+d)}{(d+k)^{\frac{5}{2}}}g(k)\\
    &\lesssim d^{\frac{1}{2}-\frac{5}{2}}kk^{-\frac{5}{2}}\\
    &\lesssim d^{-2}k^{-\frac{3}{2}} 
\end{split}
\end{equation}
where we used the fact that we already know $g(k)\asymp k^{-\frac{5}{2}}$ when $k$ is large. Then, we get
\begin{equation}
    (iv) \lesssim \sum_{\substack{k\in \mathcal{K}\\ k > d}}d^{2\alpha+2}d^{-2k+2k_{0}}(\log{(d^{\alpha})})^{2}d^{-2}k^{-\frac{3}{2}}
\end{equation}
we observe that the dominant term is for $k=d+1$ and that the series converges. We obtain
\begin{equation}
\begin{split}
    (iv) &\lesssim d^{2\alpha+2}\sum_{\substack{k\in \mathcal{K}\\ k > d}}d^{2k_{0}-2k}(\log{d^{\alpha}})^{2}d^{-2}k^{-\frac{3}{2}}\\
    &\asymp d^{2\alpha+2k_{0}-2d}(\log{(d^{\alpha})})^{2}(d+1)^{-\frac{3}{2}}
\end{split}
\end{equation}
where we used that the dominant term is for $k=d+1$.

\noindent\textbf{Case 2: $\alpha>k_{0}$.} We use the stopping-time
\begin{equation}
    T = \frac{1}{\sigma_{k_{0}}^{2}}\log{(d^{2\alpha-k_{0}})}
\end{equation}

We divide the sum into four terms again
\begin{equation}
    \begin{split}
        \tr{\left(\frac{\nu^{-2}\Psi(I-e^{-\Psi T})^{2}}{(I+\Psi\nu^{-1})^{2}}\right)} &= 
        \underbrace{\sum_{\substack{k\in \mathcal{K}\\ 0 \leq k \leq k_{0}}}\frac{d^{2\alpha+2}(1-e^{-\sigma_{k}^{2}\sigma_{k_{0}}^{-2}\log{(d^{2\alpha-k_{0}})}})^{2}}{(1+\sigma_{k}^{2}d^{\alpha+1})^{2}}N_{k}\sigma_{k}^{2}}_{(i)}\\
        &+\underbrace{\sum_{\substack{k\in \mathcal{K}\\ k_{0}+2 \leq k \leq \alpha}}\frac{d^{2\alpha+2}(1-e^{-\sigma_{k}^{2}\sigma_{k_{0}}^{-2}\log{(d^{2\alpha-k_{0}})}})^{2}}{(1+\sigma_{k}^{2}d^{\alpha+1})^{2}}N_{k}\sigma_{k}^{2}}_{(ii)}\\
        &+\underbrace{\sum_{\substack{k\in \mathcal{K}\\ \alpha+2 \leq k \leq d}}\frac{d^{2\alpha+2}(1-e^{-\sigma_{k}^{2}\sigma_{k_{0}}^{-2}\log{(d^{2\alpha-k_{0}})}})^{2}}{(1+\sigma_{k}^{2}d^{\alpha+1})^{2}}N_{k}\sigma_{k}^{2}}_{(iii)}\\
        &+ \underbrace{\sum_{\substack{k\in \mathcal{K}\\ k > d}}\frac{d^{2\alpha+2}(1-e^{-\sigma_{k}^{2}\sigma_{k_{0}}^{-2}\log{(d^{2\alpha-k_{0}})}})^{2}}{(1+\sigma_{k}^{2}d^{\alpha+1})^{2}}N_{k}\sigma_{k}^{2}}_{(iv)}.
    \end{split}
\end{equation}

For the first term, since $k\leq k_{0}$ we can use $\sigma_{k}^{2}\asymp d^{-k-1}$ and $N_{k} \asymp d^{k}$. Also, using that $k\leq k_{0} < \alpha$ we get
\begin{equation}\begin{split}
    (i) &\asymp \sum_{\substack{k\in \mathcal{K}\\ 0 \leq k \leq k_{0}}}\frac{d^{2\alpha}(1-e^{-d^{-k+k_{0}}\log{(d^{2\alpha-k_{0}})}})^{2}}{(1+d^{-k-1}d^{\alpha+1})^{2}} d^{-k-1}d^{k}\\
    &\asymp \sum_{\substack{k\in \mathcal{K}\\ 0 \leq k \leq k_{0}}}\frac{d^{2\alpha+2}d^{-1}}{(d^{\alpha-k})^{2}}\\
    &\asymp d^{2k_{0}+1}.
\end{split}
\end{equation}

For the second term we need to use Taylor expansion in the exponential and we get
\begin{equation}
\begin{split}
    (ii) &\lesssim \sum_{\substack{k\in \mathcal{K}\\ k_{0}+1 \leq k \leq \alpha}} \frac{d^{2\alpha+1}(d^{-k+k_{0}}\log{(d^{2\alpha-k_{0}})})^{2}}{(d^{\alpha-k})^{2}}\\
    &\asymp d^{2\alpha+1}d^{2k_{0}-2\alpha}(\log{(d^{2\alpha-k_{0}})})^{2}\\
    &\asymp d^{2k_{0}+1}(\log{(d^{2\alpha-k_{0}})})^{2}.
\end{split}
\end{equation}

For the third term, we have the same reasoning as previous $(iii)$ but now the logarithm term changed and also the dominant term is $k=\alpha+1$. 

\noindent\underline{Subcase 2.1: $k\ll d$.} We get
\begin{equation}\begin{split}
    (iii) &\lesssim d^{2\alpha+2}(d^{k_{0}-\alpha-1}\log{(d^{2\alpha-k_{0}})})^{2}d^{-1}\\
    &\lesssim d^{2k_{0}-1}(\log{(d^{2\alpha-k_{0}})})^{2}
\end{split}
\end{equation}
\noindent\underline{Subcase 2.2: $k=cd,\ c\in(0,1]$.} We get
\begin{equation}
    (iii) \lesssim d^{2\alpha + 2k_{0}-\frac{3}{2}}(\log{(d^{2\alpha-k_{0}})})^{2}\sum_{\substack{k\in \mathcal{K}\\ k_{0}+2 \leq k \leq d\\ k=cd}}d^{-2cd}.
\end{equation}

The fourth term is also the same as before just changing the logarithm term, we get
\begin{equation}
    (iv) \lesssim d^{2\alpha+2k_{0}-2d}(\log{(d^{\alpha})})^{2}. 
\end{equation}

\noindent\textbf{Combining the trace bounds.}

\noindent\textbf{Case 1: $0<\alpha\leq k_{0}$.}
\begin{equation}\label{eq:student4.2alpha1}
\begin{split}
4)(2) = \left\langle \Psi^{\frac{1}{2}}\beta,
\frac{\frac{\Psi}{(I+\Psi \nu^{-1})^{2}}\tr{\left(\frac{\nu^{-2}\Psi(I-e^{-\Psi T})^{2}}{(I+\Psi\nu^{-1})^{2}}\right)}}{d^{\alpha}-\tr{\left(\frac{\Psi^{2}\nu^{-2}}{(I+\Psi\nu^{-1})^{2}}\right)}}\Psi^{\frac{1}{2}}\beta\right\rangle &= \sum_{\substack{k\in \mathcal{K}\\ 1 \leq k \leq k_{0}}}\frac{d^{2\alpha+1}}{d^{\alpha}}\frac{\sigma_{k}^{4}}{(1+d^{-k+1+\alpha+1})^{2}}\sum_{l=1}^{N_{k}}\beta_{k,l}^{2}\\
&\asymp d^{\alpha+1}d^{-k_{0}-1}\\
&\asymp d^{\alpha-k_{0}}.
\end{split}
\end{equation}
\noindent\textbf{Case 2: $\alpha>k_{0}$.}
\begin{equation}\label{eq:student4.2alpha2}
\begin{split}
4)(2) = \left\langle \Psi^{\frac{1}{2}}\beta,
\frac{\frac{\Psi}{(I+\Psi \nu^{-1})^{2}}\tr{\left(\frac{\nu^{-2}\Psi(I-e^{-\Psi T})^{2}}{(I+\Psi\nu^{-1})^{2}}\right)}}{d^{\alpha}-\tr{\left(\frac{\Psi^{2}\nu^{-2}}{(I+\Psi\nu^{-1})^{2}}\right)}}\Psi^{\frac{1}{2}}\beta\right\rangle &= \sum_{\substack{k\in \mathcal{K}\\ 1 \leq k \leq k_{0}}}\frac{d^{2k_{0}+1}}{d^{\alpha}}\frac{\sigma_{k}^{4}}{(1+d^{-k+1+\alpha+1})^{2}}\sum_{l=1}^{N_{k}}\beta_{k,l}^{2}\\
&\asymp \frac{d^{2k_{0}+1}}{d^{\alpha}}\frac{d^{-k_{0}-1}}{d^{2\alpha-2k_{0}}}\\
&\asymp d^{-3\alpha+3k_{0}}.
\end{split}
\end{equation}

Finally, comparing the terms \eqref{eq:student1alpha1},\eqref{eq:student23alpha1}, \eqref{eq:student4.1alpha1}, \eqref{eq:student4.2alpha1} and \eqref{eq:student1alpha2},\eqref{eq:student23alpha2}, \eqref{eq:student4.1alpha2}, \eqref{eq:student4.2alpha2} we obtain
\begin{equation}
    \overline{\mathcal{L}}_{\mathrm{ST}}(T) \asymp \begin{cases}
        1,\; \text{ if $\alpha \in (0,\;k_{0}]$},\;\\
        d^{-2\alpha+2k_{0}},\; \text{ if $\alpha > k_{0}$}.\\
    \end{cases}
\end{equation}

Comparing with \eqref{eq:teacherreluerror}, under the unit norm hypothesis, we obtain
\begin{equation}
    \overline{\mathcal{L}}_{\mathrm{ST}}(T) \asymp \overline{\mathcal{L}}_{\mathrm{TE}}^{2}
\end{equation}
as we wanted.
\end{proof}

\section{More general settings}\label{app:generalsettings}

We provide a more complete analysis of the weak-to-strong phenomenon. We consider here the case of a general stopping-time with a pure harmonic target function and the case of a general stopping-time with a general target function that can span all degrees up to $k_{0}$. In both cases, we analyze the interplay between the parameters $\alpha$, $\gamma$, and $k_{0}$, related to the teacher capacity, the stopping-time, and the target complexity, respectively. We are again assuming $\lVert f^{*}\rVert_{\mathcal{D}}=1$ and $\beta_{0,l}=0$ for all $l$. We use $\overline{P}_0=\nu(\Psi+\nu I)^{-1}$ and $S_{T_{\gamma}}=(I-e^{-\Psi T_{\gamma}})^2$ throughout this appendix.

\subsection{General stopping-time and pure harmonic target}\label{subapp:generaltime}

With the purpose of understanding how the stopping-time alone influence on the weak-to-strong phenomenon we study what happens when we choose
\begin{equation}
    T_{\gamma} = \frac{1}{\sigma_{\gamma}^{2}}\begin{cases}  \log{(d^{\alpha})},\; \alpha \in (0,\;k_{0}]\\
    \log{(d^{2\alpha-k_{0}})},\; \alpha > k_{0}
    \end{cases}
\end{equation}
for a general $\gamma\in \mathcal{K}$ with $\gamma\geq 1$ but still uses a pure harmonic target function of degree $k_{0}$. This is the same stopping-time of Section \ref{app:reluproofs} with the substitution of $\sigma_{k_{0}}^{2}$ by $\sigma_{\gamma}^{2}$ and writing $(\log{M}_{\mathrm{TE}}+\log{\overline{\mathcal{L}}_{\mathrm{TE}}^{-1}})$ explicitly using Theorem \ref{theorem:teachererror}.

In this case, we observe that the teacher error and the denominator trace inside the student error do not change since they do not depend on the stopping-time. Let us calculate each of the student error terms separately.

\noindent\textbf{Term 1:} $\langle \Psi^{\frac{1}{2}}\beta, e^{-2\Psi T_{\gamma}}\Psi^{\frac{1}{2}}\beta\rangle$

\noindent\textbf{Case 1: $0<\alpha\leq k_{0}$.}
\begin{equation}\label{eq:general1.1}
\begin{split}
    1) \coloneqq \langle \Psi^{\frac{1}{2}}\beta, e^{-2\Psi T_{\gamma}}\Psi^{\frac{1}{2}}\beta\rangle &= \sum_{\substack{k\in \mathcal{K}\\ 1 \leq k \leq k_{0}}}\frac{\sigma_{k}^{2}}{e^{2\sigma_{k}^{2}T_{\gamma}}}\sum_{l=1}^{N_{k}}\beta_{k,l}^{2}\\ 
    &\asymp \frac{\sigma_{k_{0}}^{2}\sum_{l=1}^{N_{k_{0}}}\beta_{k_{0},l}^{2}}{\exp{(2\sigma_{k_{0}}^{2}d^{\gamma+1}\log{(d^{\alpha})})}}\\
    &\asymp \frac{1}{(d^{\alpha})^{2d^{\gamma-k_{0}}}}.
\end{split}
\end{equation}

\noindent\textbf{Case 2: $\alpha>k_{0}$.}
\begin{equation}\label{eq:general1.2}
    \begin{split}
    1) = \langle \Psi^{\frac{1}{2}}\beta, e^{-2\Psi T_{\gamma}}\Psi^{\frac{1}{2}}\beta\rangle &\asymp \frac{1}{(d^{2\alpha-k_{0}})^{2d^{\gamma-k_{0}}}}.
    \end{split}
\end{equation}

\noindent\textbf{Term 2:}  $\langle \Psi^{\frac{1}{2}}\beta, -(I-e^{-\Psi T_{\gamma}})^{2}\overline{P}_0\Psi^{\frac{1}{2}}\beta\rangle$

We have 
\begin{equation}
\begin{split}
    2) \coloneqq \langle \Psi^{\frac{1}{2}}\beta, -(I-e^{-\Psi T_{\gamma}})^{2}\overline{P}_0\Psi^{\frac{1}{2}}\beta\rangle &\asymp \sum_{\substack{k\in \mathcal{K}\\ 1 \leq k \leq k_{0}}} \frac{\sigma_{k}^{2}(1-e^{-\sigma_{k}^{2}T_{\gamma}})^{2}}{\sigma_{k}^{2}\nu^{-1}+1}\sum_{l=1}^{N_{k}}\beta_{k,l}^{2}.
\end{split}
\end{equation}

\noindent\textbf{Case 1: $0<\alpha\leq k_{0}$.}
\begin{equation}\label{eq:student2alpha1-general}
\begin{split}
    \langle \Psi^{\frac{1}{2}}\beta, -(I-e^{-\Psi T_{\gamma}})^{2}\overline{P}_0\Psi^{\frac{1}{2}}\beta\rangle &\asymp -\frac{d^{-k_{0}-1}(1-(d^{-\alpha})^{d^{\gamma-k_{0}}})^{2}}{d^{-k_{0}-1+\alpha+1}+1}\sum_{l=1}^{N_{k_{0}}}\beta_{k_{0},l}^{2}\\
    &\asymp -(1-2(d^{-\alpha})^{d^{\gamma-k_{0}}}+(d^{-\alpha})^{2d^{\gamma-k_{0}}}).
\end{split}
\end{equation}

\noindent\textbf{Case 2: $\alpha>k_{0}$.}
\begin{equation}\label{eq:student2alpha2-general}
\begin{split}
    \langle \Psi^{\frac{1}{2}}\beta, -(I-e^{-\Psi T_{\gamma}})^{2}\overline{P}_0\Psi^{\frac{1}{2}}\beta\rangle &\asymp -\frac{d^{-k_{0}-1}(1-(d^{-2\alpha+k_{0}})^{d^{\gamma-k_{0}}})^{2}}{d^{\alpha-k_{0}}} \sum_{l=1}^{N_{k_{0}}}\beta_{k_{0},l}^{2}\\
    &\asymp -d^{-\alpha+k_{0}}(1-2(d^{-2\alpha+k_{0}})^{d^{\gamma-k_{0}}}+ (d^{-2\alpha+k_{0}})^{2d^{\gamma-k_{0}}}).
\end{split}
\end{equation}

\noindent\textbf{Term 3:}  $\langle \Psi^{\frac{1}{2}}\beta, \overline{P}_0(I-e^{-2\Psi T_{\gamma}})\Psi^{\frac{1}{2}}\beta\rangle$

This term is treated similarly to the previous one
\begin{equation}
\begin{split}
    3) \coloneqq \langle \Psi^{\frac{1}{2}}\beta, \overline{P}_0(I-e^{-2\Psi T_{\gamma}})\Psi^{\frac{1}{2}}\beta\rangle &\asymp \sum_{\substack{k\in \mathcal{K}\\ 1 \leq k \leq k_{0}}} \frac{\sigma_{k}^{2}(1-e^{-2\sigma_{k}^{2}T_{\gamma}})}{\sigma_{k}^{2}\nu^{-1}+1}\sum_{l=1}^{N_{k}}\beta_{k,l}^{2}.
\end{split}
\end{equation}

\noindent\textbf{Case 1: $0<\alpha\leq k_{0}$.}
\begin{equation}
\begin{split}
    \langle \Psi^{\frac{1}{2}}\beta, \overline{P}_0(I-e^{-2\Psi T_{\gamma}})\Psi^{\frac{1}{2}}\beta\rangle &\asymp d^{-k_{0}-1}(1-(d^{-\alpha})^{2d^{\gamma-k_{0}}})\sum_{l=1}^{N_{k_{0}}}\beta_{k_{0},l}^{2}\\
    &\asymp (1-(d^{-\alpha})^{2d^{\gamma-k_{0}}})).
\end{split}
\end{equation}

\noindent\textbf{Case 2: $\alpha>k_{0}$.}
\begin{equation}
\begin{split}
    \langle \Psi^{\frac{1}{2}}\beta, \overline{P}_0(I-e^{-2\Psi T_{\gamma}})\Psi^{\frac{1}{2}}\beta\rangle &\asymp d^{-\alpha+k_{0}}(1-(d^{-2\alpha+k_{0}})^{2d^{\gamma-k_{0}}}).
\end{split}
\end{equation}

Putting together the contributions of the second and third term we get in first order the following.

\noindent\textbf{Case 1: $0<\alpha\leq k_{0}$.}
\begin{equation}\label{eq:general2,3.1}
\begin{split}
    2)+3) &\asymp (2(d^{-\alpha})^{d^{\gamma-k_{0}}}-2(d^{-\alpha})^{2d^{\gamma-k_{0}}})\\
    &\lesssim (d^{-\alpha})^{d^{\gamma-k_{0}}}. 
\end{split}
\end{equation}

\noindent\textbf{Case 2: $\alpha>k_{0}$.}
\begin{equation}\label{eq:general2,3.2}
\begin{split}
    2)+3)&\asymp (2(d^{-2\alpha+k_{0}})^{d^{\gamma-k_{0}}}-2(d^{-2\alpha+k_{0}})^{2d^{\gamma-k_{0}}})d^{-\alpha+k_{0}}\\
    &\lesssim (d^{-2\alpha+k_{0}})^{d^{\gamma-k_{0}}}d^{-\gamma+k_{0}}.
\end{split}
\end{equation}
\noindent\textbf{Term 4:}  $\langle \Psi^{\frac{1}{2}}\beta, \overline{\mathcal{P}}(S_{T_{\gamma}})\Psi^{\frac{1}{2}}\beta\rangle$

By Theorem \ref{theo:doubleresolventapprox} we have
\begin{equation}
    \langle \Psi^{\frac{1}{2}}\beta, \overline{\mathcal{P}}(S_{T_{\gamma}})\Psi^{\frac{1}{2}}\beta\rangle = \left\langle \Psi^{\frac{1}{2}}\beta, \frac{(I-e^{-\Psi T_{\gamma}})^{2}}{(I+\Psi\nu^{-1})^{2}}+\frac{\frac{\Psi}{(I+\Psi \nu^{-1})^{2}}\tr{\left(\frac{\nu^{-2}\Psi(I-e^{-\Psi T_{\gamma}})^{2}}{(I+\Psi\nu^{-1})^{2}}\right)}}{d^{\alpha}-\tr{\left(\frac{\Psi^{2}\nu^{-2}}{(I+\Psi\nu^{-1})^{2}}\right)}} \Psi^{\frac{1}{2}}\beta\right\rangle.
\end{equation}

Again, we treat each term separately. The first one is 
\begin{equation}
\begin{split}
4)(1) \coloneqq \left\langle \Psi^{\frac{1}{2}}\beta,
\frac{(I-e^{-\Psi T_{\gamma}})^{2}}{(I+\Psi\nu^{-1})^{2}}
\Psi^{\frac{1}{2}}\beta\right\rangle
&=\sum_{\substack{k\in \mathcal{K}\\ 1 \leq k \leq k_{0}}}\frac{(1-e^{-\sigma_{k}^{2}T_{\gamma}})^{2}}{(1+d^{\alpha-k})^{2}}d^{-k-1}\sum_{l=1}^{N_{k}}\beta_{k,l}^{2}.
\end{split}
\end{equation}

\noindent\textbf{Case 1: $0<\alpha\leq k_{0}$.}
\begin{equation}\label{eq:general4.1.1}
\begin{split}
\left\langle \Psi^{\frac{1}{2}}\beta,
\frac{(I-e^{-\Psi T_{\gamma}})^{2}}{(I+\Psi\nu^{-1})^{2}}
\Psi^{\frac{1}{2}}\beta\right\rangle
&\asymp (1-(d^{-\alpha})^{d^{\gamma-k_{0}}})^{2}d^{-k_{0}-1}\sum_{l=1}^{N_{k_{0}}}\beta_{k,l}^{2}\\
&\asymp 1.
\end{split}
\end{equation}
\noindent\textbf{Case 2: $\alpha>k_{0}$.}
\begin{equation}\label{eq:general4.1.2}
\begin{split}
\left\langle \Psi^{\frac{1}{2}}\beta,
\frac{(I-e^{-\Psi T_{\gamma}})^{2}}{(I+\Psi\nu^{-1})^{2}}
\Psi^{\frac{1}{2}}\beta\right\rangle
&\asymp \frac{(1-(d^{-2\alpha+k_{0}})^{d^{\gamma-k_{0}}})^{2}}{d^{2\alpha-2k_{0}}}d^{-k_{0}-1}\sum_{l=1}^{N_{k_{0}}}\beta_{k,l}^{2}\\
&\asymp d^{-2\alpha +2k_{0}}
\end{split}
\end{equation}

The second term is 
\begin{equation}
    4)(2) \coloneqq \left\langle \Psi^{\frac{1}{2}}\beta, \frac{\frac{\Psi}{(I+\Psi \nu^{-1})^{2}}\tr{\left(\frac{\nu^{-2}\Psi(I-e^{-\Psi T_{\gamma}})^{2}}{(I+\Psi\nu^{-1})^{2}}\right)}}{d^{\alpha}-\tr{\left(\frac{\Psi^{2}\nu^{-2}}{(I+\Psi\nu^{-1})^{2}}\right)}} \Psi^{\frac{1}{2}}\beta\right\rangle =     \left\langle \Psi^{\frac{1}{2}}\beta, \frac{\frac{\Psi}{(I+\Psi \nu^{-1})^{2}}\tr{\left(\frac{\nu^{-2}\Psi(I-e^{-\Psi T_{\gamma}})^{2}}{(I+\Psi\nu^{-1})^{2}}\right)}}{\tr{\left(\frac{\Psi\nu^{-1}}{(I+\Psi\nu^{-1})^{2}}\right)}} \Psi^{\frac{1}{2}}\beta\right\rangle
\end{equation}
where we used that
\begin{equation}
    \begin{split}
        d^{\alpha} - \tr{\left(\frac{\Psi^{2}\nu^{-2}}{(I+\Psi \nu^{-1})^{2}}\right)} &= \tr{\left(\frac{\Psi\nu^{-1}}{I+\Psi \nu^{-1}}\right)} - \tr{\left(\frac{\Psi^{2}\nu^{-2}}{(I+\Psi \nu^{-1})^{2}}\right)}\\
        &= \tr{\left(\frac{\Psi\nu^{-1}+\Psi^{2}\nu^{-2}-\Psi^{2}\nu^{-2}}{(I+\Psi\nu^{-1})^{2}}\right)}\\
        &= \tr{\left(\frac{\Psi\nu^{-1}}{(I+\Psi\nu^{-1})^{2}}\right)}
    \end{split}
\end{equation}
and this denominator trace will not change with the new stopping-time.

\noindent\textbf{Upper bound for the numerator trace.}

We have
\begin{equation}
    \tr{\left(\frac{\nu^{-2}\Psi(I-e^{-\Psi T_{\gamma}})^{2}}{(I+\Psi\nu^{-1})^{2}}\right)} \asymp \sum_{k\in \mathcal{K}}\frac{d^{2\alpha+2}(1-e^{-\sigma_{k}^{2}T_{\gamma}})^{2}}{(1+\sigma_{k}^{2}d^{\alpha+1})^{2}}N_{k}\sigma_{k}^{2}.
\end{equation}

Now, $k_{0}$ only appears inside the logarithm such that the division $\alpha \in (0,k_{0}]$ for comparison with the teacher error will make a difference only inside the logarithm. However, differently from previous calculation, we now divide the sum with respect to $\alpha, \gamma$ and $d$ instead of $\alpha, k_{0}$ and $d$ as before. Moreover, we need to further subdivide into the case $\alpha \leq \gamma$ and $\alpha>\gamma$.

\noindent\textbf{Case 1: $0<\alpha\leq k_{0}$.}

\noindent\underline{Subcase 1.1: $\gamma\geq\alpha$.}

We subdivide the sum into $4$ terms.

\begin{equation}
\begin{split}
    (i) &\asymp \sum_{\substack{k\in \mathcal{K}\\ 0 \leq k \leq \alpha}} \frac{d^{2\alpha+2}(1-e^{-d^{-k+\gamma}\log{(d^{\alpha})}})^{2}}{(1+d^{-k-1}d^{\alpha+1})^{2}}d^{-k-1}d^{k}\\
    &\sim d^{2\alpha+1}\sum_{\substack{k\in \mathcal{K}\\ 0 \leq k \leq \alpha}}\frac{(1-e^{-d^{\gamma-k}\log{(d^{\alpha})}})^{2}}{(d^{\alpha-k})^{2}}
\end{split}
\end{equation}
since $k \leq \alpha \leq \gamma$ the exponential is negligible and we obtain
\begin{equation}
\begin{split}
    (i) &\lesssim d\sum_{\substack{k\in \mathcal{K}\\ 0 \leq k \leq \alpha}}d^{2k} \lesssim d^{2\alpha+1}.
\end{split}
\end{equation}
The second sum is
\begin{equation}
    (ii) \asymp \sum_{\substack{k\in \mathcal{K}\\ \alpha < k \leq \gamma}}\frac{d^{2\alpha+2}(1-e^{-d^{-k+\gamma}\log{(d^{\alpha}})})^{2}}{(1+d^{-k-1}d^{\alpha+1})^{2}}d^{-k-1}d^{k}
\end{equation}
such that the exponential is negligible again but the dominant term in the denominator is now $1$ and we get
\begin{equation}
\begin{split}
    (ii) &\asymp d^{2\alpha+2}\sum_{\substack{k\in \mathcal{K}\\ \alpha < k \leq \gamma}}d^{-k-1}d^{k}\\
    &\asymp d^{2\alpha+1}
\end{split}
\end{equation}

For the third sum we have
\begin{equation}
    (iii) \asymp \sum_{\substack{k\in \mathcal{K}\\ \gamma < k \leq d}}\frac{d^{2\alpha+2}(1-e^{-\sigma_{k}^{2}T_{\gamma}})^{2}}{(1+\sigma_{k}^{2}d^{\alpha+1})^{2}}N_{k}\sigma_{k}^{2}
\end{equation}
Using the same reasoning as for the specific stopping-time case of Appendix \ref{app:reluproofs} we have
\begin{equation}
    (iii) \lesssim d^{2\alpha+2}\sum_{\substack{k\in \mathcal{K}\\ \gamma < k \leq d}}((2d)^{\gamma-k}\log{(d^{\alpha})})^{2}\frac{g(k)}{d}
\end{equation}

\noindent\textit{Subcase 1.1.1: $k\ll d$.} $\lesssim d^{2\alpha-3}(\log{(d^{\alpha})})^{2}$ by same reasoning as before.

\noindent\textit{Subcase 1.1.2: $k=cd,\ c\in(0,1]$.} $\lesssim d^{2\alpha+2\gamma-\frac{3}{2}}(\log{(d^{\alpha}}))^{2}(2d)^{-2\epsilon d}$ by same reasoning as before.

For the last sum we have
\begin{equation}
\begin{split}
    (iv) &\asymp \sum_{\substack{k\in \mathcal{K}\\ k > d}}\frac{d^{2\alpha+2}(1-e^{-\sigma_{k}^{2}T_{\gamma}})^{2}}{(1+\sigma_{k}^{2}d^{\alpha+1})^{2}}N_{k}\sigma_{k}^{2}\\
    &\lesssim d^{2\alpha}2^{2\gamma-2d}d^{2\gamma-2d}(d+1)^{-\frac{3}{2}}(\log{(d^{\alpha})})^{2}
\end{split}
\end{equation}
by same previous reasoning.

\noindent\underline{Subcase 1.2: $\gamma<\alpha$.}
We again subdivide the sum into $4$ terms.

\begin{equation}
\begin{split}
    (i) &\asymp \sum_{\substack{k\in \mathcal{K}\\ 0 \leq k \leq \gamma}} \frac{d^{2\alpha+2}(1-e^{-d^{-k+\gamma}\log{(d^{\alpha})}})^{2}}{(1+d^{-k-1}d^{\alpha+1})^{2}}d^{-k-1}d^{k}\\
    &\sim d^{2\alpha+1}\sum_{\substack{k\in \mathcal{K}\\ 0 \leq k \leq \alpha}}\frac{(1-e^{-d^{\gamma-k}\log{(d^{\alpha})}})^{2}}{(d^{\alpha-k})^{2}}
\end{split}
\end{equation}
since $k \leq\gamma$ the exponential is negligible and since $k\leq \gamma \leq \alpha$ the denominator is of order $(d^{\alpha-k})^{2}$ and we get
\begin{equation}
\begin{split}
    (i) &\asymp d\sum_{\substack{k\in \mathcal{K}\\ 0 \leq k \leq \gamma}}d^{2k}\\
    &\asymp d^{2\gamma+1}.
\end{split}
\end{equation}
The second sum is
\begin{equation}
    (ii) \asymp \sum_{\substack{k\in \mathcal{K}\\ \gamma < k \leq \alpha}}\frac{d^{2\alpha+2}(1-e^{-d^{-k+\gamma}\log{(d^{\alpha}})})^{2}}{(1+d^{-k-1}d^{\alpha+1})^{2}}d^{-k-1}d^{k},
\end{equation}
but now the dominant term in the denominator is $(d^{\alpha-k})^{2}$ and the exponential is non-negligible such that we use Taylor expansion to get
\begin{equation}
    \begin{split}
        &\lesssim \sum_{\substack{k\in \mathcal{K}\\ \gamma < k \leq \alpha}}\frac{d^{2\alpha+2}(d^{\gamma-k}\log{(d^{\alpha})})^{2}}{(d^{\alpha-k})}d^{-1}\\
        &\asymp d\sum_{\substack{k\in \mathcal{K}\\ \gamma < k \leq \alpha}}d^{2k}(d^{\gamma-k}\log{(d^{\alpha})})^{2}\\
        &\asymp d^{2\gamma+1}(\log{(d^{\alpha})})^{2}
    \end{split}
\end{equation}

For the third sum we have
\begin{equation}
    (iii) \asymp \sum_{\substack{k\in \mathcal{K}\\ \alpha < k \leq d}}\frac{d^{2\alpha+2}(1-e^{-\sigma_{k}^{2}T_{\gamma}})^{2}}{(1+\sigma_{k}^{2}d^{\alpha+1})^{2}}N_{k}\sigma_{k}^{2}
\end{equation}
Using the same reasoning as for the specific stopping-time we have
\begin{equation}
    (iii) \lesssim d^{2\alpha+2}\sum_{\substack{k\in \mathcal{K}\\ \alpha < k \leq d}}((2d)^{\gamma-k}\log{(d^{\alpha})})^{2}\frac{g(k)}{d}
\end{equation}

\noindent\textit{Subcase 1.2.1: $k\ll d$.} $\lesssim d^{2\gamma-3}(\log{(d^{\alpha})})^{2}$ by same reasoning as before.

\noindent\textit{Subcase 1.2.2: $k=cd,\ c\in(0,1]$.} $\lesssim d^{2\alpha+2\gamma-\frac{3}{2}}(\log{(d^{\alpha}}))^{2}(2d)^{-2\epsilon d}$ by same reasoning as before.

For the last sum we have
\begin{equation}
\begin{split}
    (iv) &\asymp \sum_{\substack{k\in \mathcal{K}\\ k > d}}\frac{d^{2\alpha+2}(1-e^{-\sigma_{k}^{2}T_{\gamma}})^{2}}{(1+\sigma_{k}^{2}d^{\alpha+1})^{2}}N_{k}\sigma_{k}^{2}\\
    &\lesssim d^{2\alpha}2^{2\gamma-2d}d^{2\gamma-2d}(d+1)^{-\frac{3}{2}}(\log{(d^{\alpha})})^{2}
\end{split}
\end{equation}
by same previous reasoning.

\noindent\textbf{Case 2: $\alpha>k_{0}$.}

As already discussed, the only things that change are the terms inside the log, i.e., we should substitute $\log{(d^{\alpha})}$ with $\log{(d^{2\alpha-k_{0}})}$.

\noindent\textbf{Combining the trace bounds.}

Now that we have the new expression for the numerator trace, we come back to the calculation of $4)(2)$.

\noindent\textbf{Case 1: $0<\alpha\leq k_{0}$.}

\noindent\underline{Subcase 1.1: $\gamma\geq\alpha$.} 

In this case the numerator trace is of order $d^{2\alpha+1}$ and we get
\begin{equation}\label{eq:general4.2.1.1}
    \begin{split}
        4)(2)  &\asymp \sum_{\substack{k\in \mathcal{K}\\ 1 \leq k \leq k_{0}}} \frac{d^{2\alpha+1}}{d^{\alpha}}\frac{\sigma_{k}^{4}}{(1+d^{\alpha-k})^{2}}\sum_{l=1}^{N_{k}}\beta_{k,l}^{2}\\
        &\asymp d^{\alpha-k_{0}}.
    \end{split}
\end{equation}

\noindent\underline{Subcase 1.2: $\gamma<\alpha$.}

The numerator is now of order $d^{2\gamma+1}$ and we get
\begin{equation}\label{eq:general4.2.1.2}
\begin{split}
    4)(2) &\asymp \sum_{\substack{k\in \mathcal{K}\\ 1 \leq k \leq k_{0}}} \frac{d^{2\gamma+1}}{d^{\alpha}}\frac{\sigma_{k}^{4}}{(1+d^{\alpha-k})^{2}}\sum_{l=1}^{N_{k}}\beta_{k,l}^{2}\\
    &\asymp d^{2\gamma-\alpha-k_{0}}.
\end{split}  
\end{equation}

\noindent\textbf{Case 2: $\alpha>k_{0}$.}

\noindent\underline{Subcase 2.1: $\gamma\geq\alpha$.}
\begin{equation}\label{eq:general4.2.2.1}
    \begin{split}
        4)(2) &\asymp \sum_{\substack{k\in \mathcal{K}\\ 1 \leq k \leq k_{0}}}\frac{d^{2\alpha+1}}{d^{\alpha}}\frac{\sigma_{k}^{4}}{(1+d^{\alpha-k})^{2}}\sum_{l=1}^{N_{k}}\beta_{k,l}^{2}\\\
        &\asymp d^{\alpha+1}\frac{d^{-k_{0}-1}}{(d^{\alpha-k_{0}})^{2}}\\
        &\asymp d^{k_{0}-\alpha}.
    \end{split}
\end{equation}

\noindent\underline{Subcase 2.2: $\gamma<\alpha$.}

\begin{equation}\label{eq:general4.2.2.2}
    \begin{split}
        4)(2) &\asymp d^{2\gamma+1-\alpha}\frac{d^{-k_{0}-1}}{(d^{\alpha-k_{0}})^{2}}\\
        &\asymp d^{2\alpha-3\alpha+k_{0}} \begin{cases}
            \lesssim d^{3k_{0}-3\alpha},\; \gamma \leq k_{0} < \alpha\\
            \asymp d^{k_{0}-\alpha-\epsilon},\; k_{0}\leq \gamma < \alpha.
        \end{cases}
    \end{split}
\end{equation}

In conclusion, for the general stopping-time and pure harmonic target setting we have that the teacher error does not change, that is, it is again given by
\begin{equation}
    \overline{\mathcal{L}}_{\mathrm{TE}} \asymp \begin{cases}
        1,\; \text{ if $\alpha \in (0,\;k_{0}]$},\;\\
        d^{-\alpha+k_{0}},\; \text{ if $\alpha > k_{0}$},\;\\
    \end{cases}
\end{equation}
but the student error changes and is given next.

Let us now make another subdivision. We divide the final student error again into the cases where $\alpha \in (0,k_{0}]$ and $\alpha>k_{0}$, since these are the natural subdivisions for comparison with the teacher error and are also the subdivisions used in the calculation of all terms of the student error. However, instead of subdividing between $\alpha \leq \gamma$ and $\alpha > \gamma$, as we did for the calculation of $4)(2)$, we now subdivide into the three cases $\gamma < k_{0}$, $\gamma=k_{0}$, and $\gamma > k_{0}$. This is natural since the relation between $\gamma$ and $k_{0}$ determines the intuitive cutoff induced by the stopping-time relative to the target degree.

\noindent\textbf{Final student error.}

\noindent\textbf{Case 1: $0<\alpha\leq k_{0}$.}

\noindent\underline{Subcase 1.1: $\gamma<k_{0}$.}

\begin{equation}
\begin{cases}
    1)+2)+3) \lesssim 1,\;\eqref{eq:general1.1},\;\eqref{eq:general2,3.1}\\
    4)(1) \asymp 1,\;\eqref{eq:general4.1.1}\\
    4)(2) \asymp d^{\alpha-k_{0}} \lesssim 1,\; \alpha\leq \gamma,\;\eqref{eq:general4.2.1.1}\\
    \text{or }\asymp d^{2\gamma-\alpha-k_{0}},\; \alpha>\gamma,\;\eqref{eq:general4.2.1.2}
\end{cases}
\end{equation}
such that $4)(1)$ is the dominant term.

\noindent\underline{Subcase 1.2: $\gamma\geq k_{0}$.} 

In this case $\alpha \leq \gamma$ and we have

\begin{equation}
\begin{cases}
    1)+2)+3) \lesssim 1,\;\eqref{eq:general1.1},\;\eqref{eq:general2,3.1}\\
    4)(1) \asymp 1,\;\eqref{eq:general4.1.1}\\
    4)(2) \asymp d^{\alpha-k_{0}} \lesssim 1,\;\eqref{eq:general4.2.1.1}
\end{cases}
\end{equation}
and again $4)(1)$ dominates.

We conclude that for $\alpha \in (0,k_{0}]$, the term $4)(1)$ is always dominant such that $\overline{\mathcal{L}}_{\mathrm{ST}}(T_{\gamma})\asymp 1$ for $\alpha \leq k_{0}$ without any further subdivision.

\noindent\textbf{Case 2: $\alpha>k_{0}$.}

\noindent\underline{Subcase 2.1: $\gamma<k_{0}$.}

In this case $\alpha > \gamma$ and we have 

\begin{equation}
\begin{cases}
    1)(*),\;\eqref{eq:general1.2}\\
    2)+3) (*),\;\eqref{eq:general2,3.2}\\
    4)(1) \asymp d^{-2\alpha+2k_{0}},\;\eqref{eq:general4.1.2}\\
    4)(2) \lesssim d^{2\gamma-3\alpha+k_{0}},\;\eqref{eq:general4.2.2.2}
\end{cases}
\end{equation}

\underline{(*):}
\begin{equation}
\begin{split}
    1) \asymp \frac{1}{(d^{2\alpha-k_{0}})^{2d^{\alpha-k_{0}}}} &\asymp (d^{2\alpha-k_{0}})^{-2d^{\gamma-k_{0}}}\\
    &\asymp e^{-2(2\alpha-k_{0})d^{\gamma-k_{0}}\log{d}}
\end{split}
\end{equation}
and the exponential approximates $1$ when $d$ is large such that $1) \asymp 1$ greater than $4)(1)$. Similar thing happens for $2)+3)$ such that the dominant term is $1)+2)+3)$ and we don't have weak-to-strong generalization in this case.

\noindent\underline{Subcase 2.2: $\gamma=k_{0}$.}

Again, $\alpha > \gamma$ here and we get

\begin{equation}
\begin{cases}
    1) \asymp d^{-4\alpha+2k_{0}},\;\eqref{eq:general1.2}\\
    2)+3) \lesssim d^{-3\alpha+2k_{0}},\;\eqref{eq:general2,3.2}\\
    4)(1) \asymp d^{-2\alpha+2k_{0}},\;\eqref{eq:general4.1.2}\\
    4)(2) \asymp d^{-3\alpha+3k_{0}},\;\eqref{eq:general4.2.2.2} 
\end{cases}
\end{equation}
and $4)(1)$ dominates such that we have weak-to-strong generalization in this case.

\noindent\underline{Subcase 2.3: $\gamma>k_{0}$.}

\begin{equation}
\begin{cases}
    1)(**),\;\eqref{eq:general1.2}\\
    2)+3)(**),\;\eqref{eq:general2,3.2}\\
    4)(1) \asymp d^{2k_{0}-2\alpha},\;\eqref{eq:general4.1.2}\\
    4)(2) \asymp d^{k_{0}-\alpha},\; \alpha \leq \gamma,\;\eqref{eq:general4.2.2.1}\\
    \text{or} \asymp  d^{2\gamma-3\alpha+k_{0}} \asymp d^{k_{0}-\alpha-2\epsilon},\; \alpha>\gamma,\; \epsilon \coloneqq \alpha-\gamma>0,\;\eqref{eq:general4.2.2.2}
\end{cases}
\end{equation}

\underline{(**):}

\begin{equation}
    \begin{split}
        1) &\asymp e^{-2(2\alpha-k_{0})d^{\gamma-k_{0}}\log{d}}\\
        &\lesssim e^{-2(2\alpha-k_{0})d^{\gamma-k_{0}}\log{d}}
    \end{split}
\end{equation}
which decays exponentially in $d$ while $4)(1)$ decays polynomially. The same thing happens for $2)+3)$ such that the dominant term in this case is $4)(2)$ for $\alpha \leq \gamma$ and we don't have weak-to-strong generalization. For $\alpha>\gamma$, depending on $\epsilon$, the dominant term will be $4)(1)$ (quadratic weak-to-strong) or $4)(2)$ (nonquadratic weak-to-strong) and we characterize below the precise regimes.

\subsubsection{Summary of general stopping-time and pure harmonic target}\label{app:generaltime}

The deterministic equivalent of the teacher population error satisfies 
\begin{equation}
\overline{\mathcal{L}}_{\mathrm{TE}}
\asymp
\begin{cases}
1,\;
& \alpha\leq k_{0},\;\\[2mm]
d^{-\alpha+k_{0}},\;
& \alpha>k_{0}.
\end{cases}
\label{eq:teacher-error-general-stopping}
\end{equation}

The behavior of the student depends on the position of the stopping-time
exponent \(\gamma\) relative to \(k_{0}\) and \(\alpha\). More precisely, the deterministic equivalent of the student population error satisfies
\begin{equation}
\overline{\mathcal{L}}_{\mathrm{ST}}(T_{\gamma})
\asymp
\begin{cases}
1,\;
& \alpha\leq k_{0},\;\\[2mm]
1,\;
& \alpha>k_{0},\; \gamma<k_{0},\;\\[2mm]
d^{-2\alpha+2k_{0}},\;
& \alpha>k_{0},\;
  k_{0}\leq\gamma\leq
  \dfrac{\alpha+k_{0}}{2},\;\\[4mm]
d^{2\gamma-3\alpha+k_{0}},\;
& \alpha>k_{0},\;
  \dfrac{\alpha+k_{0}}{2}<\gamma<\alpha,\;\\[4mm]
d^{-\alpha+k_{0}},\;
& \alpha>k_{0},\; \gamma\geq\alpha.
\end{cases}
\label{eq:student-error-general-stopping}
\end{equation}

Where the quadratic and non-quadratic separation comes from the fact that for the subcase $\gamma>k_{0}$ we have that the $1)+2)+3)$ terms are exponentially small in $d$ and when $\alpha> \gamma$, that is, $\alpha > k_{0}$, we have

\begin{equation}
\begin{cases}
    4)(1) \asymp d^{2k_{0}-2\alpha}\\
    4)(2) \asymp  d^{2\gamma-3\alpha+k_{0}} \asymp d^{k_{0}-\alpha-2\epsilon},\; \epsilon=\alpha-\gamma>0
\end{cases}
\end{equation}
such that the following happens.

\textbf{Quadratic WTS} happens when $k_{0}-\alpha-2\epsilon \leq$ $-2\alpha+2k_{0}$ because in this case $4)(1)$ dominates. We have
\begin{equation}
    \begin{split}
        k_{0}-\alpha-2\epsilon &\leq -2\alpha+2k_{0}\\
        -3\alpha+2\gamma &\leq -2\alpha +k_{0}\\
        2\gamma &\leq \alpha+k_{0}\\
        \gamma &\leq \frac{\alpha+k_{0}}{2}.
    \end{split}
\end{equation}

\textbf{Non-quadratic WTS} happens when $-2\alpha + 2k_{0} < k_{0}-\alpha-2\epsilon < k_{0}-\alpha$ because in this case $4)(2)$ dominates, while still being asymptotically smaller than the teacher error. The last inequality always holds since $\epsilon>0$, and for the first one we have
\begin{equation}
    \begin{split}
        -2\alpha+2k_{0} &< k_{0}-\alpha-2\epsilon\\
        k_{0} &< \alpha-2\epsilon\\
        k_{0} &<\alpha-2\alpha+2\gamma\\
        \frac{k_{0}+\alpha}{2} &< \gamma.
    \end{split}
\end{equation}

\subsection{General stopping-time and general target}\label{subapp:gentimetarget}

Let us set the notation
\begin{equation}
    \beta_{k}^{2} \coloneqq \sum_{l=1}^{N_{k}}\beta_{k,l}^{2}.
\end{equation}

Then, considering again the general stopping-time
\begin{equation}
    T_{\gamma} = \frac{1}{\sigma_{\gamma}^{2}}(\log{M_{\mathrm{TE}}}+\log{\overline{\mathcal{L}}_{\mathrm{TE}}^{-1}}) 
\end{equation}
and a general target with contributions from $k\in \mathcal{K}$ with $1\leq k\leq k_{0}$ and no restriction on the coefficients beyond the unit norm hypothesis, we obtain the following.

\noindent\textbf{Teacher error.}

\noindent\textbf{Case 1: $0<\alpha\leq k_{0}$.}
\begin{equation}\begin{split}
    \overline{\mathcal{L}}_{\mathrm{TE}} &=\nu \langle \Psi^{\frac{1}{2}}\beta, (\Psi+\nu)^{-1}\Psi^{\frac{1}{2}}\beta \rangle\\
    &\asymp d^{-\alpha-1}\left\{\sum_{\substack{k\in \mathcal{K}\\ 1 \leq k \leq \alpha}} \frac{d^{-k-1}}{d^{-k-1}}\beta_{k}^{2}+\sum_{\substack{k\in \mathcal{K}\\ \alpha < k \leq k_{0}}}\frac{d^{-k-1}}{d^{-\alpha-1}}\beta_{k}^{2}\right\}\\
    &\asymp \sum_{\substack{k\in \mathcal{K}\\ 1 \leq k \leq k_{0}}}(d^{-\alpha-1}\mathbb{1}\{1\leq k \leq \alpha\}+d^{-k-1}\mathbb{1}\{\alpha <k\leq k_{0}\})\beta_{k}^{2}.
\end{split}
\end{equation}

\noindent\textbf{Case 2: $\alpha>k_{0}$.}
\begin{equation}
    \overline{\mathcal{L}}_{\mathrm{TE}}\asymp d^{-\alpha-1}\left\{\sum_{\substack{k\in \mathcal{K}\\ 1 \leq k \leq k_{0}}}\frac{d^{-k-1}}{d^{-k-1}}\beta_{k}^{2}\right\} \asymp \sum_{\substack{k\in \mathcal{K}\\ 1 \leq k \leq k_{0}}}d^{-\alpha-1}\beta_{k}^{2}.
\end{equation}

\noindent\textbf{Student error.} 

The numerator and denominator traces do not change from the case of general stopping-time and pure harmonic since they do not involve $\beta$.

\noindent\textbf{Term 1:}  $\langle \Psi^{\frac{1}{2}}\beta, e^{-2\Psi T_{\gamma}}\Psi^{\frac{1}{2}}\beta\rangle$

\noindent\textbf{Case 1: $0<\alpha\leq k_{0}$.}
\begin{equation}\label{eq:gen1.1}
\begin{split}
    1) &\coloneqq \langle \Psi^{\frac{1}{2}}\beta, e^{-2\Psi T_{\gamma}}\Psi^{\frac{1}{2}}\beta\rangle = \sum_{\substack{k\in \mathcal{K}\\ 1 \leq k \leq k_{0}}}\frac{\sigma_{k}^{2}}{e^{2\sigma_{k}^{2}T_{\gamma}}}\beta_{k}^{2}\\ 
    &\asymp \sum_{\substack{k\in \mathcal{K}\\ 1 \leq k \leq k_{0}}}\frac{d^{-k-1}}{\exp{(2\sigma_{k}^{2}d^{\gamma+1}\log{[d^{\alpha}(\sum_{\substack{k'\in \mathcal{K}\\ 1 \leq k' \leq k_{0}}}(d^{-\alpha-1}\mathbb{1}\{1\leq k'\leq \alpha\}+d^{-k'-1}\mathbb{1}\{\alpha < k'\leq k_{0}\})\beta_{k'}^{2})^{-1}]})}}\beta_{k}^{2}\\
    &\asymp \sum_{\substack{k\in \mathcal{K}\\ 1 \leq k \leq k_{0}}}\frac{d^{-k-1}}{[d^{\alpha}(\sum_{\substack{k'\in \mathcal{K}\\ 1 \leq k' \leq k_{0}}}(d^{-\alpha-1}\mathbb{1}\{1\leq k'\leq \alpha\}+d^{-k'-1}\mathbb{1}\{\alpha < k'\leq k_{0}\})\beta_{k'}^{2})^{-1}]^{2d^{\gamma-k}}}\beta_{k}^{2}\\
    &\asymp \sum_{\substack{k\in \mathcal{K}\\ 1 \leq k \leq k_{0}}}d^{-k-1-2\alpha d^{\gamma-k}}\left[\sum_{\substack{k'\in \mathcal{K}\\ 1 \leq k' \leq k_{0}}}(d^{-\alpha-1}\mathbb{1}\{1\leq k'\leq \alpha\}+d^{-k'-1}\mathbb{1}\{\alpha < k'\leq k_{0}\})\beta_{k'}^{2}\right]^{2d^{\gamma-k}} \beta_{k}^{2}
\end{split}
\end{equation}

\noindent\textbf{Case 2: $\alpha>k_{0}$.}
\begin{equation}\label{eq:gen1.2}
    \begin{split}
     1) &=\langle \Psi^{\frac{1}{2}}\beta, e^{-2\Psi T_{\gamma}}\Psi^{\frac{1}{2}}\beta\rangle \asymp \sum_{\substack{k\in \mathcal{K}\\ 1 \leq k \leq k_{0}}}\frac{d^{-k-1}}{[d^{2\alpha+1}(\sum_{\substack{k'\in \mathcal{K}\\ 1 \leq k' \leq k_{0}}}\beta_{k'}^{2})^{-1}]^{2d^{\gamma-k}}}\beta_{k}^{2}\\
     &\asymp \sum_{\substack{k\in \mathcal{K}\\ 1 \leq k \leq k_{0}}}d^{-k-1-2(2\alpha+1)d^{\gamma-k}}\left[\sum_{\substack{k'\in \mathcal{K}\\ 1 \leq k' \leq k_{0}}}\beta_{k'}^{2}\right]^{2d^{\gamma-k}}\beta_{k}^{2}
    \end{split}
\end{equation}

\noindent\textbf{Term 2:}  $\langle \Psi^{\frac{1}{2}}\beta, -(I-e^{-\Psi T_{\gamma}})^{2}\overline{P}_0\Psi^{\frac{1}{2}}\beta\rangle$

\noindent\textbf{Case 1: $0<\alpha\leq k_{0}$.}
\begin{equation}\label{eq:gen2.1}
\begin{split}
2) &\coloneqq \langle \Psi^{\frac{1}{2}}\beta, -(I-e^{-\Psi T_{\gamma}})^{2}\overline{P}_0\Psi^{\frac{1}{2}}\beta\rangle  = -\sum_{\substack{k\in \mathcal{K}\\ 1 \leq k \leq k_{0}}} \frac{\sigma_{k}^{2}(1-e^{-\sigma_{k}^{2}T_{\gamma}})^{2}}{\sigma_{k}^{2}\nu^{-1}+1}\beta_{k}^{2}\\
    &\asymp -\sum_{\substack{k\in \mathcal{K}\\ 1 \leq k \leq k_{0}}}(1-2\left[d^{-\alpha}\sum_{\substack{k'\in \mathcal{K}\\ 1 \leq k' \leq k_{0}}}(d^{-\alpha-1}\mathbb{1}\{1\leq k'\leq \alpha\}+d^{-k'-1}\mathbb{1}\{\alpha < k'\leq k_{0}\})\beta_{k'}^{2}\right]^{d^{\gamma-k}}\\
    &+\left[d^{-\alpha}\sum_{\substack{k'\in \mathcal{K}\\ 1 \leq k' \leq k_{0}}}(d^{-\alpha-1}\mathbb{1}\{1\leq k'\leq \alpha\}+d^{-k'-1}\mathbb{1}\{\alpha < k'\leq k_{0}\})\beta_{k'}^{2}\right]^{2d^{\gamma-k}})\\
    &\cdot(d^{-\alpha-1}\mathbb{1}\{1\leq k \leq \alpha\} + d^{-k-1}\mathbb{1}\{\alpha<k\leq k_{0}\})\beta_{k}^{2}.
\end{split}
\end{equation}

\noindent\textbf{Case 2: $\alpha>k_{0}$.}
\begin{equation}\label{eq:gen2.2}
\begin{split}
    2)&\asymp -\sum_{\substack{k\in \mathcal{K}\\ 1 \leq k \leq k_{0}}}\left(1-2\left[d^{-2\alpha-1}\sum_{\substack{k'\in \mathcal{K}\\ 1 \leq k' \leq k_{0}}}\beta_{k'}^{2}\right]^{d^{\gamma-k}}+\left[d^{-2\alpha-1}\sum_{\substack{k'\in \mathcal{K}\\ 1 \leq k' \leq k_{0}}}\beta_{k'}^{2}\right]^{2d^{\gamma-k}}\right)d^{-\alpha-1}\beta_{k}^{2}.
\end{split}
\end{equation}

\noindent\textbf{Term 3:} $\langle \Psi^{\frac{1}{2}}\beta, \overline{P}_0(I-e^{-2\Psi T_{\gamma}})\Psi^{\frac{1}{2}}\beta\rangle$

\noindent\textbf{Case 1: $0<\alpha\leq k_{0}$.}
\begin{equation}\label{eq:gen3.1}
\begin{split}
    3)&\asymp \sum_{\substack{k\in \mathcal{K}\\ 1 \leq k \leq k_{0}}}\left(1-\left[d^{-\alpha}\sum_{\substack{k'\in \mathcal{K}\\ 1 \leq k' \leq k_{0}}}(d^{-\alpha-1}\mathbb{1}\{1\leq k'\leq\alpha\}+d^{-k'-1}\mathbb{1}\{\alpha<k'\leq k_{0}\})\beta_{k'}^{2}\right]^{2d^{\gamma-k}}\right)\\
    &\cdot (d^{-\alpha-1}\mathbb{1}\{1\leq k \leq \alpha\}+d^{-k-1}\mathbb{1}\{\alpha<k\leq k_{0}\})\beta_{k}^{2}.
\end{split}
\end{equation}

\noindent\textbf{Case 2: $\alpha>k_{0}$.}
\begin{equation}\label{eq:gen3.2}
\begin{split}
     3)&\asymp \sum_{\substack{k\in \mathcal{K}\\ 1 \leq k \leq k_{0}}}\left(1-\left[d^{-2\alpha-1}\sum_{\substack{k'\in \mathcal{K}\\ 1 \leq k' \leq k_{0}}}\beta_{k'}^{2}\right]^{2d^{\gamma-k}}\right)d^{-\alpha-1}\beta_{k}^{2}.
\end{split}
\end{equation}

Now, summing terms $2$ and $3$ we obtain as follows. 

\noindent\textbf{Case 1: $0<\alpha\leq k_{0}$.}
\begin{equation}\label{eq:gen2,3.1}
\begin{split}
    2)+3) &\asymp \sum_{\substack{k\in \mathcal{K}\\ 1 \leq k \leq k_{0}}}(2\left[d^{-\alpha}\sum_{\substack{k'\in \mathcal{K}\\ 1 \leq k' \leq k_{0}}}(d^{-\alpha-1}\mathbb{1}\{1\leq k'\leq \alpha\}+d^{-k'-1}\mathbb{1}\{\alpha<k'\leq k_{0}\})\beta_{k'}^{2}\right]^{d^{\gamma-k}}\\
    &-2\left[d^{-\alpha}\sum_{\substack{k'\in \mathcal{K}\\ 1 \leq k' \leq k_{0}}}(d^{-\alpha-1}\mathbb{1}\{1\leq k'\leq \alpha\}+d^{-k'-1}\mathbb{1}\{\alpha < k' \leq k_{0}\})\beta_{k'}^{2}\right]^{2d^{\gamma-k}})\\
    &\cdot (d^{-\alpha-1}\mathbb{1}\{1\leq k \leq \alpha\}+d^{-k-1}\mathbb{1}\{\alpha < k \leq k_{0}\})\beta_{k}^{2}.
\end{split}
\end{equation}

\noindent\textbf{Case 2: $\alpha>k_{0}$.}
\begin{equation}\label{eq:gen2,3.2}
\begin{split}
    2)+3)&\asymp \sum_{\substack{k\in \mathcal{K}\\ 1 \leq k \leq k_{0}}}\left(2\left[d^{-2\alpha-1}\sum_{\substack{k'\in \mathcal{K}\\ 1 \leq k' \leq k_{0}}}\beta_{k'}^{2}\right]^{d^{\gamma-k}}-2\left[d^{-2\alpha-1}\sum_{\substack{k'\in \mathcal{K}\\ 1 \leq k' \leq k_{0}}}\beta_{k'}^{2}\right]^{2d^{\gamma-k}}\right)d^{-\alpha-1}\beta_{k}^{2}.
\end{split}
\end{equation}
\noindent\textbf{Term 4:} $\langle \Psi^{\frac{1}{2}}\beta, \overline{\mathcal{P}}(S_{T_{\gamma}})\Psi^{\frac{1}{2}}\beta\rangle$

For the first term,
\begin{equation}
\begin{split}
4)(1) \coloneqq \left\langle \Psi^{\frac{1}{2}}\beta,
\frac{(I-e^{-\Psi T_{\gamma}})^{2}}{(I+\Psi\nu^{-1})^{2}}
\Psi^{\frac{1}{2}}\beta\right\rangle
&=\sum_{\substack{k\in \mathcal{K}\\ 1 \leq k \leq k_{0}}}\frac{(1-e^{-\sigma_{k}^{2}T_{\gamma}})^{2}}{(1+d^{\alpha-k})^{2}}d^{-k-1}\beta_{k}^{2}.
\end{split}
\end{equation}

\noindent\textbf{Case 1: $0<\alpha\leq k_{0}$.}
\begin{equation}\label{eq:gen4.1.1}
\begin{split}
4)(1)
&\asymp \sum_{\substack{k\in \mathcal{K}\\ 1 \leq k \leq k_{0}}}\left(1-\left[d^{-\alpha}\sum_{\substack{k'\in \mathcal{K}\\ 1 \leq k' \leq k_{0}}}(d^{-\alpha-1}\mathbb{1}\{1\leq k' \leq \alpha\}+d^{-k'-1}\mathbb{1}\{\alpha < k' \leq k_{0}\})\beta_{k'}^{2}\right]^{d^{\gamma-k}}\right)^{2}\\
&\cdot(d^{-2\alpha+k-1}\mathbb{1}\{1\leq k \leq \alpha\}+d^{-k-1}\mathbb{1}\{\alpha < k \leq k_{0}\})\beta_{k}^{2}.
\end{split}
\end{equation}

\noindent\textbf{Case 2: $\alpha>k_{0}$.}
\begin{equation}\label{eq:gen4.1.2}
\begin{split}
4)(1)
&\asymp\sum_{\substack{k\in \mathcal{K}\\ 1 \leq k \leq k_{0}}}\left(1-\left[d^{-2\alpha-1}\sum_{\substack{k'\in \mathcal{K}\\ 1 \leq k' \leq k_{0}}}\beta_{k'}^{2}\right]^{d^{\gamma-k}}\right)^{2}d^{-2\alpha+k-1}\beta_{k}^{2}.
\end{split}
\end{equation}

For the second term,
\begin{equation}
    4)(2) \coloneqq \left\langle \Psi^{\frac{1}{2}}\beta,
\frac{\frac{\Psi}{(I+\Psi \nu^{-1})^{2}}\tr{\left(\frac{\nu^{-2}\Psi(I-e^{-\Psi T_{\gamma}})^{2}}{(I+\Psi\nu^{-1})^{2}}\right)}}{d^{\alpha}-\tr{\left(\frac{\Psi^{2}\nu^{-2}}{(I+\Psi\nu^{-1})^{2}}\right)}}\Psi^{\frac{1}{2}}\beta\right\rangle,
\end{equation}
we have to subdivide into cases as follows.

\noindent\textbf{Case 1: $0<\alpha\leq k_{0}$.}

\noindent\underline{Subcase 1.1: $\gamma\geq\alpha$.}

In this case the numerator trace is of order $d^{2\alpha+1}$ and we get
\begin{equation}\label{eq:gen4.2.1.1}
    \begin{split}
        4)(2)  &\asymp \sum_{\substack{k\in \mathcal{K}\\ 1 \leq k \leq k_{0}}}\frac{d^{2\alpha+1}}{d^{\alpha}}\frac{d^{-2k-2}}{(1+d^{\alpha-k})^{2}}\beta_{k}^{2}\\
        &\asymp \sum_{\substack{k\in \mathcal{K}\\ 1 \leq k \leq k_{0}}}d^{\alpha-2k-1}(d^{-2\alpha+2k}\mathbb{1}\{1\leq k \leq \alpha\}+\mathbb{1}\{\alpha < k \leq k_{0}\})\beta_{k}^{2}\\
        &\asymp \sum_{\substack{k\in \mathcal{K}\\ 1 \leq k \leq k_{0}}}(d^{-\alpha-1}\mathbb{1}\{1\leq k \leq \alpha\}+d^{\alpha-2k-1}\mathbb{1}\{\alpha < k \leq k_{0}\})\beta_{k}^{2}
    \end{split}
\end{equation}

\noindent\underline{Subcase 1.2: $\gamma<\alpha$.}

The numerator is now of order $d^{2\gamma+1}$ and we get
\begin{equation}\label{eq:gen4.2.1.2}
\begin{split}
    4)(2) &\asymp \sum_{\substack{k\in \mathcal{K}\\ 1 \leq k \leq k_{0}}}\frac{d^{2\gamma+1}}{d^{\alpha}}\frac{d^{-2k-2}}{(1+d^{\alpha-k})^{2}}\beta_{k}^{2}\\
    &\asymp \sum_{\substack{k\in \mathcal{K}\\ 1 \leq k \leq k_{0}}}d^{2\gamma-\alpha-2k-1}(d^{-2\alpha+2k}\mathbb{1}\{1\leq k \leq \alpha\} + \mathbb{1}\{\alpha<k\leq k_{0}\})\beta_{k}^{2}\\
    &\asymp \sum_{\substack{k\in \mathcal{K}\\ 1 \leq k \leq k_{0}}}(d^{2\gamma-3\alpha-1}\mathbb{1}\{1\leq k\leq \alpha\} + d^{2\gamma-\alpha-2k-1}\mathbb{1}\{\alpha<k\leq k_{0}\})\beta_{k}^{2}
\end{split}  
\end{equation}

\noindent\textbf{Case 2: $\alpha>k_{0}$.}

\noindent\underline{Subcase 2.1: $\gamma\geq\alpha$.}
\begin{equation}\label{eq:gen4.2.2.1}
    \begin{split}
        4)(2) &\asymp \sum_{\substack{k\in \mathcal{K}\\ 1 \leq k \leq k_{0}}}\frac{d^{2\alpha+1}}{d^{\alpha}}\frac{d^{-2k-2}}{(1+d^{\alpha-k})^{2}}\beta_{k}^{2}\\
        &\asymp \sum_{\substack{k\in \mathcal{K}\\ 1 \leq k \leq k_{0}}}d^{\alpha-2k-1}d^{-2\alpha+2k}\beta_{k}^{2} \\
        &\asymp d^{-\alpha-1}\sum_{\substack{k\in \mathcal{K}\\ 1 \leq k \leq k_{0}}}\beta_{k}^{2}.
    \end{split}
\end{equation}

\noindent\underline{Subcase 2.2: $\gamma<\alpha$.}

\begin{equation}\label{eq:gen4.2.2.2}
    \begin{split}
        4)(2) &\asymp \sum_{\substack{k\in \mathcal{K}\\ 1 \leq k \leq k_{0}}}\frac{d^{2\gamma+1}}{d^{\alpha}}\frac{d^{-2k-2}}{(d^{\alpha-k})^{2}}\beta_{k}^{2}\\
        &\asymp \sum_{\substack{k\in \mathcal{K}\\ 1 \leq k \leq k_{0}}}d^{2\gamma-\alpha-2k-1}d^{-2\alpha+2k}\beta_{k}^{2}\\
        &\asymp d^{2\gamma-3\alpha-1}\sum_{\substack{k\in \mathcal{K}\\ 1 \leq k \leq k_{0}}}\beta_{k}^{2}.
    \end{split}
\end{equation}

\subsubsection{Summary of general stopping-time and general target}\label{app:generaltimetarget}

\noindent\textbf{Regime 1: $0<\alpha\leq k_{0},\ \gamma<\alpha$.} In this regime, the teacher error satisfies
\begin{equation}
\overline{\mathcal{L}}_{\mathrm{TE}}\asymp\sum_{\substack{k\in \mathcal{K}\\ 1 \leq k \leq k_{0}}}\left(d^{-\alpha-1}\mathbb{1}\{1\leq k\leq\alpha\}+d^{-k-1}\mathbb{1}\{\alpha<k\leq k_{0}\}\right)\beta_{k}^{2}.
\end{equation}
For the student, from \eqref{eq:gen1.1} we have
\begin{equation}
1)\asymp\sum_{\substack{k\in \mathcal{K}\\ 1 \leq k \leq k_{0}}}d^{-k-1}\left[d^{-\alpha}\sum_{\substack{k'\in \mathcal{K}\\ 1 \leq k' \leq k_{0}}}\left(d^{-\alpha-1}\mathbb{1}\{1\leq k'\leq\alpha\}+d^{-k'-1}\mathbb{1}\{\alpha<k'\leq k_{0}\}\right)\beta_{k'}^{2}\right]^{2d^{\gamma-k}}\beta_{k}^{2}.
\end{equation}
For $k\leq\gamma$, we have $d^{\gamma-k}\geq1$ and
\begin{equation}
d^{-\alpha}\sum_{\substack{k'\in \mathcal{K}\\ 1 \leq k' \leq k_{0}}}\left(d^{-\alpha-1}\mathbb{1}\{1\leq k'\leq\alpha\}+d^{-k'-1}\mathbb{1}\{\alpha<k'\leq k_{0}\}\right)\beta_{k'}^{2}\asymp d^{-\alpha}\overline{\mathcal{L}}_{\mathrm{TE}}=o(1),
\end{equation}
such that the corresponding contributions in \eqref{eq:gen1.1} are negligible. On the other hand, for $k>\gamma$, we have $d^{\gamma-k}$ going to zero, such that
\begin{equation}
\left[d^{-\alpha}\sum_{\substack{k'\in \mathcal{K}\\ 1 \leq k' \leq k_{0}}}\left(d^{-\alpha-1}\mathbb{1}\{1\leq k'\leq\alpha\}+d^{-k'-1}\mathbb{1}\{\alpha<k'\leq k_{0}\}\right)\beta_{k'}^{2}\right]^{2d^{\gamma-k}}\to1.
\end{equation}
Then, for the first term
\begin{equation}
1)\asymp\sum_{\substack{k\in \mathcal{K}\\ \gamma < k \leq k_{0}}}d^{-k-1}\beta_{k}^{2}+o(\overline{\mathcal{L}}_{\mathrm{TE}}).
\end{equation}

For terms $2)+3)$, from \eqref{eq:gen2,3.1},
\begin{equation}
\begin{split}
2)+3)&\asymp\sum_{\substack{k\in \mathcal{K}\\ 1 \leq k \leq k_{0}}}(2\left[d^{-\alpha}\sum_{\substack{k'\in \mathcal{K}\\ 1 \leq k' \leq k_{0}}}\left(d^{-\alpha-1}\mathbb{1}\{1\leq k'\leq\alpha\}+d^{-k'-1}\mathbb{1}\{\alpha<k'\leq k_{0}\}\right)\beta_{k'}^{2}\right]^{d^{\gamma-k}}\\
&-2\left[d^{-\alpha}\sum_{\substack{k'\in \mathcal{K}\\ 1 \leq k' \leq k_{0}}}\left(d^{-\alpha-1}\mathbb{1}\{1\leq k'\leq\alpha\}+d^{-k'-1}\mathbb{1}\{\alpha<k'\leq k_{0}\}\right)\beta_{k'}^{2}\right]^{2d^{\gamma-k}})\\
&\left(d^{-\alpha-1}\mathbb{1}\{1\leq k\leq\alpha\}+d^{-k-1}\mathbb{1}\{\alpha<k\leq k_{0}\}\right)\beta_{k}^{2}.
\end{split}
\end{equation}
For the same previous reasoning, for $k\leq\gamma$, both powers converge to $0$ since $d^{\gamma-k}\geq1$. For $k>\gamma$, both powers converge to $1$ such that their difference also converges to $0$. Thus, we obtain
\begin{equation}
2)+3)=o(\overline{\mathcal{L}}_{\mathrm{TE}}).
\end{equation}

For $4)(1)$, we have from \eqref{eq:gen4.1.1} that
\begin{equation}
\begin{split}
4)(1)&\asymp\sum_{\substack{k\in \mathcal{K}\\ 1 \leq k \leq k_{0}}}\left(1-\left[d^{-\alpha}\sum_{\substack{k'\in \mathcal{K}\\ 1 \leq k' \leq k_{0}}}\left(d^{-\alpha-1}\mathbb{1}\{1\leq k'\leq\alpha\}+d^{-k'-1}\mathbb{1}\{\alpha<k'\leq k_{0}\}\right)\beta_{k'}^{2}\right]^{d^{\gamma-k}}\right)^{2}\\
&\left(d^{-2\alpha+k-1}\mathbb{1}\{1\leq k\leq\alpha\}+d^{-k-1}\mathbb{1}\{\alpha\leq k\leq k_{0}\}\right)\beta_{k}^{2}.
\end{split}
\end{equation}
For $k\leq\gamma$, the power inside the parentheses is negligible and the term inside becomes $1$. Since $k\leq\gamma<\alpha$, then
\begin{equation}
d^{-2\alpha+k-1}\leq d^{\gamma-\alpha}d^{-\alpha-1}=o(d^{-\alpha-1}),
\end{equation}
and the terms in the sum with $k\leq \gamma$ are of order $o(\overline{\mathcal{L}}_{\mathrm{TE}})$.
For $k>\gamma$, the power inside the parentheses converges to $1$, such that the whole term converges to $0$. Then, 
\begin{equation}
4)(1)=o(\overline{\mathcal{L}}_{\mathrm{TE}}).
\end{equation}

Finally, since $\gamma<\alpha$, from \eqref{eq:gen4.2.1.2} we have
\begin{equation}
4)(2)\asymp\sum_{\substack{k\in \mathcal{K}\\ 1 \leq k \leq k_{0}}}\left(d^{2\gamma-3\alpha-1}\mathbb{1}\{1\leq k\leq\alpha\}+d^{2\gamma-\alpha-2k-1}\mathbb{1}\{\alpha<k\leq k_{0}\}\right)\beta_{k}^{2}.
\end{equation}
Here we do a different case analysis. For $k\leq\alpha$, the ratio between the coefficient in $4)(2)$ and the corresponding coefficient in the teacher error is
\begin{equation}
\frac{d^{2\gamma-3\alpha-1}}{d^{-\alpha-1}}=d^{2\gamma-2\alpha}=o(1).
\end{equation}
For $k>\alpha$, the corresponding ratio is
\begin{equation}
\frac{d^{2\gamma-\alpha-2k-1}}{d^{-k-1}}=d^{2\gamma-\alpha-k}\leq d^{2\gamma-2\alpha}=o(1).
\end{equation}
Thus,
\begin{equation}
4)(2)=o(\overline{\mathcal{L}}_{\mathrm{TE}}).
\end{equation}
We conclude that for Case 1, we obtain
\begin{equation}
\overline{\mathcal{L}}_{\mathrm{ST}}(T_{\gamma})\asymp\sum_{\substack{k\in \mathcal{K}\\ \gamma < k \leq k_{0}}}d^{-k-1}\beta_{k}^{2}+o(\overline{\mathcal{L}}_{\mathrm{TE}}).
\end{equation}
Therefore, weak-to-strong generalization occurs if and only if
\begin{equation}
\sum_{\substack{k\in \mathcal{K}\\ \gamma < k \leq k_{0}}}d^{-k-1}\beta_{k}^{2}=o\left(\overline{\mathcal{L}}_{\mathrm{TE}}\right).
\end{equation}
and then it depends on the target coefficients sizes $\beta_{k}^{2}$ for $k>\gamma$. 

\noindent\textbf{Regime 2: $0<\alpha\leq k_{0},\ \gamma\geq\alpha$.} Observe that in the previous case we only used that $\gamma<\alpha$ when analyzing terms $4)(1)$ and $4)(2)$ such that the other terms remain the same. Thus, we get again
\begin{equation}
\overline{\mathcal{L}}_{\mathrm{TE}}\asymp\sum_{\substack{k\in \mathcal{K}\\ 1 \leq k \leq k_{0}}}\left(d^{-\alpha-1}\mathbb{1}\{1\leq k\leq\alpha\}+d^{-k-1}\mathbb{1}\{\alpha<k\leq k_{0}\}\right)\beta_{k}^{2},
\end{equation}
and also
\begin{equation}
1)\asymp\sum_{\substack{k\in \mathcal{K}\\ \gamma < k \leq k_{0}}}d^{-k-1}\beta_{k}^{2}+o(\overline{\mathcal{L}}_{\mathrm{TE}}),
\end{equation}
and
\begin{equation}
2)+3)=o(\overline{\mathcal{L}}_{\mathrm{TE}}).
\end{equation}

For the first part of term $4)$, from \eqref{eq:gen4.1.1},
\begin{equation}
\begin{split}
4)(1)&\asymp\sum_{\substack{k\in \mathcal{K}\\ 1 \leq k \leq k_{0}}}\left(1-\left[d^{-\alpha}\sum_{\substack{k'\in \mathcal{K}\\ 1 \leq k' \leq k_{0}}}\left(d^{-\alpha-1}\mathbb{1}\{1\leq k'\leq\alpha\}+d^{-k'-1}\mathbb{1}\{\alpha<k'\leq k_{0}\}\right)\beta_{k'}^{2}\right]^{d^{\gamma-k}}\right)^{2}\\
&\left(d^{-2\alpha+k-1}\mathbb{1}\{1\leq k\leq\alpha\}+d^{-k-1}\mathbb{1}\{\alpha<k\leq k_{0}\}\right)\beta_{k}^{2}.
\end{split}
\end{equation}
For $k\leq\gamma$, the power inside the parentheses is negligible, whereas for $k>\gamma$ it converges to $1$. Since $\gamma\geq\alpha$, this gives
\begin{equation}
4)(1)\asymp\sum_{\substack{k\in \mathcal{K}\\ 1 \leq k \leq k_{0}}}d^{-2\alpha+k-1}\mathbb{1}\{1\leq k\leq\alpha\}\beta_{k}^{2}+\sum_{\substack{k\in \mathcal{K}\\ 1 \leq k \leq k_{0}}}d^{-k-1}\mathbb{1}\{\alpha<k\leq\gamma\}\beta_{k}^{2}+o(\overline{\mathcal{L}}_{\mathrm{TE}}),
\end{equation}
that is we divided the last indicator as
\begin{equation}
    \mathbb{1}\{\alpha<k\leq k_{0}\} = \mathbb{1}\{\alpha<k\leq \gamma\} + \mathbb{1}\{\gamma <k\leq k_{0}\} 
\end{equation}
ans used that the last contribution is $o(\overline{\mathcal{L}}_{\mathrm{TE}})$.

In particular, for $\alpha<k\leq\gamma$, term $4)(1)$ has exactly the same coefficient $d^{-k-1}$ as the corresponding term in the teacher error. And for $1\leq k \leq \alpha$ they are smaller than the teacher since in this case
\begin{equation}
    d^{-2\alpha+k-1} \leq d^{-\alpha-1}.
\end{equation}

Finally, since $\gamma\geq\alpha$, from \eqref{eq:gen4.2.1.2} we have
\begin{equation}
4)(2)\asymp\sum_{\substack{k\in \mathcal{K}\\ 1 \leq k \leq k_{0}}}\left(d^{-\alpha-1}\mathbb{1}\{1\leq k\leq\alpha\}+d^{\alpha-2k-1}\mathbb{1}\{\alpha<k\leq k_{0}\}\right)\beta_{k}^{2}.
\end{equation}
For $k\leq\alpha$, term $4)(2)$ has exactly the same coefficient as the corresponding term in the teacher error,
\begin{equation}
\sum_{\substack{k\in \mathcal{K}\\ 1 \leq k \leq k_{0}}}d^{-\alpha-1}\mathbb{1}\{1\leq k\leq\alpha\}\beta_{k}^{2}.
\end{equation}
For $k>\alpha$, the second part of \eqref{eq:gen4.2.1.2} is no larger than the corresponding teacher contribution since
\begin{equation}
\frac{d^{\alpha-2k-1}}{d^{-k-1}}=d^{\alpha-k}=o(1).
\end{equation}
Therefore, the teacher-scale contribution for $k\leq\alpha$ is reproduced by $4)(2)$. For $\alpha<k\leq\gamma$, it is reproduced by $4)(1)$ as said above, while for $k>\gamma$ (if applicable) it is reproduced by $1)$. More precisely,
\begin{equation}
\overline{\mathcal{L}}_{\mathrm{ST}}(T_{\gamma})\asymp\sum_{\substack{k\in \mathcal{K}\\ 1 \leq k \leq k_{0}}}d^{-\alpha-1}\mathbb{1}\{1\leq k\leq\alpha\}\beta_{k}^{2}+\sum_{\substack{k\in \mathcal{K}\\ 1 \leq k \leq k_{0}}}d^{-k-1}\mathbb{1}\{\alpha<k\leq\gamma\}\beta_{k}^{2}+\sum_{\substack{k\in \mathcal{K}\\ 1 \leq k \leq k_{0}}}d^{-k-1}\mathbb{1}\{\gamma<k\leq k_{0}\}\beta_{k}^{2}.
\end{equation}
Since $\gamma\geq\alpha$, the last two sums together give the entire contribution with $k>\alpha$, and hence
\begin{equation}
\overline{\mathcal{L}}_{\mathrm{ST}}(T_{\gamma})\asymp\sum_{\substack{k\in \mathcal{K}\\ 1 \leq k \leq k_{0}}}\left(d^{-\alpha-1}\mathbb{1}\{1\leq k\leq\alpha\}+d^{-k-1}\mathbb{1}\{\alpha<k\leq k_{0}\}\right)\beta_{k}^{2}\asymp\overline{\mathcal{L}}_{\mathrm{TE}}.
\end{equation}
Hence, there is no WTS when $\alpha\leq k_{0}$ and $\gamma\geq\alpha$.

Now, for the cases where $\alpha>k_{0}$ we will need to do a more complete subcase analysis. Instead of only comparing $\alpha$ and $\gamma$ as before we also need to compare how $\gamma$ relates with $k_{0}$.

\noindent\textbf{Regime 3: $\alpha>k_{0},\ \gamma<k_{0}$.} In this regime, we have
\begin{equation}
\overline{\mathcal{L}}_{\mathrm{TE}}\asymp d^{-\alpha-1}\sum_{\substack{k\in \mathcal{K}\\ 1 \leq k \leq k_{0}}}\beta_{k}^{2}.
\end{equation}
Since $1\leq \gamma < k_{0}$, we now separate the target degrees into $k\leq\gamma$ and $k>\gamma$. From \eqref{eq:gen1.2},
\begin{equation}
1)\asymp \sum_{\substack{k\in \mathcal{K}\\ 1 \leq k \leq k_{0}}}d^{-k-1}\left[d^{-2\alpha-1}\sum_{\substack{k'\in \mathcal{K}\\ 1 \leq k' \leq k_{0}}}\beta_{k'}^{2}\right]^{2d^{\gamma-k}}\beta_{k}^{2}.
\end{equation}
For $k\leq\gamma$, we have $d^{\gamma-k}\geq1$ and, since in this case $d^{-2\alpha-1}\sum_{\substack{k'\in \mathcal{K}\\ 1 \leq k' \leq k_{0}}}\beta_{k'}^{2}\asymp d^{-\alpha}\overline{\mathcal{L}}_{\mathrm{TE}}=o(\overline{\mathcal{L}}_{\mathrm{TE}})$, the corresponding contributions of $1)$ are negligible. On the other hand, for $k>\gamma$, we have $d^{\gamma-k}$ going to zero. Then, \eqref{eq:gen1.2} gives
\begin{equation}
1)\asymp \sum_{\substack{k\in \mathcal{K}\\ \gamma < k \leq k_{0}}}d^{-k-1}\beta_{k}^{2}+o(\overline{\mathcal{L}}_{\mathrm{TE}}).
\end{equation}
For $2)+3)$, \eqref{eq:gen2,3.2} gives
\begin{equation}
2)+3)\asymp \sum_{\substack{k\in \mathcal{K}\\ 1 \leq k \leq k_{0}}}\left(2\left[d^{-2\alpha-1}\sum_{\substack{k'\in \mathcal{K}\\ 1 \leq k' \leq k_{0}}}\beta_{k'}^{2}\right]^{d^{\gamma-k}}-2\left[d^{-2\alpha-1}\sum_{\substack{k'\in \mathcal{K}\\ 1 \leq k' \leq k_{0}}}\beta_{k'}^{2}\right]^{2d^{\gamma-k}}\right)d^{-\alpha-1}\beta_{k}^{2}.
\end{equation}
Again, for $k\leq\gamma$, both powers are negligible and for  $k>\gamma$, both powers converge to $1$ and their leading-order contributions cancel. Thus,
\begin{equation}
2)+3)=o(\overline{\mathcal{L}}_{\mathrm{TE}}).
\end{equation}
For the first part of term $4)$, from \eqref{eq:gen4.1.2},
\begin{equation}
4)(1)\asymp \sum_{\substack{k\in \mathcal{K}\\ 1 \leq k \leq k_{0}}}\left(1-\left[d^{-2\alpha-1}\sum_{\substack{k'\in \mathcal{K}\\ 1 \leq k' \leq k_{0}}}\beta_{k'}^{2}\right]^{d^{\gamma-k}}\right)^{2}d^{-2\alpha+k-1}\beta_{k}^{2}.
\end{equation}
For $k\leq\gamma$, the power inside the parentheses is negligible, and the corresponding contribution satisfies
\begin{equation}
\sum_{\substack{k\in \mathcal{K}\\ 1 \leq k \leq \gamma}}d^{-2\alpha+k-1}\beta_{k}^{2}\leq d^{\gamma-\alpha}d^{-\alpha-1}\sum_{\substack{k\in \mathcal{K}\\ 1 \leq k \leq \gamma}}\beta_{k}^{2}=o(\overline{\mathcal{L}}_{\mathrm{TE}}),
\end{equation}
since, in particular, we have $\alpha > \gamma$ in the case considered.
For $k>\gamma$, the power converges to $1$. Thus,
\begin{equation}
4)(1)=o(\overline{\mathcal{L}}_{\mathrm{TE}}).
\end{equation}
Finally, since $\gamma<\alpha$, from \eqref{eq:gen4.2.2.2} we have
\begin{equation}
4)(2)\asymp d^{2\gamma-3\alpha-1}\sum_{\substack{k\in \mathcal{K}\\ 1 \leq k \leq k_{0}}}\beta_{k}^{2}\asymp d^{2\gamma-2\alpha}\overline{\mathcal{L}}_{\mathrm{TE}}=o(\overline{\mathcal{L}}_{\mathrm{TE}}).
\end{equation}
In conclusion, we obtain
\begin{equation}
\overline{\mathcal{L}}_{\mathrm{ST}}(T_{\gamma})\asymp \sum_{\substack{k\in \mathcal{K}\\ \gamma < k \leq k_{0}}}d^{-k-1}\beta_{k}^{2}+o(\overline{\mathcal{L}}_{\mathrm{TE}}).
\end{equation}
Consequently, as in Case 1, in this regime WTS occurs if and only if
the first term above is $o(\overline{\mathcal{L}}_{\mathrm{TE}})$, that is,
\begin{equation}
\sum_{\substack{k\in \mathcal{K}\\ \gamma < k \leq k_{0}}}d^{-k-1}\beta_{k}^{2}=o\left(d^{-\alpha-1}\sum_{\substack{k\in \mathcal{K}\\ 1 \leq k \leq k_{0}}}\beta_{k}^{2}\right).
\end{equation}

\noindent\textbf{Regime 4: $\alpha>k_{0},\ k_{0}\leq\gamma\leq \alpha$.} In this regime, we have
\begin{equation}
\overline{\mathcal{L}}_{\mathrm{TE}}\asymp d^{-\alpha-1}\sum_{\substack{k\in \mathcal{K}\\ 1 \leq k \leq k_{0}}}\beta_{k}^{2}.
\end{equation}
Moreover, since $k\leq k_{0}\leq\gamma$, we have $d^{\gamma-k}\geq1$ for every $k\in \mathcal{K}$ with $1\leq k\leq k_{0}$. Since
\begin{equation}
d^{-2\alpha-1}\sum_{\substack{k'\in \mathcal{K}\\ 1 \leq k' \leq k_{0}}}\beta_{k'}^{2}\asymp d^{-\alpha}\overline{\mathcal{L}}_{\mathrm{TE}},
\end{equation}
we have the first term of the student \eqref{eq:gen1.2} satisfying
\begin{equation}
1)\lesssim \left[d^{-2\alpha-1}\sum_{\substack{k'\in \mathcal{K}\\ 1 \leq k' \leq k_{0}}}\beta_{k'}^{2}\right]^2\sum_{\substack{k\in \mathcal{K}\\ 1 \leq k \leq k_{0}}}d^{-k-1}\beta_k^2 \lesssim d^{-\alpha-2}\overline{\mathcal{L}}_{\mathrm{TE}}\sum_{\substack{k\in \mathcal{K}\\ 1 \leq k \leq k_{0}}}\beta_{k}^{2} \lesssim d^{-1}\overline{\mathcal{L}}_{\mathrm{TE}}^{2}= o(\overline{\mathcal{L}}_{\mathrm{TE}}^{2})
\end{equation}
From \eqref{eq:gen2,3.2}, we get
\begin{equation}
2)+3)\lesssim \left[d^{-2\alpha-1}\sum_{\substack{k'\in \mathcal{K}\\ 1 \leq k' \leq k_{0}}}\beta_{k'}^{2}\right]d^{-\alpha-1}\sum_{\substack{k\in \mathcal{K}\\ 1 \leq k \leq k_{0}}}\beta_k^2\asymp d^{-\alpha}\overline{\mathcal{L}}_{\mathrm{TE}}^2=o(\overline{\mathcal{L}}_{\mathrm{TE}}^2).
\end{equation}
For the first part of term $4)$, \eqref{eq:gen4.1.2} gives
\begin{equation}
4)(1)\asymp \sum_{\substack{k\in \mathcal{K}\\ 1 \leq k \leq k_{0}}}\left(1-\left[d^{-2\alpha-1}\sum_{\substack{k'\in \mathcal{K}\\ 1 \leq k' \leq k_{0}}}\beta_{k'}^{2}\right]^{d^{\gamma-k}}\right)^{2}d^{-2\alpha+k-1}\beta_{k}^{2}.
\end{equation}
The power inside the parentheses is again negligible and we obtain
\begin{equation}
4)(1)\asymp \sum_{\substack{k\in \mathcal{K}\\ 1 \leq k \leq k_{0}}}d^{-2\alpha+k-1}\beta_{k}^{2} \lesssim d^{-2\alpha+k_{0}-1}\sum_{\substack{k\in \mathcal{K}\\ 1 \leq k \leq k_{0}}}\beta_{k}^{2} = d^{k_{0}-\alpha}\overline{\mathcal{L}}_{\mathrm{TE}} = o (\overline{\mathcal{L}}_{\mathrm{TE}})
\end{equation}
Finally, from \eqref{eq:gen4.2.2.2} we have
\begin{equation}
4)(2)\asymp d^{2\gamma-3\alpha-1}\sum_{\substack{k\in \mathcal{K}\\ 1 \leq k \leq k_{0}}}\beta_{k}^{2}\asymp d^{2\gamma-2\alpha}\overline{\mathcal{L}}_{\mathrm{TE}} = o(\overline{\mathcal{L}}_{\mathrm{TE}})
\end{equation}
Combining these estimates, we obtain that the dominant term is $4)(1)$ such that we obtain
\begin{equation}
\overline{\mathcal{L}}_{\mathrm{ST}}(T_{\gamma})\asymp d^{k_{0}-\alpha}\overline{\mathcal{L}}_{\mathrm{TE}}= o(\overline{\mathcal{L}}_{\mathrm{TE}}).
\end{equation}
Then, we have WTS for this case. The quadratic and non-quadratic distinction as done in the pure harmonic case, would require a more careful analysis and will end up depending on the target coefficients.

\noindent\textbf{Regime 5: $\alpha>k_{0},\ \gamma\geq\alpha$.} In this regime, we have again
\begin{equation}
\overline{\mathcal{L}}_{\mathrm{TE}}\asymp d^{-\alpha-1}\sum_{\substack{k\in \mathcal{K}\\ 1 \leq k \leq k_{0}}}\beta_{k}^{2}.
\end{equation}
Moreover, since $k\leq k_{0}<\alpha\leq\gamma$, we have again that $\gamma-k>0$ and $d^{\gamma-k}\geq 1$ for every $k\in \mathcal{K}$ with $1\leq k\leq k_{0}$ such that we obtain the same conclusion as in the previous case for the first term,that is,
\begin{equation}
1)=o(\overline{\mathcal{L}}_{\mathrm{TE}}).
\end{equation}
Similarly, we have the same previous conclusion for \eqref{eq:gen2,3.2}
\begin{equation}
2)+3)=o(\overline{\mathcal{L}}_{\mathrm{TE}}).
\end{equation}
And the same thing for the first part of term $4)$,
\begin{equation}
4)(1)=o(\overline{\mathcal{L}}_{\mathrm{TE}}),
\end{equation}
Finally, since $\alpha\leq\gamma$ now we have to use  \eqref{eq:gen4.2.2.1} for the second part of $4)$ and this gives
\begin{equation}
4)(2)\asymp d^{-\alpha-1}\sum_{\substack{k\in \mathcal{K}\\ 1 \leq k \leq k_{0}}}\beta_{k}^{2}\asymp\overline{\mathcal{L}}_{\mathrm{TE}}.
\end{equation}
Therefore, term $4)(2)$ is the dominant contribution to the student error, and we obtain
\begin{equation}
\overline{\mathcal{L}}_{\mathrm{ST}}(T_{\gamma})\asymp\overline{\mathcal{L}}_{\mathrm{TE}}.
\end{equation}
Then, there is no WTS in this last case.

All regimes are summarized in the table below.

\begin{center}
\small
\renewcommand{\arraystretch}{1.5}
\setlength{\tabcolsep}{5pt}
\begin{tabularx}{\textwidth}{@{}c X c@{}}
\toprule
Regime & Comparison & Conclusion \\
\midrule
$\alpha\leq k_{0},\ \gamma<\alpha$ &
$\displaystyle \text{WTS}\iff\sum_{\substack{k\in \mathcal{K}\\ \gamma < k \leq k_{0}}}d^{-k-1}\beta_{k}^{2}=o(\overline{\mathcal{L}}_{\mathrm{TE}})$ &
Coefficient-dependent \\
\midrule
$\alpha\leq k_{0},\ \gamma\geq\alpha$ &
$\displaystyle \overline{\mathcal{L}}_{\mathrm{ST}}\asymp\overline{\mathcal{L}}_{\mathrm{TE}}$ &
No WTS \\
\midrule
$\alpha>k_{0},\ \gamma<k_{0}$ &
$\displaystyle \text{WTS}\iff\sum_{\substack{k\in \mathcal{K}\\ \gamma < k \leq k_{0}}}d^{-k-1}\beta_{k}^{2}=o\left(d^{-\alpha-1}\sum_{\substack{k\in \mathcal{K}\\ 1 \leq k \leq k_{0}}}\beta_{k}^{2}\right)$ &
Coefficient-dependent \\
\midrule
$\alpha>k_{0},\ k_{0}\leq\gamma<\alpha$ &
$\displaystyle \overline{\mathcal{L}}_{\mathrm{ST}}\asymp d^{k_{0}-\alpha}\overline{\mathcal{L}}_{\overline{\mathcal{L}}}=o(\mathcal{L}_{\mathrm{TE}})$ &
WTS \\
\midrule
$\alpha>k_{0},\ \gamma\geq\alpha$ &
$\displaystyle \overline{\mathcal{L}}_{\mathrm{ST}}\asymp\overline{\mathcal{L}}_{\mathrm{TE}}$ &
No WTS \\
\bottomrule
\end{tabularx}
\captionof{table}{Weak-to-strong generalization regimes determined by the teacher width exponent $\alpha$, stopping-time parameter $\gamma$, and maximal target degree $k_0$. Here WTS means $\overline{\mathcal{L}}_{\mathrm{ST}}=o(\overline{\mathcal{L}}_{\mathrm{TE}})$ as $d\to\infty$.}
\label{tab:wts-regimes}
\end{center}

\begin{obs}[The two coefficient-dependent conditions are one condition]\label{rem:one-condition}
\chg{In the first and third rows of the table, weak-to-strong generalization occurs if and only if $\sum_{\substack{k\in \mathcal{K}\\ \gamma < k \leq k_{0}}}d^{-k-1}\beta_k^2=o(\overline{\mathcal{L}}_{\mathrm{TE}})$. In the third row this is how the condition is written, since $\overline{\mathcal{L}}_{\mathrm{TE}}\asymp d^{-\alpha-1}\sum_{\substack{k\in \mathcal{K}\\ 1 \leq k \leq k_{0}}}\beta_k^2$ there. In the first row $\overline{\mathcal{L}}_{\mathrm{TE}}\asymp\sum_{\substack{k\in \mathcal{K}\\ 1 \leq k \leq \alpha}}d^{-\alpha-1}\beta_k^2+\sum_{\substack{k\in \mathcal{K}\\ \alpha < k \leq k_{0}}}d^{-k-1}\beta_k^2$, and the condition compares the same held-out energy $\sum_{\substack{k\in \mathcal{K}\\ \gamma < k \leq k_{0}}}d^{-k-1}\beta_k^2=\sum_{\substack{k\in \mathcal{K}\\ \gamma < k \leq k_{0}}}\lVert f^{*}_k\rVert_{\mathcal{D}}^2$ with this different expression. The two rows therefore state a single condition, \eqref{eq:wts-condition} in the main text, with the teacher error taking its two forms \eqref{eq:mixed-teacher-error}.}
\end{obs}

\section{Two mixed-target examples}\label{app:mixed-examples}

This appendix proves the two statements about the two-mode target \eqref{eq:two-mode-target} made in Section~\ref{sec:mixed-main}. Both concern the deterministic equivalents \eqref{eq:main-teacher} and \eqref{eq:main-two-term}, and both are uniform over all stopping times $t\ge0$, not only over the harmonic stopping times $T_\gamma$ of Appendix~\ref{app:generalsettings}. They are stated up to constants, which is all that the conclusions of Section~\ref{sec:mixed-main} require; the leading constants, which the figures track, are recorded in Remark~\ref{rem:mixed-constants}. Transferring the conclusions to the random spherical model requires Assumption~\ref{ass:de} at the relevant error scales, uniformly in time; the figures of this appendix are evaluations of the deterministic equivalents, not simulations.

\paragraph{The errors of a mixed target are sums over its harmonic components.}
Since $\Psi$ is the scalar $\sigma_k^2$ on the degree-$k$ harmonic block, every matrix in \eqref{eq:main-teacher} and \eqref{eq:main-two-term} is a function of $\Psi$, hence scalar on each block, and the two errors are sums over the harmonic components $f^\star_k$ of the target, weighted by their energies $\|f^\star_k\|_{\mathcal D}^2$, the squared norms of the blocks of $\Psi^{1/2}\beta$:
\begin{equation}\label{eq:mixed-additivity}
\overline{\mathcal L}_{\rm TE}=\sum_k\|f^\star_k\|_{\mathcal D}^2\,\frac{\nu}{\sigma_k^2+\nu},
\qquad
\overline{\mathcal L}_{\rm ST}(t)=\sum_k\|f^\star_k\|_{\mathcal D}^2\,\overline\ell_k(t),
\end{equation}
where $\overline\ell_k(t)$ is the deterministic equivalent of the student error for a unit target of degree $k$. Reading off \eqref{eq:main-two-term} on the degree-$k$ block,
\begin{equation}\label{eq:mixed-unit-de}
\overline\ell_k(t)
=\Bigl(\frac{\nu+\sigma_k^2e^{-\sigma_k^2t}}{\sigma_k^2+\nu}\Bigr)^2
+\frac{\nu\sigma_k^2}{(\sigma_k^2+\nu)^2}\,\overline H(t),
\qquad
\overline H(t)=\frac{\operatorname{tr}[\Psi(\Psi+\nu I)^{-2}(I-e^{-\Psi t})^2]}{\operatorname{tr}[\Psi(\Psi+\nu I)^{-2}]}\in[0,1].
\end{equation}
The first term is the bias $\overline B(t)$ and the second the variance $\overline V(t)$ of the unit degree-$k$ target; $\overline H$ is the trace ratio of \eqref{eq:main-two-term}, common to all degrees, which increases from $0$ at $t=0$ to $1$ at $t=\infty$. This is the additivity used in Section~\ref{sec:mixed-main}: the held-out energy $\mathcal E_{>\gamma}$ of \eqref{eq:wts-target-tail} and the two forms of the teacher error \eqref{eq:mixed-teacher-error} are \eqref{eq:mixed-additivity} together with Lemma~\ref{lem:mixed-nu-consequences} below. Two consequences of \eqref{eq:mixed-unit-de} are used repeatedly: every term is nonnegative, so a lower bound may drop any of them, and $\overline\ell_k(t)\le1+\frac14=\frac54$ for all $t$, since the bias is at most $1$ and the variance at most $\frac14\overline H\le\frac14$.

Throughout, $m=d^\alpha$ with $\alpha>1$ fixed, so that $\nu\asymp d^{-\alpha-1}$ by Lemma~\ref{lem:nu-relu}. Appendix~\ref{app:reluproofs} uses this estimate through the pure-target calculations; here we need it in the following packaged form.

\begin{lemma}[Consequences of the order of $\nu$]\label{lem:mixed-nu-consequences}
Fix $\alpha>1$ and $m=d^\alpha$, and let $k\in\mathcal K$ be a fixed degree. With constants depending only on $\alpha$ and $k$,
\begin{equation*}
\begin{aligned}
&\text{(i)}\quad \frac{\nu}{\sigma_k^2+\nu}\asymp d^{k-\alpha}\quad(k<\alpha),
&\qquad&\text{(ii)}\quad \frac{\nu}{\sigma_k^2+\nu}\longrightarrow1\quad(k>\alpha),\\[4pt]
&\text{(iii)}\quad \operatorname{tr}\Psi\asymp d^{-1},
&\qquad&\text{(iv)}\quad \nu\operatorname{tr}[\Psi(\Psi+\nu I)^{-2}]\asymp m.
\end{aligned}
\end{equation*}
\end{lemma}

\begin{proof}
\emph{Items~(i) and~(ii).} By Lemma~\ref{lem:nu-relu}, $\nu\asymp d^{-\alpha-1}$, while $\sigma_k^2\asymp d^{-k-1}$ for fixed $k$. Hence
\[
\frac{\sigma_k^2}{\nu}\asymp d^{\alpha-k},
\qquad\text{and}\qquad
\frac{\nu}{\sigma_k^2+\nu}=\frac1{1+\sigma_k^2/\nu}.
\]
If $k<\alpha$ the ratio $\sigma_k^2/\nu$ tends to infinity, and the right-hand side is $\asymp\nu/\sigma_k^2\asymp d^{k-\alpha}$, which is~(i). If $k>\alpha$ the ratio tends to zero, and the right-hand side tends to one, which is~(ii).

\emph{Item~(iii).} The degree-one term alone gives the lower bound, by \eqref{eq:N1sigma1squared}:
\[
\operatorname{tr}\Psi=\sum_{k\in\mathcal K}N_k\sigma_k^2\ \ge\ N_1\sigma_1^2=\frac1{4d}.
\]
For the upper bound, the trace of the kernel operator is the expectation of its diagonal, and $(r)_+^2\le r^2$:
\[
\operatorname{tr}\Psi=\mathbb E_x\,K(x,x)=\mathbb E\,(\langle u,x\rangle)_+^2\ \le\ \mathbb E\,\langle u,x\rangle^2=\frac1d .
\]

\emph{Item~(iv), upper bound.} Since $\frac{\nu\sigma_k^2}{(\sigma_k^2+\nu)^2}\le\frac{\sigma_k^2}{\sigma_k^2+\nu}$ for every $k$, the fixed-point equation \eqref{eq:main-fixed-point} gives
\[
\nu\operatorname{tr}[\Psi(\Psi+\nu I)^{-2}]
=\sum_{k\in\mathcal K}N_k\frac{\nu\sigma_k^2}{(\sigma_k^2+\nu)^2}
\ \le\ \sum_{k\in\mathcal K}N_k\frac{\sigma_k^2}{\sigma_k^2+\nu}=m .
\]

\emph{Item~(iv), lower bound.} Let $J$ be the largest degree in $\mathcal K$ with $J\le\alpha$, and let $D_J=\sum_{k\le J}N_k\asymp d^J$ be the number of eigenvalues of $\Psi$ of degree at most $J$. The first degree in $\mathcal K$ above $J$ exceeds $\alpha$, and $\sigma_k^2$ decreases in $k$ (Claim~\ref{eq:lemmasigmadecreasing}), so by Lemma~\ref{lem:nu-relu}
\[
\sup_{k>J}\frac{\sigma_k^2}{\nu}\longrightarrow0,
\qquad\text{hence}\qquad
\inf_{k>J}\frac{\nu}{\sigma_k^2+\nu}\ \ge\ 1-o(1).
\]
We bound the trace from below by retaining the degrees above $J$ and dropping the degrees up to $J$, which is valid because every summand is nonnegative:
\[
\nu\operatorname{tr}[\Psi(\Psi+\nu I)^{-2}]
=\sum_{k\in\mathcal K}N_k\,\frac{\sigma_k^2}{\sigma_k^2+\nu}\,\frac{\nu}{\sigma_k^2+\nu}
\ \ge\ (1-o(1))\sum_{k>J}N_k\frac{\sigma_k^2}{\sigma_k^2+\nu}.
\]
By the fixed-point equation \eqref{eq:main-fixed-point}, the last sum is $m$ minus the sum over $k\le J$, and each summand there is at most $N_k$:
\[
\sum_{k>J}N_k\frac{\sigma_k^2}{\sigma_k^2+\nu}\ \ge\ m-D_J .
\]
Finally, $D_J/m\to0$ if $J<\alpha$, since $D_J\asymp d^J$, and $D_J/m\to1/J!\le\frac12$ if $J=\alpha\ge2$, by \eqref{eq:Nkasymptotics}. In both cases $m-D_J\gtrsim m$, which completes the proof.
\end{proof}

The pure-target estimates of Appendices~\ref{app:reluproofs} and~\ref{app:generalsettings} bound $\overline H$ at the harmonic stopping times $T_\gamma$. The all-time statements below need the same bound at an arbitrary multiple of a harmonic time scale, which is the following lemma. Its proof is the upper bound on term 4)(2) of Appendix~\ref{subapp:generaltime}, with the stopping time left free.

\begin{lemma}[Variance filter bound]\label{lem:mixed-filter}
Fix $\alpha>1$, $m=d^\alpha$, and a harmonic degree $j\in\mathcal K$ with $j<\alpha$. There is a constant $C_{\alpha,j}$ such that for all $x\ge0$ and all large $d$,
\begin{equation}\label{eq:mixed-filter}
\overline H(x/\sigma_j^2)\le C_{\alpha,j}\,(1+x^2)\,d^{2j-2\alpha}.
\end{equation}
\end{lemma}
\begin{proof}
Multiply numerator and denominator of $\overline H$ by $\nu$. By Lemma~\ref{lem:mixed-nu-consequences} the denominator $\nu\operatorname{tr}[\Psi(\Psi+\nu I)^{-2}]$ is of order $m=d^\alpha$, and the numerator is
\[
\sum_{k\in\mathcal K}N_k\,\frac{\nu\sigma_k^2}{(\sigma_k^2+\nu)^2}\,(1-e^{-\sigma_k^2t})^2,\qquad t=x/\sigma_j^2 .
\]
For $k\le j<\alpha$, bound the filter by $1$ and $\frac{\nu\sigma_k^2}{(\sigma_k^2+\nu)^2}\le\frac{\nu}{\sigma_k^2+\nu}\asymp d^{k-\alpha}$ (Lemma~\ref{lem:mixed-nu-consequences}), so with $N_k\asymp d^k$ these degrees contribute at most a constant times $d^{2j-\alpha}$. For $k>j$, use $1-e^{-\sigma_k^2t}\le\sigma_k^2t=x\sigma_k^2/\sigma_j^2$ and $\frac{\nu\sigma_k^2}{(\sigma_k^2+\nu)^2}\sigma_k^4\le\nu\sigma_k^2$, so these degrees contribute at most
\[
\frac{x^2}{\sigma_j^4}\,\nu\sum_kN_k\sigma_k^2=\frac{x^2}{\sigma_j^4}\,\nu\operatorname{tr}\Psi\asymp x^2\,d^{2j+2}\,d^{-\alpha-1}\,d^{-1}=x^2d^{2j-\alpha},
\]
by $\sigma_j^2\asymp d^{-j-1}$, $\nu\asymp d^{-\alpha-1}$ and $\operatorname{tr}\Psi\asymp d^{-1}$ from Lemmas~\ref{lem:nu-relu} and~\ref{lem:mixed-nu-consequences}. Dividing by the denominator gives \eqref{eq:mixed-filter}.
\end{proof}

\subsection{A small component above the teacher's capacity: no polynomial improvement}\label{subapp:mixed-capacity}

\begin{prop}[No polynomial improvement at any stopping time]\label{prop:mixed-capacity}
Let $m=d^3$ and let $f^\star_d$ be the two-mode target \eqref{eq:two-mode-target} with $w_d=c/d$ for a constant $c>0$. Then
\begin{equation}\label{eq:mixed-capacity-orders}
\overline{\mathcal L}_{\rm TE}\asymp d^{-1},\qquad
\inf_{t\ge0}\overline{\mathcal L}_{\rm ST}(t)\asymp d^{-1},
\end{equation}
and the upper bound is attained at $T_2=4\log(d)/\sigma_2^2$. In particular, both errors vanish, but for every $\eta>0$ the best student error is not $O(d^{-\eta}\,\overline{\mathcal L}_{\rm TE})$: the student improves on the teacher by at most a constant factor, whatever the stopping time.
\end{prop}

\begin{proof}
Here $\alpha=3$, and Lemmas~\ref{lem:nu-relu} and~\ref{lem:mixed-nu-consequences} give $\nu\asymp d^{-4}$, $\frac{\nu}{\sigma_2^2+\nu}\asymp d^{-1}$, and, since the degree-four eigenvalue lies below the capacity edge, $\frac{\nu}{\sigma_4^2+\nu}\to1$. By \eqref{eq:mixed-additivity}, $\overline{\mathcal L}_{\rm TE}=(1-w_d)\frac{\nu}{\sigma_2^2+\nu}+w_d\frac{\nu}{\sigma_4^2+\nu}$, and both summands are of order $d^{-1}$.

\emph{Lower bound.} We bound $\overline{\mathcal L}_{\rm ST}(t)$ from below by retaining only the degree-four bias, without the exponential inside it, and dropping every other term of \eqref{eq:mixed-additivity}--\eqref{eq:mixed-unit-de}; this is valid because all of these terms are nonnegative, and it is the first inequality of Theorem~\ref{theo:quadratic-floor} on one block. For every $t\ge0$,
\begin{equation}\label{eq:mixed-capacity-floor}
\overline{\mathcal L}_{\rm ST}(t)\ \ge\ w_d\,\overline\ell_4(t)\ \ge\ w_d\Bigl(\frac{\nu}{\sigma_4^2+\nu}\Bigr)^2\ \asymp\ d^{-1}.
\end{equation}
This is the energy of the degree-four mode, which the teacher does not learn and which the student therefore never learns either.

\emph{Upper bound.} At $t=T_2$, that is $x=4\log d$ on the degree-two time scale, the degree-two bias in \eqref{eq:mixed-unit-de} is $\bigl(\frac{\nu}{\sigma_2^2+\nu}+O(d^{-4})\bigr)^2\asymp d^{-2}$, and by Lemma~\ref{lem:mixed-filter} with $j=2$ the degree-two variance is at most $\frac{\nu}{\sigma_2^2+\nu}\overline H(T_2)\lesssim d^{-1}(\log d)^2d^{-2}$. Hence $\overline\ell_2(T_2)\asymp d^{-2}$, and with $\overline\ell_4\le\frac54$,
\[
\overline{\mathcal L}_{\rm ST}(T_2)\le\overline\ell_2(T_2)+\tfrac54w_d\lesssim d^{-1}.
\]
Together with \eqref{eq:mixed-capacity-floor} this proves \eqref{eq:mixed-capacity-orders}.
\end{proof}

In the classification of Appendix~\ref{subapp:gentimetarget} this target has $k_0=4>\alpha=3$, and the stop $T_2$ falls in the first row of the table there: the held-out energy $\mathcal E_{>2}=w_d$ is of the same order as $\overline{\mathcal L}_{\rm TE}$, so the condition \eqref{eq:wts-condition} fails. The proposition adds that no other stopping time does better, because the obstruction \eqref{eq:mixed-capacity-floor} is a floor at every time.

\begin{figure}[ht]
\centering
\includegraphics[width=\linewidth]{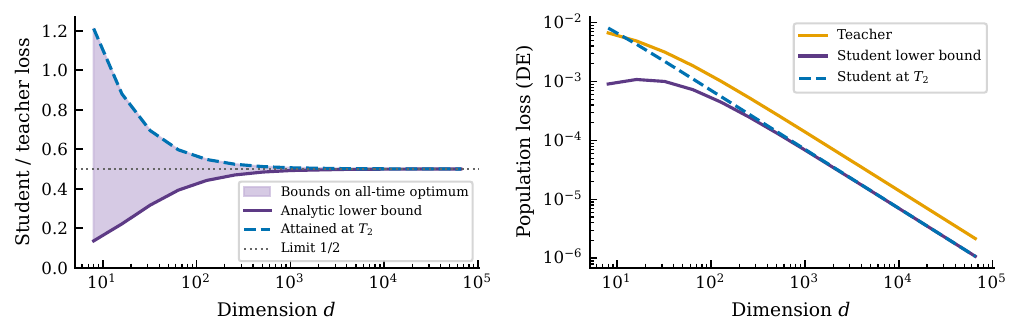}
\caption{Proposition~\ref{prop:mixed-capacity} for the deterministic equivalent. The floor $w_d\bigl(\frac{\nu}{\sigma_4^2+\nu}\bigr)^2/\overline{\mathcal L}_{\rm TE}$ of \eqref{eq:mixed-capacity-floor} and the attained ratio $\overline{\mathcal L}_{\rm ST}(T_2)/\overline{\mathcal L}_{\rm TE}$ bracket the all-time optimum. The figure uses $c=(\pi-3)/2$, for which both tend to $1/2$ (Remark~\ref{rem:mixed-constants}). The companion panel shows that the teacher error and the two student bounds all vanish at rate $d^{-1}$. These are deterministic calculations, not random-teacher observations.}
\label{fig:capacity-obstruction}
\end{figure}

\subsection{Two modes below capacity with incompatible stopping times: no quadratic improvement}\label{subapp:mixed-timing}

\begin{prop}[No quadratic improvement at any stopping time]\label{prop:mixed-timing}
Let $m=d^5$ and let $f^\star_d$ be the two-mode target \eqref{eq:two-mode-target} with $w_d=Bd^{-5}$ for a constant $B>0$. Then
\begin{equation}\label{eq:mixed-timing-orders}
\overline{\mathcal L}_{\rm TE}\asymp d^{-3},\qquad
\inf_{t\ge0}\overline{\mathcal L}_{\rm ST}(t)\asymp d^{-5},
\end{equation}
and the upper bound is attained at $t=x/\sigma_4^2$ for every fixed $x>0$, with constants depending on $x$. Consequently
\begin{equation}\label{eq:mixed-timing-ratios}
\frac{\inf_t\overline{\mathcal L}_{\rm ST}(t)}{\overline{\mathcal L}_{\rm TE}}\asymp d^{-2}\longrightarrow0,
\qquad
\frac{\inf_t\overline{\mathcal L}_{\rm ST}(t)}{\overline{\mathcal L}_{\rm TE}^{\,2}}\asymp d\longrightarrow\infty :
\end{equation}
weak-to-strong generalization occurs, but no stopping time gives quadratic improvement. The best improvement exponent is $5/3$.
\end{prop}

\begin{proof}
Here $\alpha=5$, and Lemmas~\ref{lem:nu-relu} and~\ref{lem:mixed-nu-consequences} give
\[
\nu\asymp d^{-6},\qquad
\frac{\nu}{\sigma_2^2+\nu}\asymp d^{-3},\qquad
\frac{\nu}{\sigma_4^2+\nu}\asymp d^{-1},\qquad
\nu\operatorname{tr}[\Psi(\Psi+\nu I)^{-2}]\asymp d^{5}.
\]
By \eqref{eq:mixed-additivity}, $\overline{\mathcal L}_{\rm TE}=(1-w_d)\frac{\nu}{\sigma_2^2+\nu}+w_d\frac{\nu}{\sigma_4^2+\nu}\asymp d^{-3}$, the second summand being of order $d^{-6}$.

\emph{Lower bound.} Fix $t\ge0$ and write $x=\sigma_4^2t$. We bound $\overline{\mathcal L}_{\rm ST}(t)$ from below by retaining two terms of \eqref{eq:mixed-additivity}--\eqref{eq:mixed-unit-de} and dropping all the others, which is valid because every term is nonnegative. The first retained term is the degree-four bias, in which $\frac{\nu+\sigma_4^2e^{-x}}{\sigma_4^2+\nu}\ge\frac{(\nu+\sigma_4^2)e^{-x}}{\sigma_4^2+\nu}=e^{-x}$, so
\[
w_d\,\overline\ell_4(t)\ \ge\ w_d\,e^{-2x}.
\]
The second retained term is the degree-two variance, which we bound from below by retaining only the degree-four summand of the numerator trace of $\overline H$ and dropping the other summands, again valid because they are nonnegative:
\[
(1-w_d)\,\overline\ell_2(t)\ \ge\ (1-w_d)\,\frac{\nu\sigma_2^2}{(\sigma_2^2+\nu)^2}\,
\frac{N_4\,\frac{\nu\sigma_4^2}{(\sigma_4^2+\nu)^2}}{\nu\operatorname{tr}[\Psi(\Psi+\nu I)^{-2}]}\,(1-e^{-x})^2
=:A_d\,(1-e^{-x})^2,
\]
where, by the estimates above and $N_4\asymp d^4$,
\[
A_d\asymp d^{-3}\cdot\frac{d^{4}\cdot d^{-1}}{d^{5}}=d^{-5}.
\]
Combining the two retained terms, for every $t\ge0$,
\begin{equation}\label{eq:mixed-timing-envelope}
\overline{\mathcal L}_{\rm ST}(t)\ \ge\ w_d\,e^{-2x}+A_d\,(1-e^{-x})^2
\ \ge\ \min_{0\le y\le1}\bigl[w_dy^2+A_d(1-y)^2\bigr]
=\frac{w_dA_d}{w_d+A_d}\ \asymp\ d^{-5},
\end{equation}
since $w_d$ and $A_d$ are both of order $d^{-5}$. The minimum over $y=e^{-x}$ is computed by completing the square.

\emph{Upper bound.} Let $t=x/\sigma_4^2$ with $x>0$ fixed. The degree-four component contributes at most $\frac54w_d\lesssim d^{-5}$. For the degree-two component, $\sigma_2^2t=x\sigma_2^2/\sigma_4^2\asymp xd^2$, so the exponential in its bias is negligible and the bias is $\asymp\bigl(\frac{\nu}{\sigma_2^2+\nu}\bigr)^2\asymp d^{-6}$; its variance is at most $\frac{\nu}{\sigma_2^2+\nu}\overline H(t)\lesssim d^{-3}(1+x^2)d^{-2}$ by Lemma~\ref{lem:mixed-filter} with $j=4$. Hence $\overline{\mathcal L}_{\rm ST}(x/\sigma_4^2)\lesssim d^{-5}$, which with \eqref{eq:mixed-timing-envelope} proves \eqref{eq:mixed-timing-orders}; \eqref{eq:mixed-timing-ratios} follows by division.
\end{proof}

Each pure mode separately admits quadratic improvement under its own stopping rule (Theorem~\ref{theo:phases} with $k_0=2$ or $k_0=4$ and $\alpha=5$). The mixture cannot: learning the degree-four signal requires $\sigma_4^2t$ of order one, and the same filter $1-e^{-\sigma_4^2t}$ admits the degree-four error that the teacher makes on the dominant degree-two mode \eqref{eq:mixed-timing-envelope}. The harmonic stop $T_4$ of Section~\ref{sec:mixed-main} also attains the order $d^{-5}$: this is the guarantee \eqref{eq:mixed-guarantee} for $\gamma=4$, whose upper bound $d^{k_0-\alpha}\overline{\mathcal L}_{\rm TE}\asymp d^{-4}$ is not attained by this target because its degree-four energy is small.

\begin{obs}[Leading constants used in the figures]\label{rem:mixed-constants}
The constants $c$ and $B$ play no role in Propositions~\ref{prop:mixed-capacity} and~\ref{prop:mixed-timing}, and the propositions are stated up to constants. The figures of this appendix, Figure~\ref{fig:illustrative-time-relative} and Figure~\ref{fig:mixed-target-time-curves} use the particular values $c=(\pi-3)/2$ and $B=C^2/24$ with $C=(12\pi-37)/24$, for which the two competing error terms in each example are equal at leading order, and they are normalized by the leading constants listed in Table~\ref{tab:mixed-constants}. These constants are obtained by carrying the constants through the proofs above, using the exact expressions for $N_k$ and $\sigma_k^2$ of Appendix~\ref{subapp:asympNksigmak}, a routine calculation that AI tools can perform; we record them here as claimed, for the reproducibility of the figures, and do not prove them. The numerical evaluations of Appendix~\ref{app:numerical-scheme} agree with them.
\end{obs}

\begin{table}[ht]
\centering\small
\renewcommand{\arraystretch}{1.4}
\begin{tabular}{@{}llll@{}}
\toprule
Target, width & $\overline{\mathcal L}_{\rm TE}$ & $\inf_{t\ge0}\overline{\mathcal L}_{\rm ST}(t)$ & Shown in the figures\\
\midrule
Pure $Y_{2,d}$, $m=d^3$ & $\sim c/d$ & $\overline{\mathcal L}_{\rm ST}(T_2)\sim c^2/d^2$ & Figure~\ref{fig:illustrative-time-relative}: ratio $\to1$\\
Proposition~\ref{prop:mixed-capacity}, $m=d^3$ & $\sim2c/d$ & $\sim c/d$, attained at $T_2$ & Figure~\ref{fig:capacity-obstruction}: both ratios $\to1/2$\\
Proposition~\ref{prop:mixed-timing}, $m=d^5$ & $\sim Cd^{-3}$ & $\sim\dfrac{C^2}{48}d^{-5}$, attained at $\sigma_4^2t=\log2$ & \begin{tabular}[t]{@{}l@{}}Figures~\ref{fig:mixed-target-time-curves}, \ref{fig:mixed-target-all-time-bounds}:\\ $\inf_t\overline{\mathcal L}_{\rm ST}(t)/\overline{\mathcal L}_{\rm TE}\sim C/(48d^2)$,\\ $\inf_t\overline{\mathcal L}_{\rm ST}(t)/\overline{\mathcal L}_{\rm TE}^{\,2}\sim d/48$,\\ $\inf_t\overline{\mathcal L}_{\rm ST}(t)/\overline{\mathcal L}_{\rm TE}^{5/3}\to C^{1/3}/48$\end{tabular}\\
\bottomrule
\end{tabular}
\caption{Claimed leading constants behind the figures, with $c=(\pi-3)/2$, $C=(12\pi-37)/24$ and $B=C^2/24$. The ratio in the last column is $\overline{\mathcal L}_{\rm ST}(T_2)/\overline{\mathcal L}_{\rm TE}^{\,2}$ in the first row, the two student-to-teacher ratios of Figure~\ref{fig:capacity-obstruction} in the second, and the three ratios named in the third.}
\label{tab:mixed-constants}
\end{table}

\begin{figure}[p]
\centering
\includegraphics[width=.8\textwidth]{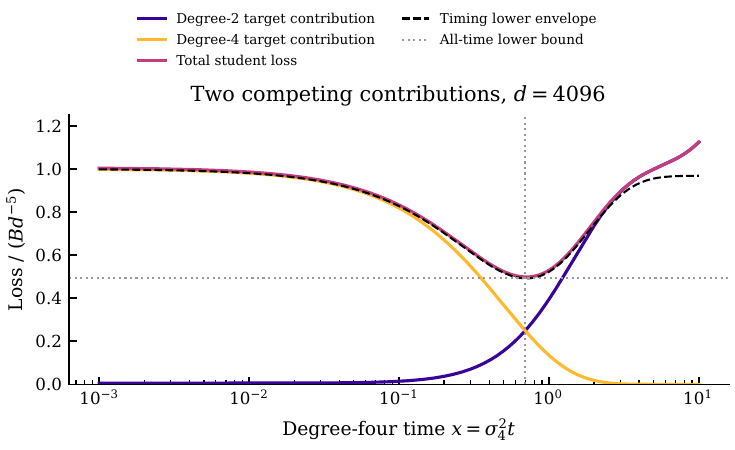}
\caption{The timing conflict of Proposition~\ref{prop:mixed-timing} at $d=4096$, with $B=C^2/24$ as in Remark~\ref{rem:mixed-constants} and loss divided by $Bd^{-5}$. The degree-four signal contribution decreases while the error inherited from the dominant degree-two mode increases. Their sum lies above the dashed envelope of \eqref{eq:mixed-timing-envelope}. Dotted lines mark its all-time lower bound and the asymptotically optimal time $\sigma_4^2t=\log2$. The plot focuses on the degree-four learning window; all curves are theoretical.}
\label{fig:mixed-target-timing-tradeoff}
\end{figure}

\begin{figure}[p]
\centering
\includegraphics[width=\textwidth]{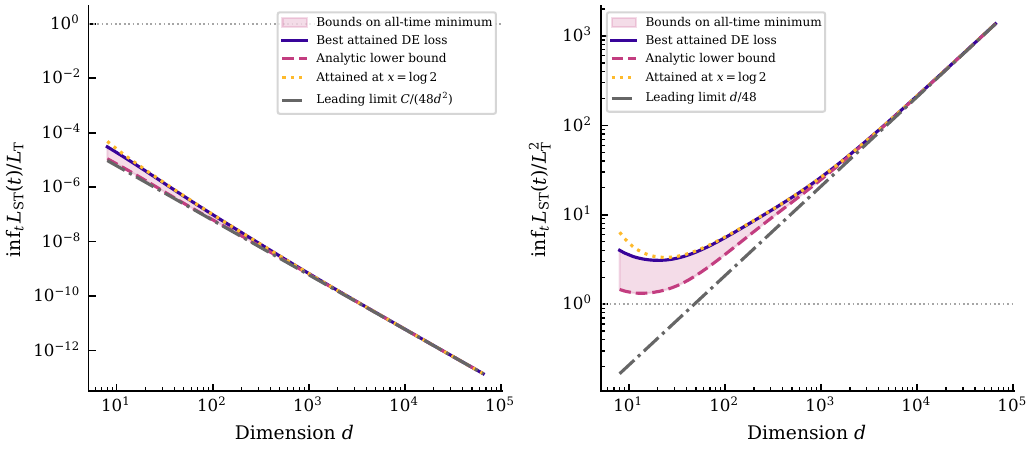}
\caption{Bounds on the all-time minimum in Proposition~\ref{prop:mixed-timing}, with $B=C^2/24$ as in Remark~\ref{rem:mixed-constants}. The shaded band lies between the analytic lower bound \eqref{eq:mixed-timing-envelope} and the best numerically attained loss. The explicit stop $\sigma_4^2t=\log2$ is another upper bound. The leading ratios $C/(48d^2)$ and $d/48$ of Remark~\ref{rem:mixed-constants} are shown separately: WTS occurs, but no stopping sequence achieves quadratic WTS. This is a conditional DE calculation.}
\label{fig:mixed-target-all-time-bounds}
\end{figure}

\section{Other activations}\label{app:general-activation}

The proofs of Appendix~\ref{app:reluproofs} compute the ReLU coefficients $\sigma_k$ exactly, but the conclusion of Theorem~\ref{theo:weaktostrongrelu} uses only a few coarse properties of the spectrum: the orders of $N_k$ and $\sigma_k^2$ at finitely many fixed degrees, a bound on the trace of $\Psi$, and the presence of the target degree in the kernel. In this appendix we show that these properties hold for a broad class of activations, and we prove the analogue of Theorem~\ref{theo:weaktostrongrelu} for them (Theorem~\ref{theo:general-activation}). The tanh and logistic sigmoid activations are worked out as examples. As elsewhere, all statements concern the deterministic equivalents \eqref{eq:main-teacher} and \eqref{eq:main-two-term}; the transfer to the random spherical model requires the analogue of Assumption~\ref{ass:de} for the activation in question.

\paragraph{Scaling.} For $u,x$ uniform on $\mathbb S^{d-1}$ the preactivation $\langle u,x\rangle$ is of order $d^{-1/2}$. ReLU is positively homogeneous, so this scale only multiplies the kernel by a constant: $(\langle u,x\rangle)_+=d^{-1/2}(\sqrt d\,\langle u,x\rangle)_+$. For a general activation the scale matters, and in this appendix we use the features $\sigma(\sqrt d\,\langle u,x\rangle)$, whose preactivations are approximately standard Gaussian. For ReLU this multiplies $\Psi$ by $d$. Under $\Psi\mapsto c\Psi$ the fixed point $\nu$ of \eqref{eq:main-fixed-point} becomes $c\nu$, the stopping time \eqref{eq:quadratic-stop} becomes $T/c$, and both $\overline{\mathcal L}_{\rm TE}$ and $\overline{\mathcal L}_{\rm ST}(T)$ are unchanged, so the rescaling has no effect on the results.

\paragraph{Harmonic coefficients.} Let $t=\langle u,x\rangle$ for independent uniform $u,x\in\mathbb S^{d-1}$, and let $P_k$ be the degree-$k$ Gegenbauer polynomial normalized by $P_k(1)=1$. These polynomials are orthogonal for the law of $t$, with $\mathbb E[P_k(t)^2]=1/N_k$, and the addition theorem states that
\[
\sum_{l=1}^{N_k}\phi_{k,l}(u)\phi_{k,l}(x)=N_kP_k(\langle u,x\rangle).
\]
Expanding $\sigma(\sqrt d\,t)$ in the $P_k$ and applying the addition theorem gives the Funk--Hecke expansion
\begin{equation}\label{eq:general-sigmak}
\sigma(\sqrt d\,\langle u,x\rangle)=\sum_{k\ge0}\sigma_k\sum_{l=1}^{N_k}\phi_{k,l}(u)\phi_{k,l}(x),
\qquad
\sigma_k=\mathbb E\bigl[\sigma(\sqrt d\,t)P_k(t)\bigr].
\end{equation}
Thus the feature $\sigma(\sqrt d\,\langle u,\cdot\rangle)$ has coordinates $\sigma_k\phi_{k,l}(u)$, and $\Psi$ has blocks $\sigma_k^2I_{N_k}$, as in \eqref{eq:relukernel}. For ReLU, $\sigma_k$ is $\sqrt d$ times the coefficient of Appendix~\ref{app:reluproofs}.

\paragraph{Hermite coefficients.} Let $z\sim\mathcal N(0,1)$ and let $\mathrm{He}_k$ be the probabilists' Hermite polynomials, so that the $\mathrm{He}_k/\sqrt{k!}$ are orthonormal for the law of $z$. Set
\begin{equation}\label{eq:hermite-coeff}
c_k=\frac{\mathbb E[\sigma(z)\mathrm{He}_k(z)]}{\sqrt{k!}},\qquad
\mathcal K_\sigma=\{k\ge0:\ c_k\neq0\}.
\end{equation}
Throughout, $\sigma$ is measurable with $|\sigma(r)|\le Ce^{C|r|}$ for some $C$. Since the Hermite polynomials are complete, $\mathcal K_\sigma$ is infinite unless $\sigma$ is a polynomial.

\begin{lemma}[Harmonic and Hermite coefficients]\label{lem:general-coefficients}
\leavevmode
\begin{enumerate}
\item[(i)] $\operatorname{tr}\Psi=\sum_{k\ge0}N_k\sigma_k^2=\mathbb E\bigl[\sigma(\sqrt d\,\langle u,x\rangle)^2\bigr]\to\mathbb E[\sigma(z)^2]$ as $d\to\infty$. In particular $\operatorname{tr}\Psi\lesssim1$.
\item[(ii)] For each fixed $k$, $N_k\sigma_k^2\to c_k^2$ as $d\to\infty$. In particular $\sigma_k^2\asymp d^{-k}$ for each fixed $k\in\mathcal K_\sigma$.
\item[(iii)] If $\sigma$ is odd, then $\sigma_k=0$ for every $d$ and every even $k$. If $\sigma$ is even, then $\sigma_k=0$ for every $d$ and every odd $k$.
\end{enumerate}
\end{lemma}

Item (ii) is standard \citep{ghorbani2021linearized}; we include the short proof.

\begin{proof}
The polynomials $h_k=\sqrt{N_k}\,P_k$ are orthonormal for the law of $t$, and by \eqref{eq:general-sigmak}
\begin{equation}\label{eq:general-Nksigmak}
N_k\sigma_k^2=\bigl(\mathbb E[\sigma(\sqrt d\,t)\,h_k(t)]\bigr)^2 .
\end{equation}
In words, $N_k\sigma_k^2$ is the squared $k$-th coefficient of $\sigma(\sqrt d\,\cdot)$ in the orthonormal polynomials of the law of $t$, while $c_k^2$ is the squared $k$-th coefficient of $\sigma$ in the orthonormal polynomials of the Gaussian.

\emph{(i)} The first equality holds because $\Psi$ has blocks $\sigma_k^2I_{N_k}$. The law of $t$ has compact support, so polynomials are dense in its $L^2$ space and $(h_k)_{k\ge0}$ is an orthonormal basis. The second equality is therefore Parseval's identity, by \eqref{eq:general-Nksigmak}. For the limit, $\sqrt d\,t$ has the density
\begin{equation}\label{eq:general-density}
\rho_d(z)=\kappa_d\Bigl(1-\frac{z^2}{d}\Bigr)^{\frac{d-3}{2}},\qquad |z|<\sqrt d,
\end{equation}
where $\kappa_d$ is a normalizing constant. Since $\log(1-y)\le-y$, we have $\rho_d(z)\le Ce^{-z^2/4}$ for $d\ge6$, and $\rho_d$ converges pointwise to the standard Gaussian density. By dominated convergence,
\[
\mathbb E\bigl[g(\sqrt d\,t)\bigr]\ \longrightarrow\ \mathbb E[g(z)]\qquad\text{for every $g$ with }|g(r)|\le Ce^{C|r|},
\]
and in particular for $g=\sigma^2$.

\emph{(ii)} In the variable $z=\sqrt d\,t$, the polynomial $h_k(z/\sqrt d)$ is the degree-$k$ orthonormal polynomial of $\rho_d$ with positive leading coefficient. It is obtained from $1,z,\dots,z^k$ by Gram--Schmidt, so its coefficients are continuous functions of the moments of $\rho_d$ of order at most $2k$, as long as the moment matrix is nonsingular. By the dominated convergence in (i), these moments converge to those of the Gaussian, whose moment matrix is positive definite. Hence the coefficients of $h_k(z/\sqrt d)$ converge to those of $\mathrm{He}_k(z)/\sqrt{k!}$, and a second application of dominated convergence gives
\[
\mathbb E\bigl[\sigma(\sqrt d\,t)\,h_k(t)\bigr]\ \longrightarrow\ \frac{\mathbb E[\sigma(z)\mathrm{He}_k(z)]}{\sqrt{k!}}=c_k .
\]
By \eqref{eq:general-Nksigmak}, $N_k\sigma_k^2\to c_k^2$. Since $N_k\asymp d^k$, this gives $\sigma_k^2\asymp d^{-k}$ when $c_k\ne0$.

\emph{(iii)} The law of $t$ is symmetric, so $P_k(-t)=(-1)^kP_k(t)$. If $\sigma$ is odd and $k$ is even, the function $t\mapsto\sigma(\sqrt d\,t)P_k(t)$ is odd, and $\sigma_k=0$ by \eqref{eq:general-sigmak}. The case of $\sigma$ even and $k$ odd is the same.
\end{proof}

\begin{theo}[Quadratic improvement for general activations]\label{theo:general-activation}
Let $\sigma$ be as above, let $f^\star$ be a unit spherical harmonic of fixed degree $k_0\in\mathcal K_\sigma$, $k_0\ge1$, and let $m=d^\alpha$ with fixed $\alpha>0$. Assume that $c_{k_\star}\neq0$ for some $k_\star>\alpha$; this holds for every $\alpha$ when $\sigma$ is not a polynomial. Then, with the stopping time $T$ of \eqref{eq:quadratic-stop},
\[
\overline{\mathcal L}_{\rm TE}\asymp d^{-(\alpha-k_0)_+},
\qquad
\overline{\mathcal L}_{\rm ST}(T)\asymp\overline{\mathcal L}_{\rm TE}^{\,2}.
\]
\end{theo}

Together with Theorem~\ref{theo:quadratic-floor}, this gives optimal quadratic improvement for every such activation once $\alpha>k_0$. The powers of $d$ are those of Theorem~\ref{theo:weaktostrongrelu} and do not depend on the activation. The only structural hypothesis is that the kernel sees the target degree, which is the hypothesis $k_0\in\mathcal K$ of Theorem~\ref{theo:weaktostrongrelu} with $\mathcal K$ replaced by $\mathcal K_\sigma$.

\begin{proof}
All constants below depend only on $\sigma$, $\alpha$, and $k_0$, and $d$ is taken large enough.

\emph{Step 1: $\nu\asymp d^{-\alpha}$.} Since $\sigma_k^2/(\sigma_k^2+\nu)\le\sigma_k^2/\nu$, the fixed point \eqref{eq:main-fixed-point} gives
\[
d^\alpha=\sum_{k\ge0}\frac{N_k\sigma_k^2}{\sigma_k^2+\nu}\le\frac{\operatorname{tr}\Psi}{\nu},
\qquad\text{so}\qquad
\nu\lesssim d^{-\alpha}
\]
by Lemma~\ref{lem:general-coefficients}(i). For the lower bound we keep only the degree $k_\star$:
\[
d^\alpha\ \ge\ \frac{N_{k_\star}\sigma_{k_\star}^2}{\sigma_{k_\star}^2+\nu}.
\]
If $\nu\le\sigma_{k_\star}^2$, the right-hand side is at least $N_{k_\star}/2\asymp d^{k_\star}$, which exceeds $d^\alpha$ because $k_\star>\alpha$. Hence $\nu>\sigma_{k_\star}^2$, and
\[
\nu\ \ge\ \frac{N_{k_\star}\sigma_{k_\star}^2}{2\,d^\alpha}\ \gtrsim\ d^{-\alpha},
\]
since $N_{k_\star}\sigma_{k_\star}^2\to c_{k_\star}^2>0$ by Lemma~\ref{lem:general-coefficients}(ii). This is the argument of Lemma~\ref{lem:nu-relu}.

\emph{Step 2: the teacher.} Write $\lambda=\sigma_{k_0}^2$ and $L=\overline{\mathcal L}_{\rm TE}$. By Lemma~\ref{lem:general-coefficients}(ii), $\lambda\asymp d^{-k_0}$. For a unit target of degree $k_0$, \eqref{eq:main-teacher} and Step~1 give
\[
L=\frac{\nu}{\lambda+\nu}\ \asymp\ \min\{1,\,d^{k_0}\nu\}\ \asymp\ d^{-(\alpha-k_0)_+}.
\]

\emph{Step 3: the student.} The lower bound $\overline{\mathcal L}_{\rm ST}(T)\ge L^2$ is Theorem~\ref{theo:quadratic-floor}. For the upper bound we bound the two terms of \eqref{eq:main-two-term}. The choice of $T$ gives $e^{-\lambda T}=L/m$, so the squared bias is
\[
\overline B(T)=\Bigl(\frac{\nu+\lambda e^{-\lambda T}}{\lambda+\nu}\Bigr)^2
=L^2\Bigl(1+\frac{1-L}{m}\Bigr)^2\le 4L^2 .
\]
In the variance, the prefactor is $\nu\lambda/(\lambda+\nu)^2=L(1-L)\le L$, so
\[
\overline V(T)\le L\,R(T),
\qquad
R(T)=\frac{\operatorname{tr}[\Psi(\Psi+\nu I)^{-2}(I-e^{-\Psi T})^2]}{\operatorname{tr}[\Psi(\Psi+\nu I)^{-2}]}\le1 .
\]
If $\alpha\le k_0$, then $L\asymp1$, and $\overline{\mathcal L}_{\rm ST}(T)\le4L^2+L\lesssim L^2$.

Let $\alpha>k_0$. It remains to show that $R(T)\lesssim L$. In the denominator we keep only the degree $k_\star$. Since $\sigma_{k_\star}^2\asymp d^{-k_\star}\le\nu$,
\[
\operatorname{tr}[\Psi(\Psi+\nu I)^{-2}]\ \ge\ \frac{N_{k_\star}\sigma_{k_\star}^2}{(\sigma_{k_\star}^2+\nu)^2}\ \ge\ \frac{N_{k_\star}\sigma_{k_\star}^2}{4\nu^2}\ \gtrsim\ \nu^{-2}.
\]
In the numerator, we bound the degrees $k\le k_0$ using $(1-e^{-xT})^2\le1$ and $x/(x+\nu)^2\le1/(4\nu)$, and the degrees $k>k_0$ using $(1-e^{-xT})^2\le x^2T^2$ and $x^2/(x+\nu)^2\le1$. This gives
\[
\operatorname{tr}[\Psi(\Psi+\nu I)^{-2}(I-e^{-\Psi T})^2]
\ \le\ \sum_{k\le k_0}\frac{N_k}{4\nu}+T^2\sum_{k>k_0}N_k\sigma_k^2
\ \lesssim\ \frac{d^{k_0}}{\nu}+d^{2k_0}(\log d)^2,
\]
where we used $\operatorname{tr}\Psi\lesssim1$ and $T=\lambda^{-1}\log(m/L)\lesssim d^{k_0}\log d$. Dividing, and using $\nu\asymp d^{-\alpha}$ and $k_0<\alpha$,
\[
R(T)\ \lesssim\ \nu^2\Bigl(\frac{d^{k_0}}{\nu}+d^{2k_0}(\log d)^2\Bigr)
\ \asymp\ d^{-(\alpha-k_0)}+d^{-2(\alpha-k_0)}(\log d)^2
\ \lesssim\ d^{-(\alpha-k_0)}\ \asymp\ L .
\]
Hence $\overline V(T)\lesssim L^2$, and $\overline{\mathcal L}_{\rm ST}(T)=\overline B(T)+\overline V(T)\lesssim L^2$.
\end{proof}

\begin{obs}[The hypotheses]\label{rem:general-activation-scope}
The hypothesis $k_0\in\mathcal K_\sigma$ cannot be dropped. If $\sigma$ has the parity opposite to that of $k_0$, then $\sigma_{k_0}=0$ for every $d$ by Lemma~\ref{lem:general-coefficients}(iii). The target is then orthogonal to every feature, so the teacher's best fit is zero; the student, which starts at zero and is trained on the teacher, stays at zero; and both errors equal one at all times. The hypothesis $c_{k_\star}\ne0$ for some $k_\star>\alpha$ fails only for a polynomial $\sigma$ of degree $p\le\alpha$. Then $\sigma_k=0$ for every $k>p$ and every $d$, so $\Psi$ has rank at most $\sum_{k\le p}N_k\asymp d^p$. For $\alpha>p$ this is smaller than $m$, the fixed point \eqref{eq:main-fixed-point} has no solution, and the teacher represents every target in the span of the features exactly.

The proof uses Lemma~\ref{lem:general-coefficients} only through $\operatorname{tr}\Psi\lesssim1$ and $N_k\sigma_k^2\asymp1$ at the two degrees $k_0$ and $k_\star$. It uses neither the exact values of the $\sigma_k$, nor monotonicity of $\sigma_k^2$ in $k$, nor any rate of decay of the $c_k$. We expect the stopping-time phases of Theorem~\ref{theo:phases} and the mixed-target examples of Appendix~\ref{app:mixed-examples} to carry over in the same way, with $\mathcal K$ replaced by $\mathcal K_\sigma$, since they use only the orders of $N_k$ and $\sigma_k^2$; we have not written this out.
\end{obs}

\begin{ex}[ReLU]\label{ex:general-relu}
Let $\sigma(r)=r_+$, and let $\varphi$ be the standard Gaussian density. Two integrations by parts, using $\mathrm{He}_k\varphi=-(\mathrm{He}_{k-1}\varphi)'$, give
\[
\mathbb E[z_+\mathrm{He}_k(z)]=\varphi(0)\,\mathrm{He}_{k-2}(0),\qquad k\ge2,
\]
which vanishes exactly when $k$ is odd. Hence $\mathcal K_\sigma=\{0,1\}\cup2\mathbb N=\mathcal K$, with $c_0^2=1/(2\pi)$ and $c_1^2=1/4$, and Theorem~\ref{theo:general-activation} recovers Theorem~\ref{theo:weaktostrongrelu}. The coefficients $\sigma_k$ with $k\ge3$ odd vanish exactly at every $d$, by Lemma~\ref{lem:general-coefficients}(iii): $r_+-r/2=|r|/2$ is even, and the linear part $r/2$ contributes only to degree one.
\end{ex}

\begin{ex}[tanh]\label{ex:general-tanh}
Let $\sigma=\tanh$. We show that $\mathcal K_\sigma$ is the set of odd integers. Since $\tanh$ is odd, $c_k=0$ for every even $k$, and $\sigma_k=0$ exactly at every even degree by Lemma~\ref{lem:general-coefficients}(iii). Let $k$ be odd. One Gaussian integration by parts gives
\[
\mathbb E[\tanh(z)\mathrm{He}_k(z)]=\mathbb E[\operatorname{sech}^2(z)\mathrm{He}_{k-1}(z)].
\]
The function $\operatorname{sech}^2$ has a positive, integrable Fourier transform,
\[
\operatorname{sech}^2(r)=\frac1{2\pi}\int_{\mathbb R}\frac{\pi\xi}{\sinh(\pi\xi/2)}\,e^{i\xi r}\,d\xi,
\]
and $\mathbb E[e^{i\xi z}\mathrm{He}_{k-1}(z)]=(i\xi)^{k-1}e^{-\xi^2/2}$. Exchanging the order of integration, and using that $k-1$ is even,
\[
\mathbb E[\tanh(z)\mathrm{He}_k(z)]
=\frac{(-1)^{(k-1)/2}}{2\pi}\int_{\mathbb R}\xi^{k-1}\,\frac{\pi\xi}{\sinh(\pi\xi/2)}\,e^{-\xi^2/2}\,d\xi .
\]
The integrand is positive, so $c_k\ne0$, with sign $(-1)^{(k-1)/2}$. Hence Theorem~\ref{theo:general-activation} applies to every odd target degree $k_0$ and every $\alpha>0$:
\[
\overline{\mathcal L}_{\rm TE}\asymp d^{-(\alpha-k_0)_+},\qquad
\overline{\mathcal L}_{\rm ST}(T)\asymp\overline{\mathcal L}_{\rm TE}^{\,2},
\]
exactly as for ReLU. For an even target degree nothing is learned, by Remark~\ref{rem:general-activation-scope}. The argument uses only that $\sigma$ is odd and that $\sigma'$ has a nonnegative Fourier transform that is not identically zero; it applies, for instance, to $\sigma=\operatorname{erf}$.
\end{ex}

\begin{ex}[Logistic sigmoid]\label{ex:general-sigmoid}
Let $s(r)=1/(1+e^{-r})=\tfrac12+\tfrac12\tanh(r/2)$. The function $s-\tfrac12$ is odd, and the constant $\tfrac12$ contributes only to degree zero. Hence $c_0=\tfrac12$ and $c_k=0$ for every even $k\ge2$, and by Lemma~\ref{lem:general-coefficients}(iii) also $\sigma_k=0$ for every even $k\ge2$ and every $d$. For odd $k$, the argument of Example~\ref{ex:general-tanh} applies with $\operatorname{sech}^2$ replaced by $s'(r)=\tfrac14\operatorname{sech}^2(r/2)$, whose Fourier transform $\pi\xi/\sinh(\pi\xi)$ is again positive. Hence $\mathcal K_s$ consists of $0$ and the odd integers, and the conclusions of Example~\ref{ex:general-tanh} hold unchanged.

A bias breaks the parity. For the activation $r\mapsto s(r+b)$, the same computation gives, for $k\ge1$,
\[
\mathbb E[s(z+b)\mathrm{He}_k(z)]=\frac1{2\pi}\int_{\mathbb R}\frac{\pi\xi}{\sinh(\pi\xi)}\,e^{i\xi b}\,(i\xi)^{k-1}e^{-\xi^2/2}\,d\xi .
\]
As a function of $b$ this extends to an entire function. For odd $k$ it is nonzero at $b=0$, and for even $k$ its derivative at $b=0$ is nonzero by the same positivity. So it is not identically zero and has countably many real zeros. Therefore, for all $b$ outside a countable set, $c_k\ne0$ for every $k$, and Theorem~\ref{theo:general-activation} applies to targets of every degree $k_0\ge1$.
\end{ex}
\clearpage
\section{Numerical details for the main figures}\label{app:numerical-scheme}

The main figures contain two kinds of curves. The lines are evaluations of the deterministic equivalents \eqref{eq:main-fixed-point}--\eqref{eq:main-two-term}. The markers in Figure~\ref{fig:illustrative-time-relative} are averages over random teachers of the population flow of Section~\ref{sec:setup}, simulated as described next.

\subsection{Simulating the population flow}\label{subapp:numerical-population}

A population-flow sample is one draw of the model of Section~\ref{sec:setup}: the teacher's feature directions $u_1,\dots,u_m$ are drawn independently and uniformly on $\mathbb S^{d-1}$ with $m=d^3$, the teacher is fitted by minimizing its population loss, and the student is the population spectral flow \eqref{studenttrain} against that teacher. Every expectation over the data is computed in closed form, so the feature directions are the only randomness.

\paragraph{Teacher.} Writing the teacher as $g(x)=\sum_{i=1}^ma_i\sigma(\langle u_i,x\rangle)$, its population least-squares weights solve $Ga=h$ with
\[
G_{ij}=\mathbb E_x\bigl[\sigma(\langle u_i,x\rangle)\sigma(\langle u_j,x\rangle)\bigr],
\qquad
h_i=\mathbb E_x\bigl[\sigma(\langle u_i,x\rangle)f^\star(x)\bigr].
\]
Both are available in closed form. For ReLU on the sphere, with $s=\langle u_i,u_j\rangle$, the kernel is $\bigl(\sqrt{1-s^2}+(\pi-\arccos s)\,s\bigr)/(2\pi d)$, and for a pure degree-$k_0$ harmonic target $h_i=\sigma_{k_0}f^\star(u_i)$. The simulations use the population-centered activation $\sigma(r)=r_+-\mu_d$, which subtracts $\mu_d^2$ from every kernel entry and removes the degree-zero mode of the kernel; the targets have no degree-zero component. The $m\times m$ system is solved in double precision; its condition number is below $10^5$ for $d\le28$, and the relative residual $\|Ga-h\|/\|h\|$ is checked to be at roundoff level for every teacher. The teacher error is $\mathcal L_{\rm TE}=1-2\langle g,f^\star\rangle_{\mathcal D}+\|g\|_{\mathcal D}^2$.

\paragraph{Student.} The population flow \eqref{studenttrain} acts on each harmonic degree $k$ by the scalar filter $1-e^{-\sigma_k^2t}$. Writing $\Pi_k$ for the projection onto degree-$k$ harmonics, the student error at every time is therefore determined by the harmonic energies of the teacher,
\begin{equation}\label{eq:numerical-harmonic-loss}
\mathcal L_{\rm ST}(t)=\sum_{k}\Bigl[\|\Pi_kf^\star\|_{\mathcal D}^2-2(1-e^{-\sigma_k^2t})\langle\Pi_kg,\Pi_kf^\star\rangle_{\mathcal D}+(1-e^{-\sigma_k^2t})^2\|\Pi_kg\|_{\mathcal D}^2\Bigr].
\end{equation}
The energies $\|\Pi_kg\|_{\mathcal D}^2=\sigma_k^2\,a^\top H_ka$ with $(H_k)_{ij}=N_kP_k(\langle u_i,u_j\rangle)$, where $P_k$ is the degree-$k$ Gegenbauer polynomial normalized by $P_k(1)=1$, follow from the addition theorem; the matrices $H_k$ are generated by the three-term recurrence of the $P_k$ on the $m\times m$ Gram matrix of the directions, which works directly with the $m\times m$ Gram matrix. One teacher fit and one pass of the recurrence give the whole curve $t\mapsto\mathcal L_{\rm ST}(t)$.

\paragraph{The harmonic cutoff.} The sum \eqref{eq:numerical-harmonic-loss} has to be cut at a finite degree $K$; we use $K=128$. The omitted part is $\sum_{k>K}(1-e^{-\sigma_k^2t})^2\|\Pi_kg\|_{\mathcal D}^2$, and since the filters decrease in $k$ it is at most $(1-e^{-\sigma_{K+2}^2t})^2$ times the teacher energy above degree $K$, which is $\|g\|_{\mathcal D}^2-\sum_{k\le K}\|\Pi_kg\|_{\mathcal D}^2$ and is known from the fit. This bound is evaluated for every teacher and every time. It is small on the displayed time window for a structural reason: the student learns degree $k$ on the time scale $1/\sigma_k^2\asymp d^{k+1}$, so at times of order $T_2\asymp d^3\log d$ the degrees above the first few are still at their zero initialization and contribute nothing to \eqref{eq:numerical-harmonic-loss}. Figure~\ref{fig:population-flow-uncertainty} (bottom row) quantifies this: cutting at $K=16$ instead of $K=128$ changes the loss by less than $10^{-6}$ relative on the whole window, and the bound on the tail above $K=128$ is below $4\times10^{-4}$ relative to the loss, attaining its maximum near $T_2$ where the loss itself is smallest.

\subsection{Averaging over teachers}\label{subapp:illustrative-time-curves}

Each marker in Figure~\ref{fig:illustrative-time-relative} is the mean of $\mathcal L_{\rm ST}(t)$ over independent teachers: 32 at $d=8$ and $d=16$, and 8 at $d=28$, where $m=21952$. Figure~\ref{fig:population-flow-uncertainty} shows what is being averaged. The top row plots every teacher's curve, their mean, a pointwise 95\% $t$ interval for the mean, and the deterministic equivalent. The teacher-to-teacher spread is invisible before the stop, where the student is learning the target, and appears after it, where the student is copying the teacher's random residual: this is the variance term $\overline V$ of \eqref{eq:main-two-term}. At the stop $T_2$ the half-width of the interval is 9\%, 3\%, and 5\% of the mean for $d=8,16,28$, and the deterministic equivalent lies below the mean by 15\%, 4\%, and 1\%, a finite-$d$ discrepancy that closes as $d$ grows. In the normalized panel of Figure~\ref{fig:illustrative-time-relative} the mean teacher error of the same teachers is used as denominator. At $d=28$ the mean student loss at $T_2$ divided by the squared mean teacher loss is $1.067$, with paired-bootstrap 95\% interval $[1.065,1.069]$ over teachers, against the deterministic-equivalent value $1.066$.

\begin{figure}[htbp]
\centering
\includegraphics[width=\linewidth]{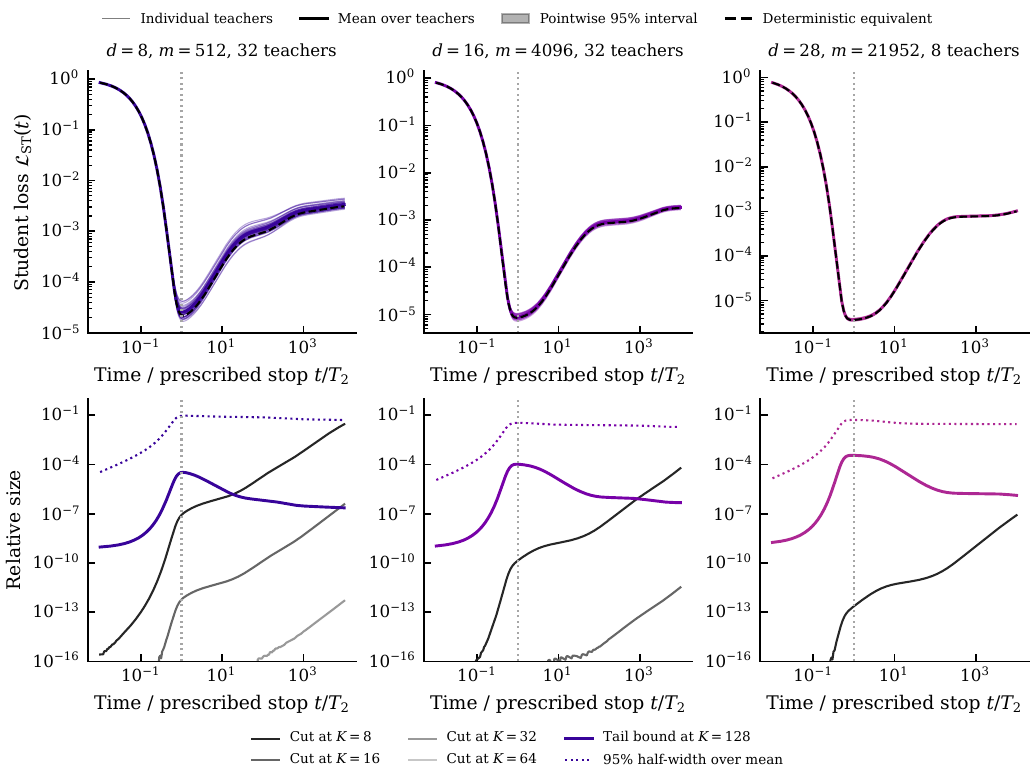}
\caption{\textbf{Averaging over teachers and the harmonic cutoff.} Unit degree-two target, $m=d^3$, the three dimensions with population-flow samples in Figure~\ref{fig:illustrative-time-relative}. Top: individual teachers, their mean, a pointwise 95\% interval for the mean, and the deterministic equivalent. Bottom, on the same time axis: the relative change of the loss when the harmonic sum \eqref{eq:numerical-harmonic-loss} is cut at $K$ instead of $K=128$ (worst teacher), the bound on the omitted tail above $K=128$, and the relative half-width of the 95\% interval. Cutting at $K\ge16$ is indistinguishable from $K=128$ on this window; the cutoff error is far below the statistical uncertainty everywhere.}
\label{fig:population-flow-uncertainty}
\end{figure}

\subsection{Deterministic-equivalent curves}\label{subapp:numerical-de}

The lines in the figures evaluate \eqref{eq:main-fixed-point}--\eqref{eq:main-two-term} with the exact $N_k$ and $\sigma_k^2$ of Appendix~\ref{subapp:asympNksigmak}, grouping every trace by harmonic degree. The scalar equation \eqref{eq:main-fixed-point} for $\nu$ is solved by bisection. Multiplicities and eigenvalues are combined in the logarithmic domain, since $N_k$ overflows and $\sigma_k^2$ underflows long before their product does, and the filters $1-e^{-x}$ are evaluated with a stable exponential-difference routine. The harmonic sums are cut at degree 2048; raising the cutoff from 1024 to 2048 changes every displayed curve by less than $0.04\%$, and the curves reproduce the exact endpoints, unit loss at $t=0$ and $\overline{\mathcal L}_{\rm TE}$ at $t=\infty$. The mixed-target figures of Section~\ref{sec:mixed-main} and Appendix~\ref{app:mixed-examples} are such evaluations; the numerically attained minima marked there are minima over a grid in $\sigma_4^2t$, and at the largest dimension shown they lie within $0.1\%$ of the analytic lower bound \eqref{eq:mixed-timing-envelope}, at a time within $0.2\%$ of $\log2$.

\subsection{The degree-zero mode}\label{subapp:numerical-zero-mode}

All figures were generated without the degree-zero mode of the kernel \eqref{eq:relukernel}. The deterministic-equivalent curves use the spectrum over positive degrees, and the population-flow samples use the centered activation $\sigma(r)=r_+-\mu_d$, whose kernel has no constant component. The main text keeps the mode, so we record its effect. For the deterministic equivalent, adding the term $N_0=1$, $\sigma_0^2=\mu_d^2$ to the spectrum changes $\nu$ and $\overline{\mathcal L}_{\rm TE}$ by a relative amount close to $1/m$, and the student curve of Figure~\ref{fig:illustrative-time-relative} by at most twice that: $0.6\%$ at $d=8$, $0.07\%$ at $d=16$, $0.01\%$ at $d=28$, and below $10^{-5}$ for $d\ge64$. For the population flow, refitting the archived teachers of Figure~\ref{fig:population-flow-uncertainty} with the uncentered kernel changes the teacher error by at most $1.9\%$, $0.2\%$, and $0.03\%$ at $d=8,16,28$ (on average $0.4\%$, $0.05\%$, and $0.01\%$), and the student curve at any time by at most $5.8\%$, $0.4\%$, and $0.06\%$. The constant component of the uncentered teacher itself carries less than $10^{-6}$ of its error at every dimension; the change comes through the fitted weights. Both effects are below the teacher-to-teacher variability shown in Figure~\ref{fig:population-flow-uncertainty}.

\end{document}